\documentclass[12pt]{article}

\usepackage{amssymb,amsfonts,amsmath}
\usepackage{amsthm}
\usepackage{mathrsfs}

\usepackage{subcaption}
\usepackage{graphicx}
\usepackage{multirow}

\usepackage{overpic}
\usepackage{upgreek}

\usepackage{geometry}
\usepackage{xcolor}
\usepackage{multirow}
\usepackage{palatino}
\usepackage{tikz}
\usetikzlibrary{arrows,shapes,chains} 
\usepackage{authblk}
\usepackage{natbib}
\usepackage{setspace}
\usepackage{url}
\usepackage{hyperref}
\usepackage{titlesec}

\usepackage{algorithm}  
\usepackage{algorithmicx}
\usepackage[noend]{algpseudocode}

\usepackage{cases}

\newtheorem{theorem}{Theorem}
\newtheorem{definition}{Definition}
\newtheorem{proposition}{Proposition}

\newtheorem*{claim}{Claim}
\newtheorem{lemma}{Lemma}

\newtheorem{corollary}{Corollary}

\titleformat{\title}{\LARGE\bfseries}{\thesection}{1em}{} 
\titleformat{\section}{\Large\bfseries}{\thesection}{1em}{}  
\titleformat{\subsection}{\normalsize\bfseries}{\thesubsection}{1em}{}  
\titleformat{\subsubsection}{\small\bfseries}{\thesubsubsection}{1em}{} 

\title{Online Bandits with (Biased) Offline Data: Adaptive Learning under Distribution Mismatch}

\author{\small Wang Chi Cheung$^1$, Lixing Lyu$^2$} 
\affil{\footnotesize 
$^1$ Department of Industrial Systems Engineering and Management, National University of Singapore \\
$^2$ Institute of Operations Research and Analytics, National University of Singapore \\
}

\allowdisplaybreaks[3]

\begin{document}

\bibliographystyle{plainnat}
\date{}  
\maketitle
\newcommand*\abs[1]{\lvert#1\rvert}

\begin{abstract}%
Traditional online learning models are typically initialized from scratch. By contrast, contemporary real-world applications often have access to historical datasets that can potentially enhanced the online learning processes. We study how offline data can be leveraged to facilitate online learning in stochastic multi-armed bandits and combinatorial bandits. In our study, the probability distributions that govern the offline data and the online rewards can be different. We first show that, without a non-trivial upper bound on their difference, no non-anticipatory policy can outperform the classical Upper Confidence Bound (UCB) policy by \cite{AuerCBF02}, even with the access to offline data. In complement, we propose an online policy MIN-UCB for multi-armed bandits. MIN-UCB outperforms the UCB when such an upper bound is available. MIN-UCB adaptively chooses to utilize the offline data when they are deemed informative, and to ignore them otherwise. We establish that MIN-UCB achieves tight regret bounds, in both instance independent and dependent settings. 
We generalize our approach to the combinatorial bandit setting by introducing MIN-COMB-UCB, and we provide corresponding instance dependent and instance independent regret bounds. We illustrate how various factors, such as the biases and the size of offline datasets, affect the utility of offline data in online learning.
We discuss several applications and conduct numerical experiments to validate our findings.
\end{abstract}

\section{Introduction}

Online learning problems with bandit feedback, including multi-armed bandits (MAB), combinatorial bandits and linear bandits, constitute a class of classical models in sequential decision-making. In traditional models, online learning is initialized with no historical dataset. The Decision Maker (DM) must acquire knowledge solely through the online interactions with the environment. 
By contrast, in many contemporary real-world scenarios, 
past datasets related to the underlying model are often available before online learning begins.
These offline datasets could be potentially beneficial to online learning. 
For instance, when Apple launches a new iPhone model in Canada, past sales data from the US market could be of relevance. 
Similar practices are common in operations management, where firms routinely leverage historical demands and pricing data to aid decision-making in new markets or selling seasons.

The above observation raises an intriguing question: Can we develop an algorithm that effectively utilizes offline data for online learning? Such approaches carry significant allure, as they potentially reduce costly exploration in the learning process. Nevertheless, it is often too idealistic, if not flawed, to assume that the historical data and the online environment are governed by the same probability distribution, given the ubiquity of distributional shifts in every-day life applications. 
Indeed, past data may be biased or misleading. 
For example, Amazon extrapolated behavioral insights from Kindle Fire tablet users to guide the Fire Phone design. However, the resulting product failed to align with core smartphone demand and was discontinued within a year \cite{wohlsen2015amazon}.
In such cases, using offline data can harm online decision-making.

This presents a fundamental challenge: \textbf{How to effectively use offline data, that may be biased or mismatch as compared to the online environment?}
Ideally, we aim to design an adaptive policy that reaps the benefit of offline data when it is informative, i.e. when the offline and online distributions are sufficiently close, while judiciously ignores the offline data and learns from scratch otherwise. 
Such an approach is particularly valuable in practice, as it can accelerate decision making and reduce costly exploration when offline data is reliable, while avoiding the risks of distribution shifts that have caused failures in real-world applications. Hence, developing methods to manage these two sources of information is both meaningful and necessary.

Motivated by the above discussions, we first consider a stochastic multi-armed bandit model with a possibly biased offline dataset. The learning horizon consists of a warm-start phase, followed by an online phase whereby decision making occurs. During the warm-start phase, the DM receives an offline dataset, consisting of samples governed by a latent probability distribution $P^{\text{(off)}}$. The subsequent online phase is the same as the standard multi-armed bandit model, except that the offline dataset can be incorporated in decision making. The random online rewards are governed by a latent probability distribution $P^{\text{(on)}}$, which can be different from $P^{\text{(off)}}$. The DM aims to maximize the total cumulative reward 
during the online phase. Subsequently, we extend the above framework to a stochastic combinatorial bandit model.

\subsection{Research Intuitions and Main Contributions}

Intuitively, in the case of multi-armed bandits, when the offline data distribution $P^{\text{(off)}}$ and online reward distribution $P^{\text{(on)}}$ are ``far apart'', the DM should conduct online learning from scratch and ignore the offline dataset. For example, the UCB policy by \cite{AuerCBF02}, which we call the ``vanilla UCB'' in the remainder of the manuscript, incurs expected regret (i.e. the performance gap to the optimal policy, formally defined in the forthcoming Section \ref{sec:model-mab}) at most 
\begin{equation} \label{eq:vanilla_reg}
     O\left(\sum_{a \in \mathcal{A}:\Delta(a) >0} \frac{\log T}{\Delta(a)}\right),
\end{equation}
where $T$ is the number of time steps in the online phase, $\mathcal{A}$ is the arm set, and $\Delta(a)\geq 0$ is the difference between the expected reward of an optimal arm and that of arm $a$ (see Section \ref{sec:model-mab} for the full definition). The bound (\ref{eq:vanilla_reg}) holds for all $P^{\text{(off)}}$, $P^{\text{(on)}}$. On the contrary, when $P^{\text{(off)}}$ and $P^{\text{(on)}}$ are ``sufficiently close'', the DM should incorporate the offline data into online learning and avoid unnecessary exploration. 
For example, when $P^{\text{(off)}} = P^{\text{(on)}}$, HUCB1 \citep{shivaswamy2012multi} and MonUCB \citep{banerjee2022artificial} incur expected regret at most
\begin{equation} \label{eq:improved_reg}
    O\left(\sum_{a \in \mathcal{A}:\Delta(a)>0} \max\left\{\frac{\log T}{\Delta(a)} - T_\text{S}(a)\Delta(a), 0\right\}\right),
\end{equation}
where $T_{\text{S}}(a)$ is the number of offline samples on arm $a$. (\ref{eq:improved_reg}) is a strictly smaller regret bound than  (\ref{eq:vanilla_reg}). However, (\ref{eq:improved_reg}) only holds when $P^{\text{(off)}} = P^{\text{(on)}}$. The above inspires the following question:
\begin{center}
\textit{(Q) Can the DM \emph{outperform} the vanilla UCB, i.e. achieves a strictly smaller regret bound than (\ref{eq:vanilla_reg}) when $P^{\text{(off)}} = P^{\text{(on)}}$, while retains the regret bound (\ref{eq:vanilla_reg}) for general $P^{\text{(off)}}, P^{\text{(on)}}$?}
\end{center}

The answer to (Q) turns out to be somewhat mixed.
Our novel contributions shed light on (Q):

\textbf{An Impossibility Result. } In Section \ref{sec:impossibility}, we show that no non-anticipatory policy can achieve the aim in (Q). Even with an offline dataset, no non-anticipatory policy can outperform the vanilla UCB without any additional knowledge or restriction on the difference between $P^{\text{(off)}}, P^{\text{(on)}}$. 

\textbf{Multi-armed Bandit Models.} To bypass the impossibility result, we assume that the DM is provided with auxiliary information dubbed \emph{valid bias bound} $V$, which supplements the offline dataset. The bound $V$ serves as an upper bound on the difference between $P^{\text{(off)}}$ and $P^{\text{(on)}}$. We propose the MIN-UCB policy, which achieves a strictly smaller regret bound than (\ref{eq:vanilla_reg}) when $P^{\text{(off)}},   P^{\text{(on)}}$ are ``sufficiently close''. In particular, our regret bound reduces to (\ref{eq:improved_reg}) if the DM knows $P^{\text{(off)}} =  P^{\text{(on)}}$. We provide both instance-dependent and instance-independent regret upper bounds on MIN-UCB. Additionally, we show the tightness of both regret upper bounds by providing the corresponding regret lower bounds. Our analysis precisely characterizes the meaning of ``far apart'' and ``sufficiently close'' in the above discussion. The design and analysis of MIN-UCB are provided in Section \ref{sec:MINUCB}.

\textbf{Generalization to Combinatorial Bandit Models.} We extend our framework to combinatorial bandits in Section \ref{sec:MIN-COMB-UCB}. 
We propose the MIN-COMB-UCB policy, and establish an instance-independent regret upper bound.
Similar to the multi-armed bandit case, our regret upper bound is uniformly better than that without access to offline data, such as \cite{chen2013combinatorial}. When $P^{\text{(off)}}, P^{\text{(on)}}$ are sufficiently close, our regret upper bound is strictly smaller than that in \cite{chen2013combinatorial}.
In the special case of combinatorial bandits with linear rewards, we further derive tight instance-independent regret upper bounds and instance-dependent regret upper bounds.    
The latter shares a similar saving term to that in the multi-armed bandit model, quantifying the benefit of offline data. In particular, when we specialize the combinatorial bandit model degenerates to the multi-armed bandit model, our regret upper bounds reduce to those presented in Section \ref{sec:MINUCB}. The generalization shows the adaptivity and scalability of our approach to broader online learning models.

\textbf{Technical Novelty.} Our technical contributions are summarized as follows. 
\begin{enumerate}
    \item \textbf{Impossibility Result under Unknown Bias:} To the best of our knowledge, our work is the first to establish a regret lower bound in the presence of potentially biased offline data, without knowing any prior information on the discrepancy between $P^{\text{(off)}}$ and $P^{\text{(on)}}$.
    \item \textbf{Tight Instance-dependent Regret Bounds:} When $P^{\text{(off)}} =  P^{\text{(on)}}$, we establish the tightness of the instance-dependent regret upper bound (\ref{eq:improved_reg}), which is novel in the literature. We further provide matching instance-dependent upper and lower bounds for general $P^{\text{(off)}}, P^{\text{(on)}}$, precisely quantifying the impact of $T_{\text{S}}(a)$ and $V$ on each sub-optimal arm $a$. The tight analysis also represents new contributions to the literature.
    \item \textbf{Tight Instance-independent Regret Bounds:} We establish a pair of instance-independent regret upper and lower bounds, which match each other up to a logarithmic factor. This optimal regret bound quantifies the effect of arbitrary $\{T_{\text{S}}(a)\}_{a \in \mathcal{A}}$ and its heterogeneity by involving the optimum of a novel linear program. Even in the special case when $P^{\text{(off)}} =  P^{\text{(on)}}$, both the upper and lower bounds are novel in the literature.
\end{enumerate}

Finally, we provide numerical experiments in Section \ref{sec:numerical}.

\subsection{Key Insights and Managerial Implications}

Our technical contributions bear several important implications. 
Firstly, our study sheds light on the benefits of offline data in multi-armed bandit settings, where the offline data are potentially distributed differently from the online data. The takeaway message is: the DM can only reap the benefits of offline data, when the DM has non-trivial auxiliary information about the discrepancy between offline and online reward distributions. The necessity of auxiliary information is formalized in our impossibility result in Section \ref{sec:impossibility}. The following thought experiment provides another perspective. Consider multi-armed bandits with possibly biased offline data, but without any auxiliary information. Now, it is intuitive to conduct an exploration to identify whether the offline data helps us accelerate the estimation of the mean rewards. Our takeaway message is that the cost of such an exploration outweighs the benefits of the offline data. By contrast, by harnessing suitable auxiliary information, we can tip the balance and reap the benefits of offline data.

Secondly, offline data improve the performance over the vanilla UCB only when the reward distributions $P^{\text{(off)}}$ and $P^{\text{(on)}}$ are sufficiently close. 
A main contribution of our analysis is the identification of a discrepancy measure between the offline and online distributions for each arm $a$. When this discrepancy is large, the offline data should be disregarded; when it is small, the offline data becomes informative and can be effectively utilized to enhance performance.
In the latter case, MIN-UCB achieves a strictly smaller regret bound than the vanilla UCB. The improvement, captured by a saving term, becomes larger when $T_\text{S}(a)$ for an $a\in {\cal A}$ increases, and $V(a)$ for an $a\in {\cal A}$ decreases. 
A similar phenomenon appears in our instance-independent regret bound. Only when $V(a)$'s are sufficiently small can the minimax regret bound be improved. Moreover, the improvement becomes more pronounced when $\min_{a \in \mathcal{A}} T_{\text{S}}(a)$ is larger and $V(a)$ for each $a$ is smaller, reflecting the intuition that a larger and more reliable offline dataset enables more effective online learning.

Thirdly, when the offline data is informative, factors other than the size of offline dataset can influence algorithm performance.
In particular, the performance depends not only on the value of $V$ and discrepancy between $P^{\text{(off)}}$ and $P^{\text{(on)}}$, but also on the direction of the discrepancy. For example, if $P^{\text{(off)}}$ underestimates the expected reward of a sub-optimal arm $a$, it allows our algorithm to eliminate $a$ faster, even if $P^{\text{(off)}}$ and $P^{\text{(on)}}$ are ``far apart".
Besides, the heterogeneity in offline dataset sizes $\{T_{\text{S}}(a)\}_{a \in \mathcal{A}}$ also plays a critical role. The impact is not determined solely by the minimum, maximum, or average of $\{T_{\text{S}}(a)\}_{a \in \mathcal{A}}$, but rather depends on the full profile of these values in a non-trivial way.

Lastly, when the offline data is informative, our approach yields notable improvements in several important special cases. Specifically, consider the setting where abundant offline data is available for only a subset of arms, while no prior information exists for the remaining arms. In other words, the DM has reliable information on a small set but must make decisions across a much larger action space. Another case arises when the length of online phase $T$ is relatively small compared to the number of arms $K$ and $\{T_{\text{S}}(a)\}_{a \in \mathcal{A}}$, such that the exploration budget is severely limited. In both scenarios, our policy demonstrates explicit and significant improvements over the standard online learning algorithms without access to the offline data. See Section \ref{sec:MINUCB} for more details.

\subsection{Related Works}

\textbf{Online Learning with Offline Data } The study of multi-armed bandits with offline data has attracted increasing attention. \cite{shivaswamy2012multi,banerjee2022artificial} are the most relevant works, analyzing the special case of $P^\text{(on)} = P^\text{(off)}$ and establishing only the instance dependent regret upper bound, whereas we provide both tight instance dependent and independent regret bounds. Online learning with offline data under $P^\text{(on)} = P^\text{(off)}$ is also explored in dynamic pricing \cite{bu2020online} and reinforcement learning \cite{HaiJLVW23,WagenmakerP23}. 
Extensions include sequentially revealed historical data \cite{gur2022adaptive} and clustered data \cite{bouneffouf2019optimal,ye2020combining,tennenholtz2021bandits}.  
\cite{zhang2019warm} is another closely related work on contextual bandits that allows $P^\text{(on)} \neq P^\text{(off)}$. We provide a technical discussion in Appendix \ref{sec:app_disc_zhang2019}.

\textbf{Bayesian Policies} Thompson sampling (TS) and related Bayesian policies are widely used to integrate offline data into online learning. When the online and offline reward distributions are identical, offline data can be used to construct well-specified priors that yield improved regret bounds than \cite{russo2014learning,RussoV16,LiuL16} when the offline data size increases. In contrast, \citep{LiuL16,SimchowitzTKHLDS21} show that TS with misspecified priors may perform worse than state-of-the-art online policies without offline data. Our paper addresses the orthogonal issue of determining whether to incorporate possibly biased offline data in online learning. 

\textbf{Machine Learning under Domain Shift } Our model is closely related to machine learning under domain shift. \cite{SiNZB23} considers offline policy learning with historical data on contextual bandits setting in a changing environment from a distributionally robust optimization perspective. \cite{chen2022data} explores a variant in RL context, whose policy is also endowed with an upper bound on the distributional drift between offline and online models. Further discussion is provided in Appendix \ref{sec:app_disc_chen2022rl}. 
Generally, distributional drift has been extensively studied in supervised learning \cite{CrammerKW08,mansour2009domain,BenDavidBCKPV10} and stochastic optimization \cite{besbesMM22}.

\textbf{Online Learning with Advice } Our work is related to online learning with advice, which considers improving performance guarantees with hints or predictions. The offline dataset in our setting can be viewed as such a hint.
Prior works include multi-armed bandits \cite{WeiL18}, linear bandits \cite{CutkoskyDDZ22}, contextual bandits \cite{wei2020taking}, full feedback settings \cite{steinhardt2014adaptivity,rakhlin2013online,rakhlin2013optimization}, online episodic MDPs \cite{golowich2022can}, and discounted MDPs \cite{feng2019does}. These works do not apply to our setting since they do not refine instance dependent regret bounds. We provide more discussions in Appendix \ref{sec:app_disc_onlineadvice}. Our theme of whether to leverage potentially biased offline data is also related to online model selection  \cite{AgarwalLNS17,pacchiano2020model,lee2021online,cutkosky2021dynamic,PacchianoDG23}. Unlike these black-box approaches, we tailor our decision-making on whether to use offline data with a novel UCB algorithm and auxiliary input $V$, and the latter is not studied in the above-mentioned works. 

\textbf{Comparison to the Conference Version \cite{cheung2024leveraging} } A preliminary version of this manuscript appears in \cite{cheung2024leveraging}. Our work incorporates the following additional contributions compared to \cite{cheung2024leveraging}:

\begin{itemize}
    \item \textbf{Extentions to Combinatorial Bandits.} \cite{cheung2024leveraging} only studies the multi-armed bandit setting. In the current manuscript, we extend our framework to general combinatorial bandits. We propose a new policy, MIN-COMB-UCB, which allows efficient (computationally and regret-wise) online learning even when the number of combinatorial arms can be exponential in the number of basic arms. For MIN-COMB-UCB, we establish an instance-independent regret upper bound, which is strictly smaller than the existing bounds without offline data when $P^{(\text{off})}$, $P^{(\text{on})}$ are close enoguh. For the special case with linear rewards case, we provide tight instance-independent regret bounds and an instance-dependent regret upper bound with a similar ``Saving" term. All these regret bounds rely on novel techniques that combine the standard combinatorial bandit framework with an arbitrary offline dataset.
    \item \textbf{Discussions on Applications.} We discuss applications of our framework in the context of dynamic pricing and social influence maximization, illustrating how our framework can be applied to these real-world scenarios. We demonstrate that leveraging offline data alongside online learning is particularly valuable for improving decision quality and scalability.
    \item \textbf{Numerical Experiments.} We conduct more systematic experiments to examine the effect of discrepancy, the length of online phase $T$, and the size of offline dataset on algorithm performance. The results validate our theoretical findings and show that the bias, direction, and magnitude of the offline dataset all affect performance. Furthermore, the results highlights the importance of both high-quality data (small bias) and effective algorithms (such as our MIN-UCB).
\end{itemize}

\subsection{Notation}
We denote ${\cal N}(\mu, 1)$ as the Gaussian distribution with mean $\mu$ and variance 1. We abbreviate ``identically and independently distributed'' as ``i.i.d.''. The relationship $R\sim P$ means that the random variable $R$ follows the probability distribution $P$.

\section{Model: Multi-armed Bandits with Offline Data}\label{sec:model-mab}

We consider a stochastic $K$-armed bandit model with possibly biased offline data. 
Denote ${\cal A}= \{1, \ldots, K\}$ as the set of $K$ arms. The event horizon consists of a warm-start phase, followed by an online phase. During the warm-start phase, the DM receives a collection of $T_\text{S}(a)\in \mathbb{Z}_{\geq 0}$ i.i.d. samples, denoted $S(a) = \{X_s(a)\}^{T_\text{S}(a)}_{s=1}$, for each $a\in {\cal A}$. We denote the set of offline samples collectively as $S = \{S(a)\}_{a\in {\cal A}}$. We postulate that $X_1(a), \ldots, X_{T_{\text{S}}(a)}(a)\sim P^{\text{(off)}}_a$, the \emph{offline reward distribution} on arm $a$. 

The subsequent online phase consists of $T$ time steps. For $t = 1, \ldots, T$, the DM chooses an arm $A_t\in {\cal A}$ to pull, which gives the DM a random reward $R_t$. Pulling arm $a$ in the online phase generates a random reward $R(a) \sim P^\text{(on)}_a$, dubbed the \emph{online reward distribution} on arm $a$. We emphasize that $P^{\text{(on)}} = \{P^{\text{(on)}}_a\}_{a\in {\cal A}}$ needs not be the same as $P^{\text{(off)}}= \{P^{\text{(off)}}_a\}_{a\in {\cal A}}$. Finally, the DM proceeds to time step $t+1$, or terminates when $t=T$.

The DM chooses arms $A_1, \ldots, A_T$ with a \emph{non-anticipatory policy} $\pi=\{\pi_t\}^\infty_{t=1}$. The function $\pi_t$ maps the observations $H_{t-1}= (S, \{(A_s. R_s)\}^{t-1}_{s=1})$ collected at the end of $t-1$, to an element in $\{x\in \mathbb{R}^{K}_{\geq 0 }: \sum_{a\in {\cal A}}x_a = 1\}$. The quantity $\pi_t(a~|~H_{t-1})\in [0, 1]$ is the probability of $A_t = a$, conditioned on $H_{t-1}$ under policy $\pi$. While our proposed algorithms are deterministic, i.e. $\pi_t(a~|~H_{t-1})\in \{0, 1\}$ always, we allow randomized policies in our regret lower bound results. The policy $\pi$ could utilize $S$, which provides a potentially biased estimate to the online model since we allow $P^{\text{(on)}}\neq P^{\text{(off)}}$.

Altogether, the underlying instance $I$ is specified by the tuple $({\cal A}, \{T_\text{S}(a)\}_{a\in {\cal A}}, P, T)$, where $P = (P^{\text{(off)}}, P^{\text{(on)}})$. The DM only knows ${\cal A}$ and $S$, but not $P$ and $T$, at the start of the online phase.  
For every $a\in {\cal A}$, we assume that both $P^{\text{(off)}}_a, P^{\text{(on)}}_a$ are 1-subGaussian, which is also known to the DM. For $a\in {\cal A}$, denote $\mu^\text{(on)}(a) = \mathbb{E}_{R(a)\sim P^\text{(on)}_a}[R(a)]$, and $\mu^\text{(off)}(a) = \mathbb{E}_{R(a)\sim P^\text{(off)}_a}[R(a)]$. We do not have any boundedness assumption on $\mu^\text{(on)}(a), \mu^\text{(off)}(a)$.

The DM aims to maximize the total reward earned in the online phase. We quantify the performance guarantee of the DM's non-anticipatory policy by its regret. Define $\mu_*^\text{(on)} = \max_{a\in {\cal A}} \mu^{\text{(on)}}(a)$, and denote $\Delta(a) = \mu^\text{(on)}_* - \mu^\text{(on)}(a)$ for each $a\in {\cal A}$.  
The DM aims to design a non-anticipatory policy $\pi$ minimize the regret 
\begin{equation}\label{eq:regret}
    \text{Reg}_T(\pi, P) = T \mu_*^\text{(on)} - \sum_{t=1}^T \mu^\text{(on)}(A_t) = \sum^T_{t=1}\Delta(A_t),
\end{equation}
despite the uncertainty on $P$. 
While (\ref{eq:regret}) involves only $\mu^\text{(on)}$ but not $\mu^\text{(off)}$, the choice of $A_t$ is influenced by both the offline data $S$ and the online data $\{A_s, R_s\}^{t-1}_{s=1}$. 

As mentioned in the introduction and detailed in the forthcoming Section \ref{sec:impossibility}, the DM cannot outperform the vanilla UCB policy with a non-anticipatory policy. Rather, the DM requires information additional to the offline dataset and a carefully designed policy. We consider the following auxiliary input in our algorithm design.

\subsection{Auxiliary Input: Valid Bias Bounds $V$}

We say that $V = \{V(a)\}_{a\in {K}}\in (\mathbb{R}_{\geq 0}\cup \{\infty\})^{|\mathcal{A}|}$ is a valid \emph{bias bound} on an instance $I$ if 
\begin{equation}\label{eq:V}
V(a) \geq |\mu^{\text{(off)}}(a) - \mu^{\text{(on)}}(a) | \quad \text{for each $a\in{\cal A}$.}
\end{equation}
In our forthcoming policy MIN-UCB, the DM is endowed the auxiliary input $V$ in addition to the offline dataset $S$ before the online phase begins. The quantity $V(a)$ serves as an upper bound on the amount of distributional shift from $P^\text{(off)}_a$ to $P^\text{(on)}_a$. The condition (\ref{eq:V}) always holds in the trivial case of $V(a) = \infty$ for all $a\in {\cal A}$, which is when the DM has no knowledge on the relationship between $P^{\text{on}}, P^{\text{off}}$. By contrast, when $V(a) \neq \infty$, the DM has non-trivial knowledge on the difference $ |\mu^{\text{(off)}}(a) - \mu^{\text{(on)}}(a) |$. 

The knowledge of an upper bound such as $V$ is in line with the model assumptions in studies on learning under distributional shift. Similar upper bounds are assumed in the contexts of supervised learning \citep{CrammerKW08}, stochastic optimization \citep{besbesMM22}, offline policy learning for contextual bandits \citep{SiNZB23}, and multi-task bandit learning \citep{WangZSRC21}.
Methodologies for constructing such bounds are also discussed in the literature. For instance, \cite{blanchet2019robust} designs an approach based on several machine learning estimators, such as LASSO. 
\cite{SiNZB23} provides managerial insight on how to estimate such bounds empirically. \cite{chen2022data} constructs them through cross-validation and implements it in a healthcare setting.

\subsection{Application: Dynamic Pricing}

Now we discuss the application of dynamic pricing (\cite{keskin2014dynamic}), which is naturally captured by our model and helps explain the rationale behind introducing offline data.

In the single-product dynamic pricing setup, a retailer dynamically adjusts the posted price of a product for sequentially arriving customers, aiming to maximize total revenue. Specifically, the arm set $\mathcal{A}$ consists of $K$ possible prices. Each customer can either purchase one unit of the product or do nothing. The customer utility values for the product are i.i.d. according to an unknown latent probability distribution $\Gamma$. At each time $t$, one customer arrives, and the DM posts a price $A_t \in \mathcal{A}$. The customer makes a purchase if and only if her/his random utility $U$ is larger than or equal to $A_t$, i.e., $U \ge A_t$. In which case the customer purchases, she/he will pay $A_t$ to the DM. Thus, $R_t = \boldsymbol{1}\{U\ge A_t\}\cdot A_t$ and $\boldsymbol{\mu}^{\text{(on)}}(a) = a \cdot \Pr(U \ge a) $. The DM, however, lacks prior information about the latent utility distribution $\Gamma$ and only observes binary purchase feedback at the posted price.

In many practical scenarios, the DM has access to offline data, past sales observations from related products or different markets. Such data can provide valuable information for learning the demand function and guiding future pricing, even if not perfectly aligned with the current process.
For example, when a retailer plans to sell a product to a new customer group that overlaps substantially (e.g. 90\%) with a previously served group, historical pricing data naturally inform new pricing strategies.
Since only a small fraction of customers differ, the aggregrated price sensitivity of the new group is expected to be similar to that of the previous one, i.e. $P^{\text{(off)}}$ and $P^{\text{(on)}}$ are close in our model. It is therefore reasonable to assume that a valid bias bound $V$ with small $V_{\max}$ exists and can be estimated via machine learning or statistic tools, as discussed previously.

The relevance of such offline data is further reflected in many real-world scenarios.
For example, when a company launches a product in Canada after its release in the United States, the two markets typically share substantial similarities, such as language, culture.
Although there exist certain differences, such as variations in income levels and regional regulations, these differences are relatively minor, and the U.S. sales data can still provide valuable insights for pricing decisions in the Canadian market.
Likewise, when Apple releases a new-generation iPhone, like iPhone 17, the target customers largely overlap with those of the previous model, such as loyal users in Apple ecosystem. While the new model may differ in some features or technology, the customer preferences and price sensitivity remain broadly consistent. 
In both examples, offline data not only exist but also have tangible modeling and algorithmic value, that they can meaningfully accelerate learning and enhance pricing performance in practice.

\section{An Impossibility Result}\label{sec:impossibility}
We illustrate our impossibility result with two instances $I_P, I_Q$. The instances $I_P, I_Q$ share the same arm set ${\cal A} = \{1, 2\},\{T_\text{S}(a)\}_{a\in {\cal A}}$ and horizon length $T$. However, $I_P, I_Q$ differ in their respective reward distributions $P$, $Q$, and $I_P, I_Q$ have different optimal arms. Consider a fixed but arbitrary $\beta\in (0, 1/2)$. Instance $I_P$ has well-aligned reward distributions $P^\text{(off)} = P^\text{(on)}$, where
$$
P^\text{(on)}_1 = {\cal N}(0, 1), \quad P^\text{(on)}_2 = {\cal N}(-T^{-\beta}, 1 ).
$$
In instance $I_P$, we have $\Delta(2) = T^{-\beta}$, and a non-empty offline dataset provides useful hints on arm 1 being optimal. For example, given $T_\text{S}(1) = T_\text{S}(2) \geq 128 (T^{2\beta} - T^{2\beta - \epsilon})\log T$ where $\epsilon \in (0, \beta)$, the existing policies H-UCB and MonUCB achieve an expected regret (see Equation (\ref{eq:improved_reg})) at most $O\left(\frac{\log T}{\Delta(2)} -T_\text{S}(2)\Delta(2)\right) =  O(\log T(T^{\beta} - T^{\beta}+ T^{\beta - \epsilon})) = O(T^{\beta - \epsilon}\log T)$,
which strictly outperforms the vanilla UCB policy that incurs expected regret $ O(\log T / \Delta(2)) = O(T^\beta \log T)$. Despite the apparent improvement, the following claim shows that any non-anticipatory policy (not just H-UCB or MonUCB) that outperforms the vanilla UCB on $I_P$ would incur a larger regret than the vanilla UCB on a suitably chosen $I_Q$. 
\begin{theorem}\label{thm:impossible}
    Let $T_\text{S}(1), T_\text{S}(2)$ be arbitrary. Consider an arbitrary non-anticipatory policy $\pi$ (which only possesses the offline dataset $S$ but not the auxiliary input $V$) satisfies $\mathbb{E}[\text{Reg}_T(\pi, P)]\leq C T^{\beta - \epsilon}\log T$ on instance $I_P$, where $\epsilon \in (0, \beta)$, $C>0$ is an absolute constant, and the horizon $T$ is so large that $C < T^\epsilon / (4\log 
T)$. Set $Q = (Q^\text{(off)}, Q^\text{(on)})$ in the instance $I_Q$ as $Q^\text{(off)} = P^\text{(off)}$, $Q^\text{(on)}_1 = P^\text{(on)}_1$,
\begin{equation*}
    Q^\text{(on)}_2  = {\cal N}\left(\frac{1}{\sqrt{C\log T}T^{\beta - (\epsilon / 2)}} - \frac{1}{T^\beta}, 1\right).
\end{equation*}
The following inequality holds: $\mathbb{E}[\text{Reg}(\pi, Q)]  \geq$
    \begin{equation}\label{eq:claimed_bad_regret}
    \begin{aligned}\frac{T^{1 - \beta}}{4} \cdot e^{-2} - C T^{\beta - \epsilon}\log T  = \Omega(T^{1 - \beta}) > \Omega(\sqrt{T}).
\end{aligned}
\end{equation}
\end{theorem}

We first remark on $I_Q$. Different from $I_P$, arm 2 is the optimal arm in $I_Q$. Thus, the offline data from $Q^\text{(off)}$ provides the misleading information that arm 1 has a higher expected reward. In addition, note that $Q^\text{(off)} = P^\text{(off)}$, so the offline datasets in instances $I_P, I_Q$ are identically distributed. The theorem suggests that no non-anticipatory policy can simultaneously (i) establish that $S$ is useful for online learning in $I_P$ to improve upon the UCB policy, (ii) establish that $S$ is to be disregarded for online learning in $I_Q$ to match the UCB policy, negating the short-lived conjecture in (Q).

Lastly, we allow $T_\text{S}(1), T_\text{S}(2)$ to be arbitrary. In the most informative case of $T_\text{S}(1)= T_\text{S}(2) = \infty$, meaning that the DM knows $P^\text{(off)},Q^\text{(off)}$ in $I_P, I_Q$ respectively, the DM still cannot achieve (i, ii) simultaneously. Indeed, the assumed improvement $\mathbb{E}[\text{Reg}_T(\pi, P)]\leq C T^{\beta - \epsilon}\log T$ limits the number of pulls on arm 2 in $I_P$ during the online phase. We harness this property to construct an upper bound on the KL divergence between on online dynamics on $I_P$, $I_Q$, which in turns implies the DM cannot differentiate between $I_P, I_Q$ under policy $\pi$. The KL divergence argument utilizes a chain rule (see Appendix \ref{sec:app_aux}) on KL divergence adapted to our setting. 

\subsection{Proof for Theorem \ref{thm:impossible}}

We first set up some notation. Consider a non-anticipatory policy $\pi$ and an instance $I$ with reward distribution $P$. We denote $\rho_{P, \pi}$ as the joint probability distribution function on $S, A_1, R_1, \ldots, A_T, R_T$, the concatenation of the offline dataset $S$ and the online trajectory under policy $\pi$ on instance $I$. For a $\sigma(S, A_1, R_1, \ldots, A_T, R_T)$-measurable event $E$, we denote $\Pr_{P, \pi}(E)$ as the probability of $E$ holds under $\rho_{P, \pi}$. For a  $\sigma(S, A_1, R_1, \ldots, A_T, R_T)$-measurable random variable $Y$, we denote   $\mathbb{E}_{P, \pi}[Y]$ as the expectation of $Y$ under the joint probability distribution function $\rho_{P, \pi}$. In our analysis, we make use of $Y$ being $\text{Reg}_T(\pi, P)$ or $N_T(a)$ for an arm $a$. To lighten our notational burden, we abbreviate $\mathbb{E}_{P, \pi}[\text{Reg}_T(\pi, P) ]$ as $\mathbb{E}[\text{Reg}_T(\pi, P) ]$. By the condition $C < T^\epsilon/(4\log T)$, we know that $$\frac{1}{\sqrt{C\log T}T^{\beta - (\epsilon / 2)}} - \frac{1}{T^\beta} >\frac{1}{T^\beta}  > 0,$$ so in the instance with $Q$, arm 2 is the unqiue optimal arm. 
Consider the event $E = \{N_T(2) > T/2\}$. Then $\mathbb{E}[\text{Reg}(\pi, P)] + \mathbb{E}[\text{Reg}(\pi, Q)] $
\begin{subequations}
\begin{align}
\geq &  \frac{1}{T^\beta} \cdot \frac{T}{2} \cdot \Pr_{P, \pi}(E)  \nonumber \\
& + \left[\frac{1}{\sqrt{C\log T}T^{\beta - (\epsilon / 2)}} - \frac{1}{T^\beta}\right]\cdot \frac{T}{2} \Pr_{Q, \pi}(E^C)\nonumber\\
\geq &  \frac{T^{1 - \beta}}{2}\cdot \left[\Pr_{P, \pi}(E) +\Pr_{Q, \pi}(E^C)\right]\nonumber\\
 \geq & \frac{T^{1 - \beta}}{4}\cdot \exp\left[ - \mathbb{E}_{\pi, P}[N_T(2)]\cdot \text{KL}(P^{(\text{on})}_2, Q^{(\text{on})}_2)   \right. \nonumber \\
& \left. - T_S(2) \cdot \text{KL}(P^{(\text{off})}_2, Q^{(\text{off})}_2)\right] \label{eq:by_chain_crucial}\\
= &  \frac{T^{1 - \beta}}{4}\cdot \exp\left[ - \mathbb{E}_{\pi, P}[N_T(2)]\cdot \text{KL}(P^{(\text{on})}_2, Q^{(\text{on})}_2) \right] \label{eq:by_same}.
\end{align}
\end{subequations}
(\ref{eq:by_chain_crucial}) is again by the chain rule in Theorem \ref{thm:chain} in the Appendix. (\ref{eq:by_same}) is because $ P^\text{(off)} = Q^{\text{(off)}}$. It is evident that $\text{KL}(P^{(\text{on})}_2, Q^{(\text{on})}_2) = 1 / (2C T^{2\beta - \epsilon}\log T)$, and by the Claim assumption that $\text{Reg}(\pi,P) = (1/T^{\beta}) \mathbb{E}_{\pi, P}[N_T(2)] \leq C T^{\beta - \epsilon}\log T$, we have
$ \mathbb{E}_{\pi, P}[N_T(2)] \leq C T^{2\beta - \epsilon}\log T$, which leads to 
$$
\mathbb{E}[\text{Reg}(\pi, P)] + \mathbb{E}[\text{Reg}(\pi, Q)] \geq \frac{T^{1 - \beta}}{4} \cdot e^{-2}.
$$
Again by the Theorem assumption $\mathbb{E}[\text{Reg}(\pi,P)] \leq C T^{\beta - \epsilon}\log T$, we arrive at the claimed inequality, and the Theorem is proved.

\section{MIN-UCB: Leverage Offline Data with Auxiliary Information}\label{sec:MINUCB}

We present our MIN-UCB policy in Algorithm \ref{alg:UCB_min}.
MIN-UCB follows the optimism-in-face-of-uncertainty principle.
At time $t> K$, MIN-UCB sets $\min\{\text{UCB}_t(a), \text{UCB}_t^\text{S}(a)\}$ as the UCB on $\mu^\text{(on)}(a)$. While $\text{UCB}_t(a)$ follows \cite{AuerCBF02}, the construction of  $\text{UCB}_t^\text{S}(a)$ in (\ref{eq:alg-minucb-ucbs}) adheres to the following ideas that incorporates a valid bias bound $V$, the auxiliary input.

\begin{algorithm}[htb]
	\caption{Policy MIN-UCB} \label{alg:UCB_min}
	\begin{algorithmic}[1]
	    \State \textbf{Input:} Valid bias bound $V$ on the instance, confidence parameter $\{\delta_t\}^\infty_{t=1}$, offline samples $S$.
        \State For each $a\in {\cal A}$, compute $\hat{X}(a) = \frac{\sum^{T_\text{S}(a)}_{s=1}X_s(a)}{T_\text{S}(a)}$, and initialize $\hat{R}_1(a) = 0$.
        \State At $t = 1, \ldots, K$, pull each arm once, then set $N_{K+1}(a) = 1$ for all $a$.
     \For{$t = K+1,  \ldots, T$}
        \State For all $a \in \mathcal{A}$, compute the vanilla UCB 
        \begin{equation}
            \label{eq:alg-minucb-ucb}
            \text{UCB}_t(a) = \hat{R}_t(a) + \text{rad}_t(a),
        \end{equation}
        where $\text{rad}_t(a) = \sqrt{\frac{2\log(2t / \delta_t)}{N_t(a)}}$.
        \State For all $a \in \mathcal{A}$, compute the warm-start UCB: $\text{UCB}^\text{S}_t(a) =$ 
        \begin{equation}
        \frac{N_t(a) \cdot \hat{R}_t(a) + T_\text{S}(a)\cdot \hat{X}(a)}{N_t(a) + T_{\text{S}}(a)} + \text{rad}^\text{S}_t(a) ,
         \label{eq:alg-minucb-ucbs}
         \end{equation}
         where $\text{rad}^\text{S}_t(a) =$
         \begin{equation}
         \label{eq:alg-minucb-radS}
          \sqrt{\frac{2\log(2t / \delta_t)}{N_t(a)+ T_\text{S}(a)}} + \frac{ T_\text{S}(a)}{N_t(a)+ T_\text{S}(a)}\cdot V(a).
         \end{equation}
         \State\label{eq:alg-minucb-selection} Select
         $
         A_t\in \text{argmax}_{a\in {\cal A}}\{  \min\{
         \text{UCB}_t(a),  \text{UCB}^\text{S}_t(a) 
         \}\}.$
        \State Observe $R_t\sim P^{\text{(on)}}_{A_t}$. For each $a\in {\cal A}$, update $N_{t+1}(a) = N_t(a) + \mathbf{1}(A_t = a)$,  $$\hat{R}_{t+1}(a) = \begin{cases}
    \frac{N_t(a)\cdot \hat{R}_t(a) }{N_t(a) + 1} + \frac{R_t}{N_t(a) + 1}        &   a = A_t\\
    \hat{R}_t(a)  &  a \ne A_t
  \end{cases}.$$
    \EndFor
	\end{algorithmic}
\end{algorithm}

In (\ref{eq:alg-minucb-ucbs}), the sample mean $\frac{N_t(a) \cdot \hat{R}_t(a) + T_\text{S}(a)\cdot \hat{X}(a)}{N_t(a) + T_{\text{S}}(a)}$ incorporates both online and offline data. Then, we inject optimism by adding $\text{rad}^\text{S}(a)$, set in (\ref{eq:alg-minucb-radS}). The square-root term in (\ref{eq:alg-minucb-radS}) reflects the optimism in view of the randomness of the sample mean. The term  $\frac{ T_\text{S}(a)}{N_t(a)+ T_\text{S}(a)}\cdot V(a)$ in (\ref{eq:alg-minucb-radS}) injects optimism to correct the potential bias from the offline dataset. The correction requires the auxiliary input $V$. 

Finally, we adopt $\text{UCB}_t^\text{S}(a)$ as the UCB when the bias bound $V(a)$ is so small and the sample size $T_\text{S}(a)$ is so large such that $\text{UCB}_t(a) > \text{UCB}_t^\text{S}(a)$. Otherwise, $P^\text{(off)}$ is ``far away''  $P^\text{(on)}$, and we follow the vanilla UCB. Our adaptive decisions on whether to incorporate offline data at each time step is different from the vanilla UCB policy that always ignores offline data, or from HUCB and MonUCB that always incorporate all offline data.

We remark on the valid bias bound $V$. Suppose that the DM has no knowledge on $P^{\text{(off)}}, P^{\text{(on)}}$ at the start of the online phase. Then, it is advisable to set $V(a) = \infty$ for all $a$, which specializes MIN-UCB to the vanilla UCB. Indeed, the impossibility result in Section \ref{sec:impossibility} dashes out any hope for adaptively tuning $V$ in MIN-UCB with online data, in order to outperform the vanilla UCB. In contrast, when the DM has a non-trivial 
upper bound $V(a)$ on $\left|\mu^{\text{(off)}}(a) - \mu^{\text{(on)}}(a) \right|$, we show that our algorithm can uniformly outperform the vanilla UCB that direclty ignores the offline data.

We next embark on the analysis on MIN-UCB, by establishing instance-dependent regret bounds in Section \ref{sec:dep} and instance-independent regret bounds in Section \ref{sec:indpt}. Both sections relies on considering the following events of accurate estimations by $\text{UCB}_t(a), \text{UCB}^\text{S}_t(a)$. For each $t$, define $\xi_t = \cap_{a\in {\cal A}} (\xi_t(a) \cap \xi^\text{S}_t(a))$, where $\xi_t(a) = \left\{ \mu^{\text{(on)}}(a) \leq \text{UCB}_t(a)\leq \mu^{\text{(on)}}(a)+ 2 \text{rad}_t(a) \right\}$, and $\xi^\text{S}_t(a) =$
\begin{align*}
& \left\{ \mu^{\text{(on)}}(a) \leq \text{UCB}^\text{S}_t(a)\leq \mu^{(\text{on})}(a) + \text{rad}^\text{S}_t(a) +  \right.\nonumber\\
&\left. \sqrt{\frac{2\log(2 t / \delta_t)}{N_t(a)+ T_\text{S}(a)}}  + \frac{T_{\text{S}}(a) \cdot (\mu^{\text{(off)}}(a) - \mu^{\text{(on)}}(a))}{N_t(a) + T_\text{S}(a)}\right\}.
\end{align*}
While $\xi_t(a)$ incorporates estimation error due to stochastic variations, $\xi^\text{S}_t(a)$ additionally involves the potential bias of the offline data. The following Lemma \ref{lem:conf-event} is proved in Appendix \ref{sec:app_pf_lemma_conf}.
\begin{lemma}\label{lem:conf-event}
    $\Pr(\xi_t) \geq 1 - 2K \delta_t.$
\end{lemma}

\subsection{Instance-Dependent Regret Bounds}\label{sec:dep} 
Our instance-dependent regret bounds depend on the following discrepancy measure 
\begin{equation}\label{eq:omega}
\omega(a) = V(a) + (\mu^{\text{(off)}}(a) - \mu^{\text{(on)}}(a))
\end{equation}
on an arm $a$. Note that $\omega(a)$ depends on both the valid bias bound $V(a)$ and the model parameters $\mu^{\text{(off)}}(a),  \mu^{\text{(on)}}(a).$ By the validity of $V$, we know that  $\omega(a)\in [0, 2 V(a)]$.

\begin{theorem}\label{thm:upper_ins_dep}
    Consider policy MIN-UCB, which inputs a valid bias bound $V$ on the underlying instance $I$ and  $\delta_t = 1/(2Kt^2)$ for $t=1, 2, \ldots$.  We have $\mathbb{E}[\text{Reg}_T(\text{MIN-UCB}, P)] \leq $ 
    \begin{equation} \label{eq:reg_dep_upper}
          O\left(\sum_{a: \Delta(a)>0} \max \left \{ \frac{\log(T)}{\Delta(a)} 
           - \text{Sav}_0(a), \Delta(a) \right \}\right),
    \end{equation}
    where $\sum_{a: \Delta(a)>0}$ is over $\{a\in {\cal A}: \Delta(a)> 0\}$, and 
    \begin{equation}\label{eq:save_0}
       \text{Sav}_0(a) = T_\text{S}(a) \cdot\Delta(a) \cdot \max\left\{ 1 - \frac{\omega(a)}{\Delta(a)}, 0\right\}^2.
    \end{equation}
\end{theorem}
Theorem \ref{thm:upper_ins_dep} is proved in the Appendix \ref{sec:app_upper_ins_dep}. We provide a sketch proof in Section \ref{sec:sketch_dpt}. A full explicit bound is in (\ref{eq:reg_dep_upper_explicit}) in the Appendix. 
Comparing with the regret bound (\ref{eq:vanilla_reg}) by the vanilla UCB, 
MIN-UCB provides an improvement on the regret order bound by the saving term $\text{Sav}_0(a)$. 

First, $\text{Sav}_0(a)\geq 0$ always holds. The case of $\omega(a) <\Delta(a)$ means that the reward distributions $P^{\text{(off)}}_a, P^{\text{(on)}}_a$ are ``sufficiently close''. The offline data on arm $a$ are incorporated by MIN-UCB to improve upon the vanila UCB. By contrast, the case of $\omega(a) \geq \Delta(a)$ means that  $P^{\text{(off)}}_a, P^{\text{(on)}}_a$ are ``far apart'', so that the offline data on arm $a$ are ignored in the learning process, while MIN-UCB still matches the performance guarantee of the vanilla UCB. We emphasize that $\omega(a)$ and $\Delta(a)$ are both latent, so MIN-UCB conducts the above adaptive choice under uncertainty.

Next, when $\omega(a) < \Delta(a)$ and the offline data on $a$ are beneficial, note that $\text{Sav}_0(a)$ is increasing in $T_\text{S}(a)$. This monotonicity adheres to the intuition that more offline data leads to a higher reward when $P^{\text{(off)}}_a, P^{\text{(on)}}_a$ are ``sufficiently close''. The term $\text{Sav}_0(a)$ decreases as $\omega(a)$ increases. A smaller $V(a)$ means a tighter upper bound on $|\mu^{\text{(off)}}(a) - \mu^{\text{(on)}}(a)|$ is known to the DM, which leads to a better peformance.

Interestingly, $\text{Sav}_0(a)$ increases when $\mu^{\text{(off)}}(a) - \mu^{\text{(on)}}(a)$ decreases. Paradoxically, $\text{Sav}_0(a)$ increases with $|\mu^{\text{(off)}}(a) - \mu^{\text{(on)}}(a)|$, under the case of $\mu^{\text{(off)}}(a) < \mu^{\text{(on)}}(a)$. Indeed, the saving term  $\text{Sav}_0(a)$ depends not only on the magnitude of the distributional shift, but also the direction. To gain intuition, consider the case of $\mu^{\text{(off)}}(a) < \mu^{\text{(on)}}(a)$. On $\mu^{\text{(on)}}(a)$, the offline dataset on $a$ provides a pessimistic estimate, which encourages the DM to eliminate the sub-optimal arm $a$, hence improves the DM's performance. 

Lastly, we compare to existing baselines. By setting $V(a) = \infty$ for all $a$, the regret bound (\ref{eq:reg_dep_upper}) reduces to (\ref{eq:vanilla_reg}) of the vanilla UCB. In the case when $P^\text{(on)} = P^{(\text{off})}$ and it is known to the DM, setting $V(a) = 0$ for all $a$ reduces the regret bound (\ref{eq:reg_dep_upper}) to (\ref{eq:improved_reg}). In both cases we achieve the state-of-the-art regret bounds, while additionally we quantify the impact of distributional drift to the regret bound. Finally, the ARROW-CB by \cite{zhang2019warm} can be applied to our model. While ARROW-CB does not require any auxiliary input, their regret bound is at least $\Omega(K^{1/3}T^{2/3})$. ARROW-CB provides improvment to purely online policies in the contextual bandit setting, as discussed in Appednix \ref{sec:app_disc_zhang2019}.

We next provides a regret lower bound that nearly matches (\ref{eq:reg_dep_upper}). To proceed, we consider
a subset of instances with bounded distributional drift, as defined below.
\begin{definition}\label{def:I_V}
Let $V = (V(a))_{a\in {\cal A}}\in \mathbb{R}_{\geq 0}^K$. We define ${\cal I}_{V} = \{I: |\mu^\text{(on)}(a) - \mu^\text{(off)}(a)|\leq V(a)\text{ for all $a\in {\cal A}$}\}$. 
\end{definition}
Knowing a valid bias bound $V$ is equivalent to knowing that the underlying instance $I$ lies in the subset ${\cal I}_V$, leading to the following which is equivalent to Theorem \ref{thm:upper_ins_dep}:
\begin{corollary}
For any $V \in \mathbb{R}_{\geq 0}^K$, the MIN-UCB policy with input $V$ and $\delta_t = 1/(2K t^2)$ satisfies $\mathbb{E}[\text{Reg}_T(\text{MIN-UCB}, P)] \leq (\ref{eq:reg_dep_upper})$, for all $I\in {\cal I}_V$.
\end{corollary}

Our regret lower bounds involve instances with Gaussian rewards.
\begin{definition}
An instance $I$ is a Gaussian instance if $P^\text{(on)}_a,P^\text{(off)}_a$ are Gaussian distributions with variance one, for every $a\in {\cal A}.$
\end{definition}
Similar to existing works, we focus on consistent policies.
\begin{definition}
For $C>0, p\in (0, 1)$ and a collection ${\cal I}$ of instances, a non-anticipatory policy $\pi$ is said to be $(C, p)$-consistent on ${\cal I}$, if for all $I\in {\cal I}$, it holds that $\mathbb{E}[\text{Reg}_{T}(P, \pi)]\leq C\cdot  T^p$.
\end{definition}
For example, the vanilla UCB is $(C, 1/2)$ consistent for some $C$ proportional to $\sqrt{K}$. 

\begin{theorem}\label{thm:dep_lb}
    Let $V\in \mathbb{R}_{\geq 0}^K$ be arbitrary, and consider a fixed but arbitrary Gaussian instance $I \in {\cal I}_V$. For any $(C, p)$-consistent policy $\pi$ on ${\cal I}_V$, we have the following regret lower bound of $\pi$ on $I$:  $\mathbb{E}[\text{Reg}_T(\pi, P)] \ge $
    \begin{align}
    & \sum_{a : \Delta(a) > 0}  \left\{\frac{2(1-p)}{(1+\epsilon)^2} \cdot \frac{\log T}{\Delta(a)} -  \text{Sav}_\epsilon(a) +\kappa_\epsilon(a)  \right\} \label{eq:dep_lb}
    \end{align}
    holds for any $\epsilon\in (0, 1]$, where $\text{Sav}_\epsilon(a) =$
    \begin{equation}\label{eq:save_epsilon}
        T_\text{S}(a)\cdot\Delta(a)\cdot \max\left\{1 - \frac{\omega(a)}{(1+\epsilon)\Delta(a)}, 0\right\}^2,
    \end{equation}
    and $\kappa_\epsilon(a) =  \frac{1}{2(1+\epsilon)^2\Delta(a) }\log\frac{\epsilon \Delta(a)}{8 C}$ is independent of $T$ and $T_\text{S}(a)$.
\end{theorem}
Theorem \ref{thm:dep_lb} is proved in Appendix \ref{sec:app_lower_ins_dep}. Modulo the additive term $\kappa_\epsilon(a)$, the  regret upper bound (\ref{eq:reg_dep_upper}) and lower bound (\ref{eq:dep_lb}) are nearly matching. Both bounds feature the $\Theta((\log T)/\Delta(a))$ regret term that is classical in the literature 
\cite{lattimore2020bandit}. Importantly, the lower bound (\ref{eq:dep_lb}) features the regret saving term $\text{Sav}_\epsilon(a)$, which closely matches the term $\text{Sav}_0(a)$ in the upper bound. Indeed, $\epsilon\in (0, 1]$ is arbitrary, $\text{Sav}_\epsilon(a)$ tends to the term $\text{Sav}_0(a)$ in the regret upper bound (\ref{eq:reg_dep_upper}) as we tend $\epsilon$ to 0. The close match highlights the fundamental nature of the discrepancy measure $\omega(a)$. In addition, the lower bound (\ref{eq:dep_lb}) carries the insight that the offline data on arm $a$ facilitates improvement over the vanilla UCB if and only if $\omega(a) < (1+\epsilon)\Delta(a)$, which is consistent with our insights from the regret upper bound (\ref{eq:reg_dep_upper}) by tending $\epsilon$ to $ 0$. 

Lastly, in the case $P^{\text{(off)}} = P^{\text{(on)}}$, the lower bound (\ref{eq:dep_lb}) reduces to $\Omega(\sum_{a:\Delta(a) > 0} [\frac{\log T}{\Delta(a)} - T_\text{S}(a) \Delta(a) + \kappa_\epsilon(a)])$, which ascertains the near tightness of (\ref{eq:improved_reg}) by HUCB1 and MonUCB. We highlight that a nearly tight regret lower bound in the well-aligned case of $P^{\text{(off)}} = P^{\text{(on)}}$ is novel in the literature.

\subsection{Sketch Proof for Theorem \ref{thm:upper_ins_dep}}\label{sec:sketch_dpt}

Theorem \ref{thm:upper_ins_dep} is a direct conclusion of the following lemma:
\begin{lemma}\label{lemma:crucial_bound}
    Let $a$ be a sub-optimal arm, that is $\Delta(a) > 0$. 
    Conditioned on the event $\xi_t$, if it holds that $N_t(a) > $
    $$
    32 \cdot \frac{\log(4 K t^4 )}{\Delta(a)^2} - T_\text{S}(a) \cdot \max\left\{ 1 - \frac{\omega(a)}{\Delta(a)}, 0\right\}^2,
    $$
    then $A_t \neq a$ with certainty.
\end{lemma}
To relate Theorem \ref{thm:upper_ins_dep} to Lemma \ref{lemma:crucial_bound}, we remark that $\frac{\text{Sav}_0(a)}{\Delta(a)} = T_\text{S}(a) \cdot \max\{ 1 - \frac{\omega(a)}{\Delta(a)}, 0\}^2$. 

\begin{proof}[Proof of Lemma \ref{lemma:crucial_bound}]
We consider two main cases on $\omega(a)$ and $\Delta(a)$:

    \textbf{Case 1: $\omega(a) < \Delta(a)$}. In this case, $P^{\text{(off)}}_a$ and $P^{\text{(on)}}_a$ are sufficiently closed. The case condition implies that $\frac{\omega(a)}{\Delta(a)} \in [0, 1)$ and 
    \begin{equation*}
        \max \left \{ 1 - \frac{\omega(a)}{\Delta(a)}, 0 \right \}^2 = \left(1-\frac{\omega(a)}{\Delta(a)} \right)^2 > 0.
    \end{equation*}
    To proceed. we consider two sub-cases:

    \textbf{Sub-case 1a: $T_\text{S}(a)\cdot ( 1 - \frac{\omega(a)}{\Delta(a)})^2 \geq \frac{16 \log(4K t^4)}{\Delta(a)^2}$.} In this sub-case, $T_{\text{S}}(a)$ is sufficiently large such that the offine data for arm $a$ provides enough information for the algorithm to  eliminate this sub-optimal arm. We focus on analyzing $\text{UCB}_t^{\text{S}}(a)$:
    \begin{equation*}
        \begin{aligned}
            \sqrt{\frac{2\log(2t / \delta_t)}{N_t(a)+ T_\text{S}(a)}} & \leq \sqrt{\frac{2\log(4Kt^4)}{T_\text{S}(a)}} \\
            & \leq \left( 1 - \frac{\omega(a)}{\Delta(a)} \right) \sqrt{\frac{2\log(4Kt^4)}{\frac{16 \log(4K t^4)}{\Delta(a)^2}   
    }} \\
    & =  \left( 1 - \frac{\omega(a)}{\Delta(a)} \right) \cdot \frac{\Delta(a)}{\sqrt{8}}.
        \end{aligned}
    \end{equation*}
The second inequality is by the lower bound $T_\text{S}(a)$ in the case condition in \textbf{Sub-case 1a}. In addition,
    $$
    \frac{ T_\text{S}(a)}{N_t(a)+ T_\text{S}(a)}\cdot \omega(a)  \leq \frac{ \omega(a) }{\Delta(a)} \cdot\Delta(a).
    $$
Consequently, conditioned on event $\xi_t$, $\text{UCB}^\text{S}_t(a) \leq \mu^{\text{(on)}}(a) +$
    \begin{subequations}
    \begin{align}
        &\quad \ 2\cdot\sqrt{\frac{2\log(2 t / \delta_t)}{N_t(a)+ T_\text{S}(a)}} + \frac{T_\text{S}(a)}{T_{\text{S}}(a) + N_t(a)}\cdot \omega(a) \label{eq:subcase2abyconc}\\
        & \leq \mu^{\text{(on)}}(a) +   \left( 1 - \frac{\omega(a)}{\Delta(a)} \right) \cdot \frac{\Delta(a)}{\sqrt{2}} + \frac{ \omega(a) }{\Delta(a)} \cdot\Delta(a)\nonumber\\
        &< \mu^{(\text{on})}_*\leq \min\left\{
         \text{UCB}_t(a_*),  \text{UCB}^\text{S}_t(a_*) 
         \right\}, \label{eq:bysubcase2aagain}
    \end{align}
    \end{subequations}
    The strict inequality in step (\ref{eq:bysubcase2aagain}) is by the sub-case condition that $1 - \frac{\omega(a)}{\Delta(a)} > 0 $. The less-than-equal steps in Lines (\ref{eq:subcase2abyconc}, \ref{eq:bysubcase2aagain}) is by conditioning on the event $\xi_t$. Altogether, $A_t\neq a$ with certainty.
    
    \textbf{Sub-case 1b: $T_\text{S}(a)\cdot ( 1 - \frac{\omega(a)}{\Delta(a)})^2 < \frac{16 \log(4K t^4)}{\Delta(a)^2}$.} In this sub-case, although $P^{\text{(off)}}_a$ and $P^{\text{(on)}}_a$ are sufficiently closed, $T_{\text{S}}(a)$ is not large enough for $\text{UCB}_t^{\text{S}}(a)$ to shrink faster than $\text{UCB}_t(a)$. Thus the offline data does not provide enough evidence for the algorithm to eliminate the sub-optimal arm $a$ more quickly. Then we have $N_t(a) > $
    \begin{equation*}
    \begin{aligned}
    & 32 \cdot \frac{\log(4 K t^4 )}{\Delta(a)^2} - T_\text{S}(a) \cdot \max\left\{ 1 - \frac{\omega(a)}{\Delta(a)}, 0\right\}^2 \\
    \geq &\frac{16 \log(4K t^4)}{\Delta(a)^2}.
    \end{aligned}
    \end{equation*}
    Consequently, we have $$
    \text{rad}_t(a) = \sqrt{\frac{2\log(2t / \delta_t)}{N_t(a)}} <\sqrt{\frac{2\log(4Kt^4)}{ 16 \cdot \frac{\log(4K t^4)}{\Delta(a)^2} }} \leq \frac{\Delta(a)}{\sqrt{8}},
    $$
    leading to $\text{UCB}_t(a)\leq \mu^{\text{(on)}}(a) + 2\cdot \frac{\Delta(a)}{\sqrt{8}} < \mu^{(\text{on})}_*\leq \min\{
         \text{UCB}_t(a_*),  \text{UCB}^\text{S}_t(a_*) 
         \} $,
    where the first and the last inequalities are by the event $\xi_t$, thus $A_t\neq a$ with certainty. 
    
     \textbf{Case 2: $\omega(a) \geq \Delta(a)$}. Then $\max\{ 1 - \frac{\omega(a)}{\Delta(a)}, 0\}^2 = 0$, and $
    N_t(a) > 32 \cdot \frac{\log(4 K t^4 )}{\Delta(a)^2}$. In this case, $P^{\text{(off)}}_a$ and $P^{\text{(on)}}_a$ are far apart. Similar to \textbf{Case 1b}, we arrive at $\text{rad}_t(a) =$
    $$
     \sqrt{\frac{2\log(2t / \delta_t)}{N_t(a)}} <\sqrt{\frac{2\log(4Kt^4)}{ 32 \cdot \frac{\log(4K t^4)}{\Delta(a)^2} }} \leq \frac{\Delta(a)}{4} \le \frac{\Delta(a)}{\sqrt{8}}. 
    $$
    Then same as \textbf{Case 1b}, $\text{UCB}_t(a)\leq \min\{
         \text{UCB}_t(a_*),  \text{UCB}^\text{S}_t(a_*) 
         \}$.
    Altogether, $A_t\neq a$ with certainty. Altogether, all cases are covered and the Lemma is proved.
\end{proof}

\subsection{Instance-Indepedent Regret Bounds}\label{sec:indpt}
Now we analyze the instance-independent bound. 

\begin{theorem}\label{thm:upper_indpt}
    Assume $2\leq K\leq T$. Consider policy MIN-UCB, which inputs a valid bias bound $V$ on the underlying instance and  $\delta_t = \delta /(2Kt^2)$ for any $t$, where $\delta\in (0,1)$. With probability at least $1-\delta$,\; $     \text{Reg}_T(\text{MIN-UCB}, P)  =$
    {\small \begin{equation}\label{eq:reg_indpt}
    O \left ( \min \left \{ \sqrt{KT\log\frac{T}{\delta}}, \left(\sqrt{\frac{\log(T/\delta)}{\tau_*}}+ V_\text{max}  \right)\cdot T 
        \right \}
        \right),
    \end{equation}
    }
    where $V_\text{max}= \max_{a\in {\cal A}}V(a)$, and $(\tau_*, n_*)$ is an optimum solution of the following linear program:
    \begin{equation}
        \begin{aligned}
             \text{(LP): }\max_{\tau,n} \quad & \tau \\
            \text{s.t.} \quad & \tau \leq T_{\text{S}}(a) + n(a) \quad \forall a \in \mathcal{A},\\
            & \sum_{a\in {\cal A}}n(a) =T,\\
            & \tau\geq 0, n(a)\geq 0\quad\;\; \forall a \in \mathcal{A}.
        \end{aligned}
        \label{eq:indpt-LP}
    \end{equation}
\end{theorem}
In the minimum in (\ref{eq:reg_indpt}), the first term corresponds to the regret due to the confidence radii $\{\text{rad}_t(A_t)\}^T_{t=K+1}$ in the vanilla UCB. The second term corresponds to the regret due to the confidence radii $\{\text{rad}^\text{S}_t(A_t)\}^T_{t=K+1}$ in the warm-start $\text{UCB}^\text{S}$. The minimum corresponds to the fact that MIN-UCB uses the smaller of the two UCBs to be the surrogate for the true mean reward.

We make several observations on the regret bound (\ref{eq:reg_indpt}). First, observe that
\begin{equation*}
    \tau_* \geq T / K + \min_{a \in \mathcal{A}} T_{\text{S}}(a).
\end{equation*}
Indeed, the solution $(\bar{\tau}, \bar{n})$ defined as $\bar{n}(a) = T/K$ for all $a$ and $\bar{\tau} = T/K + \min_{a \in \mathcal{A}} T_{\text{S}}(a)$ is feasible to the LP Program (\ref{eq:indpt-LP}). Thus we have $T / \sqrt{\tau_*} = o (\sqrt{KT} )$. Furthermore, when $V_\text{max} = o(\sqrt{K/T})$, we have $\text{Reg}_T(\text{MIN-UCB}, P) = o(\sqrt{KT\log(T/\delta)})$, and $\text{Reg}_T(\text{MIN-UCB}, P) = O(\sqrt{KT\log(T/\delta)})$ otherwise. This observation demonstrates that only when $V_\text{max} = o(\sqrt{K/T})$, the offline data is informative and able to enhance algorithm performance in terms of minimax regret bound. In this case, the improvement becomes more substantial when $V_\text{max}$ is smaller and $\tau_*$ is larger, reflecting that intuition that a larger and more informative offline dataset can better facilitate online learning.

Second, the second term in (\ref{eq:reg_indpt}) requires a new analysis, as we need to determine the worst-case regret under the heterogeneity in the sample sizes $\{T_\text{S}(a)\}_{a\in {\cal A}}$. We overcome the heterogeneity by a novel linear program in (\ref{eq:indpt-LP}). The constraints in the LP are set such that $\{n_*(a)\}_{a\in {\cal A}}$ represents a worst-case choice of  $\{\mathbb{E}[N_T(a)]\}_{a\in {\cal A}}$ that maximizes the instance-independent regret when $P^\text{(on)} = P^\text{(off)}$. We incorporate the auxiliary decision variable $\tau$ in (\ref{eq:indpt-LP}). A non-anticipatory incurs at most $O(1/\sqrt{\tau_*})$ regret per time round under the worst-case realization of the number of arm pulls.

Third, we discuss two special cases to better illustrate the $\tau_*$ and regret bound. When $T_\text{S}(a) = T_{\text{S}}$ for all $a$ for some fixed integer $T_{\text{S}} \ge 0$, it can be verified that $\tau_* = (T/K)+T_{\text{S}}$ is the optimal solution to the LP (\ref{eq:indpt-LP}), with $n_*(a) = T/K$ for all $a$, meaning that
\begin{equation*}
    \frac{T}{\sqrt{\tau_*}} = \sqrt{KT}\cdot \sqrt{\frac{T}{T + K T_{\text{S}}}}\leq \sqrt{KT}.
\end{equation*}
More generally, consider the case where $T_{\text{S}}(a) = T_{\text{S}} > 0$ for $K_0$ arms, with $0 < K_0 < K$, and $T_{\text{S}}(a) = 0$ for the remaining $K-K_0$ arms. In this case, it can be verified that
\begin{equation*}
    \tau_* = \begin{cases}
        \frac{T}{K - K_0} & T_{\text{S}}(K - K_0) > T,\\
        \frac{T + T_{\text{S}} K_0}{K} & \text{otherwise},
    \end{cases}
\end{equation*}
which implies
\begin{equation*}
    \frac{T}{\sqrt{\tau_*}} = \begin{cases}
        \sqrt{(K - K_0) T} & T_{\text{S}}(K - K_0) > T,\\
        \sqrt{KT} \cdot \sqrt{\frac{T}{T + T_{\text{S}}K_0}} & \text{otherwise.}
    \end{cases}
\end{equation*}
This expression demonstrates how the size of offline dataset influences the regret bound. This scenario also represents a setting in which the offline data is concentrated in a small subset of arms, while the remaining arms have little or no prior information. Our result demonstrates that our MIN-UCB can adaptively explores arms with limited or no historical data, thereby reducing the exploration cost and improving the regret bound.

In addition, when $T = \Theta(K)$ and $\min_{a \in \mathcal{A}} T_{\text{S}}(a) > \Omega(T)$, then $\tau_* \ge \Omega(T)$, and the standard term $\sqrt{KT} = \Theta(T)$, while $T / \sqrt{\tau_*} = o(\sqrt{T})$. Consequently, when $V_{\max} = o(1)$, our algorithm achieves a regret bound that is strictly better than the standard UCB $O(\sqrt{KT \log(T/\delta)})$. This result illustrates how offline data can significantly enhance online learning performance when the length of online phase $T$ is relatively small, i.e. when there is limited opportunity to explore all the arms thoroughly during the online phase. In such settings, it is critical to effectively utilize informative offline data, and our approach is the one that is capable of leveraging such informative offline data to improve learning performance and reduce the regret.

Despite the non-closed-form nature of $\tau_*$, we demonstrate that $\tau_*$ is fundamental to a tight regret bound by the following regret lower bound, which matches (\ref{eq:reg_indpt}) up to a logarithmic factor.

\begin{theorem}
    Let $V_\text{max}\in \mathbb{R}_{\geq 0}$ and $\{T_\text{S}(a)\}_{a\in {\cal A}}\in \mathbb{N}^{\cal A}$ be fixed but arbitrary, and let there be $K\geq 2$ arms. Set $V(a) = V_\text{max}$ for all $a\in {\cal K}$.
    For any non-anticipatory policy $\pi$, there exists a Gaussian instance $I\in {\cal I}_V$ with offline sample size $\{T_\text{S}(a)\}_{a\in {\cal A}}$, such that $\mathbb{E}[\text{Reg}_T(\pi, P)] = $
\begin{equation*}
    \Omega \left(\min\left\{\sqrt{KT}, \left(\frac{1}{\sqrt{\tau_*}} + V_\text{max}\right)\cdot T\right\} \right), 
\end{equation*}
where $\tau_*$ is the optimum of (\ref{eq:indpt-LP}).
\label{thm-lower-independent}
\end{theorem}
Theorem \ref{thm-lower-independent} is proved in Appendix \ref{sec:app_lower_ins_indep}. We conclude this section by highlighting that, in the case of $P^{\text{(off)}} = P^{\text{(on)}}$ no existing work establishes an instance-independent regret bound of $o(\sqrt{KT\log(T/\delta)})$. By setting $V(a) = 0$ for all $a$, our analysis implies that MIN-UCB achieves a regret of $O(T\sqrt{\log(T/\delta)/\tau_*})$, with matching lower bound of $\Omega(T/\sqrt{\tau_*})$, showing its near-optimality in the well-aligned setting.

\section{MIN-COMB-UCB: Generalization to Combinatorial Bandits}\label{sec:MIN-COMB-UCB} 
We extend our algorithmic framework to the stochastic combinatorial bandit setting, which generalizes the setting in \cite{chen2013combinatorial}. In stochastic combinatorial bandits, a problem instance is described by the tuple $(\mathcal{A},\mathcal{B},P,\{T_{\text{S}}(a)\}_{a \in \mathcal{A}},T)$. The set $\mathcal{A} = [K]$ is the set of base arms, and $\mathcal{B} \subseteq 2^{\mathcal{A}}$ denotes a non-empty collection of feasible actions, where each feasible action is a subset of base arms. We let $m = \max_{A \in \mathcal{A}} |A|$ denote the maximum cardinality of a feasible action in $\mathcal{B}$. We denote a base arm by lower case letter $a$ and a feasible action by upper case letter $A$.

We denote $P = (P^{\text{(off)}} = \{P^{\text{(off)}}_a\}_{a\in {\cal A}},P^{\text{(on)}}= \{P^{\text{(on)}}_a\}_{a\in {\cal A}})$ as the pair of latent probability distributions of the outcomes with the base arms, in offline and online environments respectively. For each base arm $a$, there are $T_{\text{S}}(a)\in \mathbb{Z}_{\geq 0}$ offline samples. 
The quantity $T$ is the length of the horizon. 

Similar to the multi-armed bandit model, the model dynamic consists of a warm-start phase, then an online phase. In the warm-start phase, the DM receives a dataset $S(a) = \{X_s(a)\}_{s=1}^{T_{\text{S}}(a)}$ for each base arm $a \in \mathcal{A}$, where $X_s(a) \sim P^{\text{(off)}}_a$ for $s = 1, \ldots, T_{\text{S}}(a)$ are i.i.d. samples. The online phase follows the standard stochastic combinatorial bandit model with semi-bandit feedback over $T$ steps. To this end, we remark that there are other feedback structures, such as full-information or bandit feedback 
\cite{audibert2011minimax,mannor2011bandits,chen2013combinatorial}. 

In the online phase, at each time $t \in [T]$, the DM chooses a feasible action $A_t \in \mathcal{B}$. Upon playing $A_t$, the DM observes the base-arm-level outcomes $(R_t(a))_{a \in A_t}$ for the base arms belonging to $A_t$, and receives the total reward $\Upsilon_t(A_t)$. 
Each base-arm-level outcome $R_t(a) \sim P^{\text{(on)}}_a$ is a realization drawn from the online distribution associated with base arm $a$. 
The $\Upsilon_t(A_t)$ is a non-negative random variable representing the total reward obtained when action $A_t$ is played. This reward depends on the problem instance, the chosen action $A_t$, and the outcome $R_t(a)$ of each base arm $ a \in A_t$. 
A simple but common example is the case of linear rewards, where $\Upsilon_t(A_t) = \sum_{a \in A_t} R_t(a)$, but more general nonlinear reward structures are also allowed. Specifically, for each action $A \in \mathcal{B}$, we assume that the expected reward of playing $A$ at time $t$, denoted by $\mathbb{E}[\Upsilon_t(A)]$, depends only on action $A$ and on the vector of expected base-arm rewards $\boldsymbol{u} = (u(a))_{a \in \mathcal{A}}\in \mathbb{R}^K$. We write $\text{r}_{\boldsymbol{u}}(A) = \mathbb{E}[\Upsilon_t(A)]$ for the expected reward of action $A$ given reward vector $\boldsymbol{u}$. In line with the literature (\cite{chen2013combinatorial}), we make two assumptions on $\text{r}_{\boldsymbol{u}}(A)$:
\begin{itemize}
    \item \textbf{Monotonicity}. The expected reward of $A \in \mathcal{B}$ is monotonically nondecreasing with respect to $\boldsymbol{u}$. That is, for any $\boldsymbol{u}_1 \le \boldsymbol{u}_2$ coordinate-wise, we have $\text{r}_{\boldsymbol{u}_1}(A) \le \text{r}_{\boldsymbol{u}_2}(A)$, for all $A \in \mathcal{B}$.
    \item \textbf{Bounded Smoothness}. There exists a strictly increasing function $f(\cdot)$ such that, for any $\boldsymbol{u}_1$, $\boldsymbol{u}_2$ and for any $A \in \mathcal{B}$, we have $| \text{r}_{\boldsymbol{u}_1}(A) - \text{r}_{\boldsymbol{u}_2}(A) | \le f(x)$ for any $x$ such that $x \ge \max_{a \in A} |u_1(a) - u_2(a)|$.
\end{itemize}

The DM selects $A_1,\ldots, A_T$ via a non-anticipatory policy $\pi = \{\pi_t\}^\infty_{t=1}$, which can utilize the offline dataset $S$. The DM knows $\mathcal{A}$, $\mathcal{B}$, $S$, but not $P$, $T$, and aims to maximize the total expected reward in the online phase, i.e. $\mathbb{E}[\sum_{t=1}^T \text{r}_{\boldsymbol{\mu}^{\text{(on)}}}(A_t)]$. As in the multi-armed bandit case, each action $A_t$ is chosen based on offline data and observed outcomes during time $1, \ldots, t-1$ in the online phase, without knowing the true $\boldsymbol{\mu}^{\text{(on)}}$. Denote $\text{r}^*_{\boldsymbol{u}} = \max_{A\in \mathcal{B}} \text{r}_{\boldsymbol{u}}(A)$ and $A^*_{\boldsymbol{u}} = \mathop{\arg \max}_{A\in \mathcal{B}} \text{r}_{\boldsymbol{u}}(A)$ as the optimal value and the corresponding action under the reward vector $\boldsymbol{u}\in \mathbb{R}^K$. Following the standard approach in the existing works (\cite{chen2013combinatorial,wang2017improving}), we assume access to an $(\alpha,\beta)$-optimization oracle, where $\alpha, \beta \in (0,1]$. Specifically, given the reward vector $\boldsymbol{u}\in \mathbb{R}^K$, the optimization oracle $\text{Oracle}(\boldsymbol{u})$ returns an action $A \in \mathcal{B}$ such that $\Pr(\text{r}_{\boldsymbol{u}}(A) \ge \alpha \cdot \text{r}^*_{\boldsymbol{u}}) \ge \beta$. Given such an oracle, it is inappropriate to compare the algorithmic performance directly with the optimal value $\text{r}^*_{\boldsymbol{\mu}^{\text{(on)}}}$. Instead, we compare against the scaled optimal value $\alpha \cdot \beta \cdot \text{r}^*_{\boldsymbol{\mu}^{\text{(on)}}}$. The regret is defined as
\begin{equation}
    \text{Reg}_T(\pi,P) = \sum_{t=1}^T \left ( \alpha \cdot \beta \cdot \text{r}^*_{\boldsymbol{\mu}^{\text{(on)}}} - \mathbb{E}\left[\text{r}_{\boldsymbol{\mu}^{\text{(on)}}}(A_t) \right ]\right).
    \label{eq:combi-bandit-regret-def}
\end{equation}
The assumptions on monotonicity, bounded smoothness, and optimization oracle are standard in related works, even though alternative formulations exist that may slightly differ from ours. While such differences may lead to minor variations in the final form of regret bounds, they do not affect the core insights or conclusions of our analysis. 

Finally, as in the multi-armed bandit setting, we assume access to the auxiliary input $V = \{V(a)\}_{a \in \mathcal{A}}$, defined in (\ref{eq:V}). We then present the generalized policy MIN-COMB-UCB in Section \ref{sec:alg-min-comb-ucb} and provide its instance-independent regret upper bound. Section \ref{sec:analysis-comb-semi} discusses the special case of combinatorial bandits with linear rewards, where we establish a tight and more benign instance-independent bound and an additional instance-dependent bound.

\subsection{Design and Analysis for Policy MIN-COMB-UCB}

\label{sec:alg-min-comb-ucb}

The policy MIN-COMB-UCB is presented in Algorithm \ref{alg:COMB_UCB_min}. The $t_0$ in Line 7 denotes the time step at which the \textbf{while} loop ends.
There are several key differences between this policy and MIN-UCB. First, MIN-COMB-UCB uses a different initialization procedure to collect initial observations for all base arms during the online phase. Second, after computing $\text{UCB}_t(a)$ and $\text{UCB}^\text{S}_t(a)$, MIN-COMB-UCB constructs the reward vector $(\min\left\{\text{UCB}_t(a),  \text{UCB}^\text{S}_t(a) \right\} )_{a \in \mathcal{A}}$, and passes it to the optimization oracle $\text{Oracle}(\cdot)$ to select the action $A_t$ (\ref{eq:alg-comb-At-rule}). This step generalizes the corresponding step in MIN-UCB (\ref{eq:alg-minucb-selection}). In fact, in the case of multi-armed bandits with $\alpha = \beta = 1$, the action
\begin{equation*}
    A_t \in \text{argmax}_{a\in {\cal A}}\left\{  \min\left\{
         \text{UCB}_t(a),  \text{UCB}^\text{S}_t(a) 
         \right\}\right\},
\end{equation*}
is consistent with the output of $\text{Oracle}\left(\left(\min\left\{\text{UCB}_t(a),  \text{UCB}^\text{S}_t(a) \right\} \right)_{a \in \mathcal{A}}\right)$, matching the selection rule in (\ref{eq:alg-minucb-selection}).

\begin{algorithm}[htb]
	\caption{Policy MIN-COMB-UCB} \label{alg:COMB_UCB_min}
	\begin{algorithmic}[1]
	    \State \textbf{Input:} Valid bias bound $V$ on the instance, confidence parameter $\{\delta_t\}^\infty_{t=1}$, offline samples $S$.
        \State For each $a\in {\cal A}$, compute $\hat{X}(a) = \frac{\sum^{T_\text{S}(a)}_{s=1}X_s(a)}{T_\text{S}(a)}$, and initialize $\hat{R}_1(a) = +\infty$, $N_1(a) = 0$, $t_0=1$.
        \State \textbf{Initialization:} 
        \While{$\exists a \in \mathcal{A}$ such that $N_{t_0}(a) = 0$}
        \State Choose 
        \begin{equation*}
            A_{t_0} \in \text{Oracle} \left( (\hat{R}_{t_0}(a))_{a \in \mathcal{A}} \right).
        \end{equation*}
        \State Observe $\{R_{t_0}(a)\}_{a \in A_{t_0}}$, where $R_{t_0}(a)\sim P^{\text{(on)}}_{a}$. For each $a\in {\cal A}$, update $\hat{R}_{t_0+1}(a) = $
        \begin{equation*}
            \begin{aligned}
                 \begin{cases}
    \frac{N_{t_0}(a)\cdot \hat{R}_{t_0}(a) }{N_{t_0}(a) + 1} + \frac{R_{t_0}(a)}{N_{t_0}(a) + 1}        &   a \in A_{t_0},\\
    \hat{R}_{t_0}(a)  &  a \notin A_{t_0},
  \end{cases} 
            \end{aligned}
        \end{equation*}
        and $N_{t_0+1}(a)  = N_{t_0}(a) + \mathbf{1}(A_{t_0} = a)$, $t_0 = t_0+1$.
        \EndWhile
     \For{$t = t_0+1,  \ldots, T$}
        \State For each base arm $a \in \mathcal{A}$, compute $\text{UCB}_t(a)$ in (\ref{eq:alg-minucb-ucb}) and $\text{UCB}^\text{S}_t(a)$ in (\ref{eq:alg-minucb-ucbs}).
         \State Choose $A_t$ by applying the optimization oracle: $A_t = $
         \begin{equation}
             \text{Oracle} \left(\left(\min\left\{
         \text{UCB}_t(a),  \text{UCB}^\text{S}_t(a) 
         \right\} \right)_{a \in \mathcal{A}} \right).
             \label{eq:alg-comb-At-rule}
         \end{equation}
        \State Observe $\{R_t(a)\}_{a \in A_t}$, where $R_t(a)\sim P^{\text{(on)}}_{a}$. For each $a\in {\cal A}$, update
        \begin{equation*}
            \begin{aligned}
                N_{t+1}(a) & = N_t(a) + \mathbf{1}(a \in A_t),\\
                \hat{R}_{t+1}(a) &= \begin{cases}
    \frac{N_t(a)\cdot \hat{R}_t(a) }{N_t(a) + 1} + \frac{R_t(a)}{N_t(a) + 1}        &   a \in A_t,\\
    \hat{R}_t(a)  &  \text{otherwise.}
  \end{cases}
            \end{aligned}
        \end{equation*}
    \EndFor
    \end{algorithmic}
\end{algorithm}

Now we provide the instance-independent regret upper bound for our MIN-COMB-UCB where $f$ satisfies $f(x) = \gamma \cdot x^\rho$ for some $\gamma > 0$ and $\rho \in (0,1]$.

\begin{theorem}
    Assume $2\leq K\leq T$ and $f(x) = \gamma \cdot x^\rho$ for some $\gamma > 0$ and $\rho \in (0,1]$. Consider MIN-COMB-UCB, which inputs a valid bias bound $V$ on the underlying instance and  $\delta_t = \delta /(2Kt^2)$ for $t=1, \ldots$, where $\delta\in (0,1)$. With probability $\ge 1- (\delta\pi^2 / 6)$, we have $\text{Reg}_T(\text{MIN-COMB-UCB},P) \le O \left ( \gamma \cdot \min \left \{ T_{1},T_{2}  \right \} \right)$,
    where
    \begin{equation*}
        \begin{aligned}
            T_1 & = \left(K \log\left(T / \delta\right) \right)^{\rho / 2} \cdot T^{1 - \rho / 2},\\
            T_2 & = \left( V_{\max}^{\rho} + (\log(T/ \delta))^{\rho/2} \cdot \tau_*^{-\rho/2} \right) \cdot T, 
        \end{aligned}
    \end{equation*}
    and $V_\text{max}= \max_{a\in {\cal A}}V(a)$, and $(\tau_*, n_*)$ is an optimum solution of the linear program (\ref{eq:indpt-LP}).
    \label{thm:comb-gen-instance-independent-upper-bound}
\end{theorem}

Theorem \ref{thm:comb-gen-instance-independent-upper-bound} is proved in Appendix \ref{sec:app-pf-comb-upper-indep}. Its proof shares a similar framework to that of Theorem \ref{thm:upper_indpt}, but requires different techniques due to the nonlinearity of $\rho$ and the bounded smoothness assumption. 
Similarly to the multi-armed bandit case, $T_1$ arises from $\{\text{rad}_t(A_t)\}^T_{t=t_0+1}$ and corresponds to the regret bound of the CUCB algorithm without offline data (\cite{chen2013combinatorial}). $T_2$ is derived from $\{\text{rad}_t^{\text{S}}(A_t)\}^T_{t=t_0+1}$. Hence, the regret bound in (\ref{thm:comb-gen-instance-independent-upper-bound}) is always no worse than that in (\cite{chen2013combinatorial}), and the former is strictly better when $T_2<T_1$, representing the regime when the offline data are beneficial to online learning. Furthermore, in the multi-armed bandit special case, both $\gamma$ and $\rho$ can be set to 1, and the bound recovers Theorem \ref{thm:upper_indpt}.

\subsection{Special Case: Combinatorial Bandits with Linear Rewards}

\label{sec:analysis-comb-semi}

We further discuss the special setting of the combinatorial bandits with linear rewards. In this setting, the total reward $R_t$ for action $A_t$ is simply equal to the sum of the rewards of the individual arms, i.e. $R_t = \sum_{a \in A_t} R_t(a)$. Consequently, the expected reward function satisfies $\text{r}_{\boldsymbol{u}}(A) = \sum_{a \in A} u(a)$. For simplicity, we further assume $\alpha = \beta = 1$ for the optimization oracle. Under this formulation, the monotonicity assumption is satisfied trivially, and the bounded smoothness assumption holds with $f(x) = m x$. 
Applying Theorem \ref{thm:comb-gen-instance-independent-upper-bound}, we directly obtain the following regret upper bound:
{
\small
\begin{equation}
    O \left (\min \left \{ m \sqrt{K T \log\left(T / \delta\right)} , \left(\sqrt{\frac{\log(T/ \delta)}{\tau_*}} + V_{\max} \right) \cdot m T 
        \right \}
        \right).
    \label{eq:comb-semi-upper-indpt-nottight}
\end{equation}
}
However, we can obtain the following improved instance-independent regret upper bound:

\begin{theorem}
    Assume $2\leq K\leq T$. Consider MIN-COMB-UCB on combinatorial bandits with linear rewards model, which inputs a valid bias bound $V$ on the underlying instance and  $\delta_t = \delta /(2Kt^2)$ for $t=1, \ldots$, where $\delta\in (0,1)$. With probability $\ge 1- (\pi^2\delta / 6)$, $\text{Reg}_T(\text{MIN-COMB-UCB},P) \le $
    {\small
    \begin{equation}\label{eq:comb-semi-upper-indpt}
     O \left ( \min \left \{ \sqrt{mKT\log(T / \delta)}, \left(\sqrt{\frac{\log(T / \delta)}{\tau_*^{\text{C}}}}+ V_\text{max}  \right)\cdot mT 
        \right \}
        \right),
    \end{equation}
    }
    where $V_\text{max}= \max_{a\in {\cal A}}V(a)$, and $(\tau_*^{\text{C}}, n_*^{\text{C}})$ is an optimum solution of the following linear program:
    \begin{equation}
        \begin{aligned}
            \max_{\tau,n} \quad & \tau \\
            \text{s.t.} \quad & \tau \leq T_{\text{S}}(a) + n(a) \quad \forall a \in \mathcal{A},\\
            & \sum_{a\in {\cal A}}n(a) =mT,\\
            & \tau\geq 0, n(a)\geq 0\quad\;\; \forall a \in \mathcal{A}.
        \end{aligned}
        \label{eq:comb-inst-indpt-LP}
    \end{equation}
    \label{thm:comb-instance-independent-upper-bound}
\end{theorem}

Theorem \ref{thm:comb-instance-independent-upper-bound} is proved in Appendix \ref{sec:app-pf-comb-semi-upper-indep}, following the same line as Theorem \ref{thm:upper_indpt}.
The bound in (\ref{eq:comb-semi-upper-indpt}) is tighter than the one (\ref{eq:comb-semi-upper-indpt-nottight}), since the first term in (\ref{eq:comb-semi-upper-indpt}) saves a $\sqrt{m}$ factor and the second term satisfies $\tau_*^{\text{C}} \ge \tau_*$. The improvement comes from the specific structure of linear rewards, with further discussions in Appendix \ref{sec:app-pf-comb-semi-upper-indep}. Similarly, the first term in $\min$ corresponds to the instance-independent bound for CombUCB1 algorithm without offline data (\cite{kveton2015tight}). Thus, the bound improves upon it by introducing $\{\text{rad}_t^{\text{S}}(A_t)\}^T_{t=t_0+1}$ to incorporate the benefit from offline data.
We further show that this bound is tight up to a logrithmic factor.

\begin{theorem}
    Let $V_\text{max}\in \mathbb{R}_{\geq 0}$ and $K$ and $m$ satisfy $K/m$ is an integer. Set $V(a) = V_\text{max}$ for all base arm $a\in \mathcal{A}$.
    For any non-anticipatory policy $\pi$, there exists a Gaussian combinatorial bandits with linear rewards instance with offline sample size $\{T_\text{S}(a)\}_{a\in {\cal A}}$ satisfiy $T_{\text{S}}(\ell + (j-1)m) = T_{\text{S},j}$ for $j = 1,\cdots,K/m$ and $\ell = 1,\cdots,m$ for some integers $T_{\text{S},j}$, such that $\mathbb{E}[\text{Reg}_T(\pi, P)] =$
\begin{equation*}
     \Omega \left(\min\left\{\sqrt{mKT}, \left(\frac{1}{\sqrt{\tau_*^{\text{C}}}} + V_\text{max}\right)\cdot mT\right\} \right), 
\end{equation*}
where $\tau_*$ is the optimum of (\ref{eq:comb-inst-indpt-LP}).
\label{thm:comb-instance-independent-lower-bound}
\end{theorem}

Theorem \ref{thm:comb-instance-independent-lower-bound} is proved in Appendix \ref{sec:app_comb_semi_lower_ins_indep}. Moreover, we also derive the instance-dependent regret upper bound. Before proceeding, for each action $A$, we define $\Delta(A) = \sum_{a \in A_*} \mu^{\text{(on)}}(a) - \sum_{a \in A} \mu^{\text{(on)}}(a)$ 
as its optimality gap, and $ \Delta_{\min}(a) = \min_{A \in \mathcal{B}: a \in A,\Delta(A) > 0} \Delta(A)$
as the minimum sub-optimality gap for $a \in \mathcal{A} \backslash A^*_{\boldsymbol{\mu}^{\text{(on)}}}$. Together with the discrepancy measure $\{\omega(a)\}_{a \in \mathcal{A}}$ defined in (\ref{eq:omega}), our instance-dependent regret upper bound is presented as follows:

\begin{theorem}
    Set $\delta_t = \frac{1}{2Kt^2}$, then $\mathbb{E}\left[\text{Reg}_T(\text{MIN-COMB-UCB},P) \right] 
        \le$
    {\small
    \begin{equation}
        \begin{aligned}
        O \left( \sum_{a \in a \in \mathcal{A} \backslash A_*} \max \left \{\frac{m\log(T)}{\Delta_{\min}(a)} - \text{Sav}^{(\text{Com})}(a), \Delta_{\max}\right \}\right).
        \end{aligned}
        \label{eq:comb-semi-dependent-upper}
    \end{equation}
    }
    where $\text{Sav}^{(\text{Com})}(a) =$
    \begin{equation*}
        m\cdot T_{\text{S}}(a) \cdot \Delta_{\min}(a) \cdot \max \left\{1 - \frac{\omega(a)}{\Delta_{\min}(a)},0 \right \}^2,
    \end{equation*}
    and $\Delta_{\max} = \max_{A \in \mathcal{B}} \Delta(A)$.
    \label{thm:comb-instance-dependent-upper-bound}
\end{theorem}

Theorem \ref{thm:comb-instance-dependent-upper-bound} is proved in Appendix \ref{sec:app-pf-comb-semi-upper-dep}. The instance-dependent bound for combinatorial bandits with linear rewards, achieved by the near-optimal algorithm CombUCB1 without offline data (\cite{kveton2015tight} Theorem 5), is
\begin{equation*}
    O \left( \sum_{a \in a \in \mathcal{A} \backslash A_*} \frac{m\log(T)}{\Delta_{\min}(a)} \right).
\end{equation*}
In comparison, our bound in (\ref{eq:comb-semi-dependent-upper}) introduces an additional saving term $\text{Sav}^{(\text{Com})}(a) \ge 0$, which improves regret whenever useful offline information is available.
The bound (\ref{eq:comb-semi-dependent-upper}) can be viewed as a generalization of the instance-dependent bound for multi-armed bandits in Theorem \ref{thm:upper_ins_dep}. The saving term shares similar insights with those in Theorem \ref{thm:upper_ins_dep}.
Moreover, when $m = 1$ and $\Delta_{\min}(a) = \Delta(a)$ for each sub-optimal $a$, the bound (\ref{eq:comb-semi-dependent-upper}) reduces exactly to the bound to the multi-armed bandit bound in Theorem \ref{thm:upper_ins_dep}, demonstrating the generality of our framework.

\subsection{Application: Social Influence Maximization}

\textcolor{black}{In the social influence maximization problem (\cite{kempe2003maximizing,lei2015online,chen2013combinatorial,chen2016combinatorial,chen2020optimization}), the objective is to select a small set of seed nodes in a social network to maximize the spread of influence.} The network is represented as a weighted directed graph $\mathcal{G} = (\mathcal{V},\mathcal{E})$, where $\mathcal{V}$ is a finite set of nodes and $\mathcal{E}$ is a set of directed edges, respectively. Each edge $(u,v) \in \mathcal{E}$ is associated with an unknown propagation probability $p(u,v)$. Given a seed set $S \subseteq \mathcal{V}$, the diffusion process proceeds as follows. For simplicity, we consider the ``one-step" diffusion model: each node $u \in S$ independently attempts to activate its inactive out-neighbors $v$ with probability $p(u,v)$. The reward of $S$ is defined as the total number of activated nodes. \textcolor{black}{This formulation follows the discussions in \cite{chen2016combinatorial}. For more details, readers can refer to \cite{chen2016combinatorial}.}

This problem can be formulated by the combinatorial bandit framework. The propagation probabilities $(p(u,v))_{(u,v) \in \mathcal{E}}$ are unknown and must be learned through repeated selections of seed sets while maximizing the total reward.
In this formulation, each base arm corresponds to an edge $(u,v)$, representing a potential channel between two users. Each feasible action refers to a group of edges induced by a selected seed set $S$. The expected reward of an action is the expected number of nodes activated through the seed set.
The bounded smoothness function can be set as $f(x) = |\mathcal{V}| |\mathcal{E}| x$. 

In practice, it is often reasonable to assume the availability of offline data in this setting. Historical social media records provide rich information about user-to-user influence. 
For example, social media stars or key opinion leaders often have extensive past records on how their influence propagates to followers, and such patterns tend to remain relatively stable over time, so we can reasonably expect a small distribution shift bound $V_{\max}$.
Therefore, access to offline data can substantially improve the learning efficiency in the online phase by providing accurate initial estimates of propagation probabilities and accelerating the convergence of the learned influence model.

\section{Numerical Experiments}

\label{sec:numerical}

We conduct numerical experiments to validate our findings and illustrate how $V$, $T$ and $\{T_{\text{S}}(a)\}_{a \in \mathcal{A}}$ affect algorithm performance. We compare MIN-UCB with three benchmarks. The first is vanilla UCB \cite{AuerCBF02} that ignores the offline data (``PURE-UCB''). The second chooses an arm with the largest $\text{UCB}^{\text{S}}_t(a)$ in each round defined in (\ref{eq:alg-minucb-ucbs}), denoted ``UCBS''. The third is MonUCB in \cite{banerjee2022artificial} that incorporates the offline data under $P^{\text{(off)}} = P^{\text{(on)}}$, refered to "MONUCB".

In all experiments, we simulate MIN-UCB and the benchmarks over various problem instances. Offline and online rewards for each arm are independently drawn from Gaussian distributions with variance 1. We set $K = 10$, and without loss of generality, let $a = 1$ be the unique optimal arm. For simplicity, $T_{\text{S}}(a) = T_{\text{S}}$ for all $a$ for some integer $T_{\text{S}}>0$ and $V(a) = |\mu^{\text{(off)}}(a) - \mu^{\text{(on)}}(a)|$ for all $a$. The online rewawrds are set as $\mu^{\text{(on)}}(1)=1$ for the optimal arm and $\mu^{\text{(on)}}(a) = 0$ otherwise, implying $\Delta(a) = 1$ for all sub-optimal arms. Under each instance, each algorithm is run independently for 50 trials, and we report the sample means and standard deviations of regret.

\subsection{Effect Of Discrepancy}

We investigate the effect of discrepancy and $V$ on algorithm performance. We fix $T_{\text{S}} = 1000$, $T = 10000$, and consider two experiment groups. In the first group, ``Optimistic Bias", we set $\mu^{\text{(off)}}(1) = \mu^{\text{(on)}}(1) - v$, and $\mu^{\text{(on)}}(a) = \mu^{\text{(on)}}(a) + v$ for $a > 1$, meaning that the offline data overestimate the sub-optimal arms while underestimating the optimal one. In the second group, ``Pessimistic Bias", we set $\mu^{\text{(off)}}(a) = \mu^{\text{(on)}}(a) - v$ for all $a$, meaning that the offline data consistently underestimate rewards. In both groups we vary $v \in \{0.1,0.2,\cdots ,1.0\}$. The results are shown in Figure \ref{figs:regret-V}.

\begin{figure}[htb]
    \centering
	\begin{subfigure}{.41\textwidth}
		\centering
		\includegraphics[width=1\textwidth]{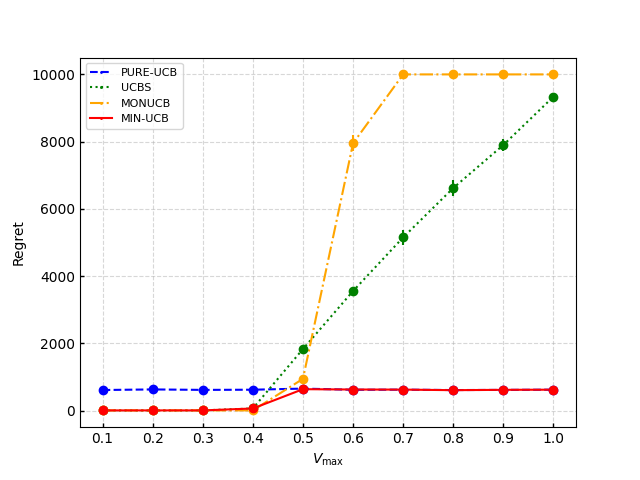}
		\caption{Optimistic Bias}
		\label{fig:reg-V-pos}
	\end{subfigure}
    \begin{subfigure}{.41\textwidth}
		\centering
	    \includegraphics[width=1\textwidth]{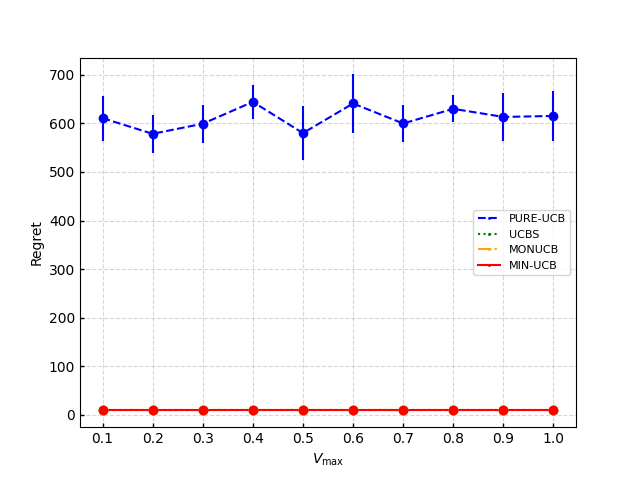}
	    \caption{Pessimistic Bias}
	    \label{fig:reg-V-neg}
	\end{subfigure}
    \caption{Effect of Discrepancy: Both magnitude and direction of bias are important.}
    \label{figs:regret-V}
\end{figure}

We analyze the Optimistic Bias group, with results shown in Figure \ref{fig:reg-V-pos}. 
The regret for PURE-UCB remains constant across all values of $v$, since it ignores the offline data and relies solely on online observations. In contrast, the performance of other three policies varies with $v$, as these approaches incorporate offline information in different ways. 
In this setting, the offline data becomes increasingly misaligned with the online environment as the bias level $v$ increases. Formally, for each sub-optimal arm $a > 1$, the discrepancy measure $\omega(a) = V(a) + \mu^{\text{(off)}}(a) - \mu^{\text{(on)}}(a)= 2v$ for $a$. When $v \le 0.4$, $\omega(a) < \Delta(a)$, the offline data for $a$ are informative. As shown in Figure \ref{fig:reg-V-pos}, UCBS, MONUCB, and our MIN-UCB all outperform PURE-UCB by exploiting informative offline data. However, when $v \ge 0.5$, $\omega(a) \ge \Delta(a)$, the offline data for $a$ is misleading. Then the performance of UCBS and MONUCB, both of which rely heavily and blindly trust the offline data, deteriorates rapidly. In contrast, our MIN-UCB automatically detects the unreliability of the offline data and adapts to pure online learning, achieving regret same as the PURE-UCB. This demonstrates the robustness and adaptivity of our policy. These observations align with our instance-dependent regret bound in Section \ref{sec:MINUCB}, which shows that for each sub-optimal $a$, the offline data are helpful only when $\omega(a) < \Delta(a)$, where the saving term $\text{Sav}_0(a) > 0$ and the performance can be improved.

We now turn to the Pessimistic Bias group, with results shown in Figure \ref{fig:reg-V-neg}. Unlike the Optimistic Bias group, the offline data here consistently underestimate the true rewards for all arms with the same bias $v$. Despite this, the offline data still provide informative signals for identifying the optimal arm. As shown in Figure \ref{fig:reg-V-neg}, UCBS, MONUCB, and MIN-UCB all maintain a remarkably low and stable regret across all $v$, outperforming PURE-UCB. In this setting, the discrepancy measure satisfies $\omega(a) = V(a) + \mu^{\text{(off)}}(a) - \mu^{\text{(on)}}(a)= 0$ for all $a$ and $v$. Consequently, the condition $\omega(a) < \Delta(a)$ always holds. Moreover, since $\omega(a)$, $T_{\text{S}}(a)$ and $\Delta(a)$ are constant across different $v$, our instance-dependent bound implies that the regret of MIN-UCB remains stable and shows a large improvement over PURE-UCB, consistent with the empirical results in Figure \ref{fig:reg-V-neg}.

Ultimately, these observations validate our theoretical findings and illustrate the importance of both the magnitude and direction of bias in determining its role in facilitating online learning.


\subsection{Effect Of $T$ and $T_{\text{S}}$}

We study the effect of $T$ and $\{T_{\text{S}}(a)\}_{a \in \mathcal{A}}$ on algorithm performance. We set $T_{\text{S}} = 10000$, $\mu^{\text{(off)}}(1) = \mu^{\text{(on)}}(1) - v$, and $\mu^{\text{(on)}}(a) = \mu^{\text{(on)}}(a) + v$ for $a > 1$ and $v \in \{0.4,0.5,0.6\}$. For each $v$, we vary $T$ from 50 to 7000. The results, shown in Figure \ref{figs:regret-T}, reveal a clear dependence on offline data quality. When $v$ is small ($v < 0.5$), the offline data are informative, UCBS, MONUCB, and MIN-UCB outperform PURE-UCB. In contrast, when $v$ is large ($v \ge 0.5$), the offline data becomes misleading. In this case, UCBS and MONUCB, which blindly rely on the information from offline data, suffer significant performance deterioration. In constrast, MIN-UCB effectively detects uninformative offline data and defaults to pure online learning, achieving the similar performance as PURE-UCB. These findings validate our regret bounds and highlight the importance of both high-quality offline data and a robust, adaptive policy that can safely and effectively utilize such data, like our MIN-UCB.

\begin{figure}[h]
    \centering
	\begin{subfigure}{.31\textwidth}
		\centering
		\includegraphics[width=1\columnwidth]{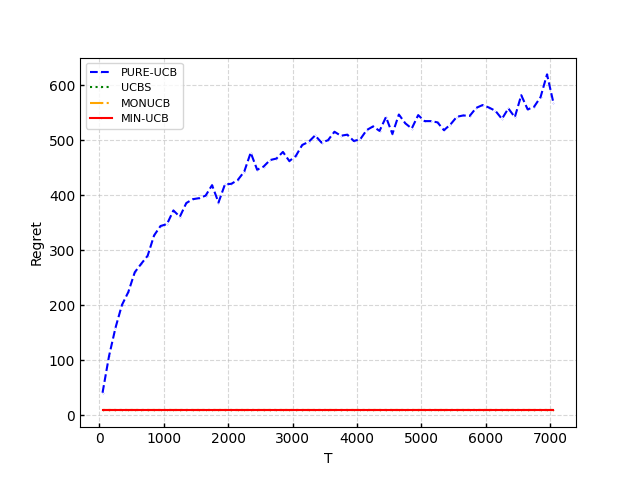}
		\caption{$V_{\max} = 0.4$}
		\label{fig:reg-T-v04}
	\end{subfigure}
    \begin{subfigure}{.31\textwidth}
		\centering
	    \includegraphics[width=1\textwidth]{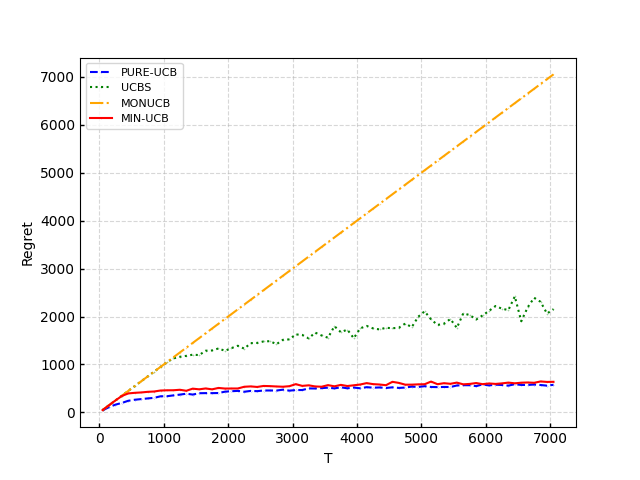}
	    \caption{$V_{\max} = 0.5$}
	    \label{fig:reg-T-v05}
	\end{subfigure}
    \begin{subfigure}{.31\textwidth}
		\centering
	    \includegraphics[width=1\textwidth]{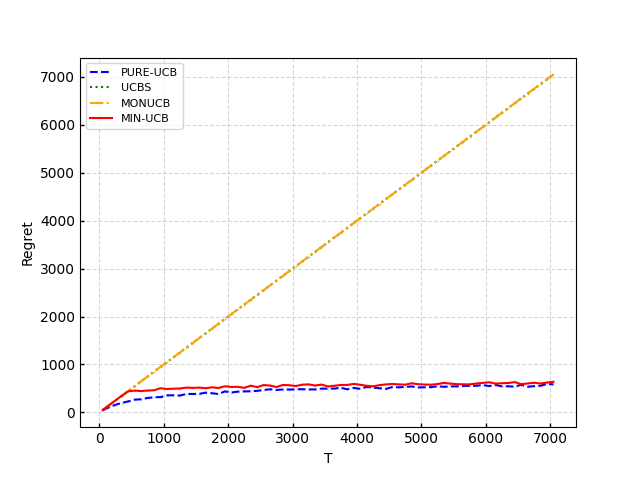}
	    \caption{$V_{\max} = 0.6$}
	    \label{fig:reg-T-v06}
	\end{subfigure}
    \caption{Effect of $T$ and $T_{\text{S}}$: informative offline data can significantly enhance online learning.}
    \label{figs:regret-T}
\end{figure}

\section{Conclusion and Remarks}
We propose novel algorithms to harness possibly biased offline data in the stochastic multi-armed bandit and combinatorial bandit settings. Our analysis demonstrates that our algorithms achieve the best possible regret bounds under the provision of valid bias bounds, which are shown to be necessary for out-performing existing algorithms. 
Through theoretical results and numerical validations, we demonstrate how various factors, such as the biases' magnitudes, the directions of the biases, valid bias bounds, 
affect the usefulness of offline data in online learning. We also demonstrate the applications of our framework, such as dynamic pricing and social influence maximization. Interesting future directions include extending our framework to other online learning models, such as online Markov decision processes and contextual bandits, and investigating whether the benefits and limitations of leveraging offline data remain in other models. It is also intriguing to study our setting with other metrics on the distribution drift.

\bibliography{off}

\newpage
\appendix
\onecolumn

\section{More Detailed Comparisons with Existing Works}
\subsection{Comparing with \cite{zhang2019warm}} \label{sec:app_disc_zhang2019}
\cite{zhang2019warm} consider a stochastic contextual $K$-armed bandit model with possibly biased offline data, which generalizes our setting of stochastic $K$-armed bandits. On one hand, their proposed algorithm ARROW-CB does not require any knowledge on the discrepancy between the offline and online data, while our proposed algorithm requires knowing an upper bound $V(a)$ to $|\mu^\text{(off)}(a) - \mu^\text{(on)}(a)|$ for each $a$. On the other hand, in the case of stochastic $K$-armed bandits, the regret bound of ARROW-CB (see Theorem 1 in \cite{zhang2019warm}) is at least the regret bound of the explore-then-commit algorithm (see Chapter 6 in \cite{lattimore2020bandit}) that ignores all offline data, which in particular does not offer improved regret bound compared to existing baselines that do not utilize offline data. 

In more details, the ARROW-CB algorithm inputs an exploration probability parameter $\epsilon\in (0, 1)$, and a set $\Lambda\subset [0, 1]$ of weighted combination parameters that hedges between ignoring all offline data and incorporating all offline data. In their contextual setting where they compare ARROW-CB  with a benchmark policy class $\Pi$, they establish the following regret upper bound on ARROW-CB on an instance $I$:
\begin{equation}\label{eq:zhang_regret}
    \epsilon T + 3\sqrt{T\log(T|\Pi|)} + \frac{32}{\sqrt{\epsilon}}\cdot \sqrt{KT\log(8T|\Lambda|)} + \min_{\lambda\in \Lambda}\{\text{disc}(\lambda, I)\}. 
\end{equation}
The function $\text{disc}(\lambda, I)$ (see Equation (4) in \cite{zhang2019warm}) is a discrepancy measure on the distance between the offline and online models, with the property that $\text{disc}(\lambda, I)\geq 0$ for all $\lambda, I$. 

On one hand, (\ref{eq:zhang_regret}) provides improvements to the conventional regret bound $O(\sqrt{KT\log(|\Pi|T)})$ in terms of the dependency on $|\Pi|$ when  $\text{disc}(\lambda, I)$ is sufficiently small. On the other hand, the (\ref{eq:zhang_regret}) does not provide any improvement to the conventional regret bounds in the stochastic $K$-armed bandits setting. Indeed, in  $K$-armed setting we have $|\Pi| = K$, where $\Pi$ consists of the policies of always pulling arm $a$, for $a\in {\cal A}$. Ignoring the non-negative term $\min_{\lambda\in \Lambda}\{\text{disc}(\lambda, I)\}$ and optimizing $\epsilon$ in the remaining terms show that the remaining terms sum to at least
$$
K^{1/3}T^{2/3} + 3\sqrt{T\log(TK)} + 32 K^{1/3}T^{2/3} \sqrt{\log T},
$$
which is no better than the explore-then-commit policy's regret of $O(K^{1/3}T^{2/3} \sqrt{\log T})$. 

Finally, one additional distinction is that \cite{zhang2019warm} requires $T_\text{S}(1) = \ldots = T_\text{S}(K)$, while we allow the offline sample sizes $T_\text{S}(1), \ldots, T_\text{S}(K)$ to be arbitrary non-negative integers.

\subsection{Comparing wtih Online Learning with Advice}
\label{sec:app_disc_onlineadvice}
The frameworks of \cite{rakhlin2013online,steinhardt2014adaptivity,WeiL18,WeiLA20} utilize a commonly defined quantity ${\cal E}$ to quantify the error in the predictions, which are provided in an online manner. We focus our discussion on comparison with \cite{WeiL18}. The framework of \cite{WeiL18} provides the following results for adversarial multi-armed bandits with side information. Consider a $K$-armed bandit model, where an arm $a\in {\cal A}$ is associated with loss $\ell_{t}(a)\in [-1, 1]$ at time $t$ for $t\in \{1, \ldots, T\}$, and the loss vectors $\{\ell_{1}, \ldots, \ell_T\}$ over the horizon of $T$ time steps are generated by an oblivious adversary. Before choosing an arm $a_t\in {\cal A}$ at time $t$, the agent is endowed with a prediction $m_t\in [-1 , 1]^K$, where $m_t(a)$ is a prediction on $\ell_t(a)$. They design a version of optimistic online mirror design algorithm that incoporates the predictions $m_1, \ldots, m_T$, and achieves an expected regret bound of
\begin{equation}\label{eq:weiluo}
\mathbb{E}\left[\sum^T_{t=1} \ell_t(a) - \min_{a\in K} \sum^T_{t=1}\ell_t(a)\right] = O(\sqrt{K{\cal E} \log(T)}),
\end{equation}
where the expectation is taken over the internalized randomization of the algorithm, where an arm is randomly chosen each round. The quantity ${\cal E} = \sum^T_{t=1} \|\ell_t - m_t\|^2_\infty = O(T)$ represents the prediction errors. 

We argue that their framework can only guarantees a $O(\sqrt{KT\log(T)})$ in terms of expected regret bound in our setting, since the prediction error term ${\cal E}$ depends on the \emph{realized} losses rather than their means. Indeed, In our stochastic setting (translated to a loss minimization instead of reward maximization model), the loss vectors are $\ell_1, \ldots, \ell_T$ are independent and identically distributed according to a common (but latent) distribution ${\cal D}$. Consider the case when the loss of arm $a$ is distributed according to the Bernoulli distribution with mean $\mu(a)\in [1/4, 3/4]$ for each $a\in {\cal A}$. Even in the case when we have the prediction $m_t(a) = \mu(a)$ for each $t$, namely the predictions reveal the true mean, we still have $(\ell_t(a) - \mu(a))^2\geq 1/16$ for each arm with certainty, leading to ${\cal E} \geq T / 16$. 

The framework of \cite{WeiL18} can only gauranttee a regret bound of $o(\sqrt{KT})$ when ${\cal E} = o(T)$. In our setting, it requires that $\|\ell_t - m_t\|_\infty = o(1)$, meaning that the prediction $m_t$ has to be correlated with, thus contains information of, the realized loss $\ell_t$, for $\Omega(T)$ many time rounds $t$. Different from \cite{WeiL18}, we achieve a regret bound of $o(\sqrt{KT})$ in the presense of accurate offline data, instead of receiving hints about the realized online rewards. A similar conclusion (with a worse regret bound than (\ref{eq:weiluo})) holds when we compare our results with \cite{WeiLA20}, which presents regret bounds on adversarial contextual bandits with predictions, since their regret bounds are also defined in terms of ${\cal E}$ mentioned above.

\subsection{Comparing with \cite{chen2022data}}
\label{sec:app_disc_chen2022rl}
\cite{chen2022data} explores a Finite-Horizon Reinforcement Learning Model incorporating possibly biased historical data, another generalization of our stochastic multi-armed bandit model. Similar to our work, their approach "Data-pooling Perturbed LSVI" (Algorithm 2) requires an upper bound on the discrepancy between offline and online data, defined as $\Delta$ in their Assumption 1. They propose an cross-validation method and illustrate through a case study that a substantial range of $\Delta$ values produce similar performance for their algorithm. However, their approach cannot facilitate the design of an algorithm that achieve our improved regret bounds compared to existing benchmarks that neglects the offline data.

In details, their ``Data-pooling Perturbed LSVI'' algorithm is based on classical Value Iteration (VI). They adapt the VI on online learning setting by randomly perturbing the estimated value-go-functions, and applying least squares fitting to estimate the Q-function. This approach is fundemantally different from the OFU principle we apply in the multi-armed bandit context.
Besides, they follow a different way to combine offline and online data. Specifically, they compute a weighted parameter $\lambda_t$ that balances between online and offline data, via $t$, $\Delta$, and $T_\text{S}^{\min} = \min_{h,s,a} T_{\text{S}}(h,s,a)$, where $T_{\text{S}}(h,s,a)$ refers to the number of offline samples for pair $(h,s,a)$, refering to epoch $h$, state $s$, and action $a$, respectively. This $\lambda_t$ is shown in their equation (10). They derive the following regret upper bound:

\begin{equation}\label{eq:chen_regret}
    \left(1 + 2 p_0 \right)^{-1} H S K \sum_{t=1}^{T / (SK)} \left(\varepsilon_R^{\text{DP}}(t,T_\text{S}^{\min}) + H (\bar{W} + 1)\varepsilon_P^{\text{DP}}(t,T_\text{S}^{\min}) + \varepsilon_V^{\text{DP}}(t,T_\text{S}^{\min}) \right) + 4 H T \delta
\end{equation}

where $p_0$, $\bar{W}$, $\delta$ are their input parameter, $H$ is the length of a horizon, and $S$ is the number of state. $\varepsilon_R^{\text{DP}}(\cdot,\cdot)$, $\varepsilon_P^{\text{DP}}(\cdot,\cdot)$, $\varepsilon_V^{\text{DP}}(\cdot,\cdot)$ are confidence radius defined in their Appendix EC.4.2. They state that when $T_\text{S}^{\min} \ge 1$ and $\Delta$ is small (See their Theorem EC.1), this regret bound is strictly smaller than the case that without combining historical data. 

However, (\ref{eq:chen_regret}) does not offer an explicit closed-form bound, and they do not provide any regret lower bound.
More importantly, their implicit bound solely depends on  $T_\text{S}^{\min} = \min_{h,s,a} T_{\text{S}}(h,s,a)$, which is equivalent to $\min_{a} T_{\text{S}}(a)$ in our model. In contrast, both our instance dependent bound (Theorem \ref{thm:upper_ins_dep}) and instance independent bound (Theorem \ref{thm:upper_indpt}) demonstrate how the difference in $T_S(a)$ among different arms  affects the regret. Thus, their approach cannot deliver the tight regret bounds on and insights into our model.

\section{Proofs for Instance-dependent Regret Upper Bounds}
\subsection{Proof for Lemma \ref{lem:conf-event}}\label{sec:app_pf_lemma_conf}
The proof uses the Chernoff inequality:
\begin{proposition}[Chernoff Inequality]
Let $G_1, \ldots, G_m$ be independent (though not necessarily identically distributed) 1-subGaussian random variables. For any $\delta\in (0, 1)$, it holds that
\begin{align}
   &\Pr\left[\frac{1}{m}\mathbb{E}\left[\sum^m_{i=1}G_i\right] \leq  \frac{1}{m}\sum^m_{i=1}G_i + \sqrt{\frac{2\log (2/\delta)}{m}}\right] \geq 1-\delta/2,\nonumber\\
    &\Pr\left[\frac{1}{m}\sum^m_{i=1}G_i \leq \frac{1}{m}\mathbb{E}\left[\sum^m_{i=1}G_i\right] + \sqrt{\frac{2\log (2/\delta)}{m}}\right] \geq 1-\delta/2.\nonumber
\end{align}
  
\end{proposition}
\proof{Proof of Lemma \ref{lem:conf-event}}
To prove the Lemma, it suffices to prove that $\Pr[\xi^\text{S}_t(a)]\geq 1-\delta_t$. Indeed, the inequality implies $\Pr(\xi_t(a))\geq 1-\delta_t$ by setting $T_\text{S}(a) = 0$, and then a union bound on the failure probabilities of $\xi^\text{S}_t(a), \xi_t(a)$ for $a\in {\cal A}$ establishes the Lemma. Now, we have
\begin{align}
    &\Pr[\mu^{\text{(on)}}(a) \leq \text{UCB}^\text{S}_t(a)] \nonumber\\
    =&  \Pr\left[\mu^{\text{(on)}}(a) \leq 
 \frac{N_t(a) \cdot \hat{R}_t(a) + T_\text{S}(a)\cdot \hat{X}(a)}{N_t(a) + T_{\text{S}}(a)} + \sqrt{\frac{2\log(2 t / \delta_t)}{N^+_t(a)+ T_\text{S}(a)}} + \frac{ T_\text{S}(a)}{N^+_t(a)+ T_\text{S}(a)}\cdot V(a).\right]\nonumber\\
 =&  \Pr\left[\frac{N_t(a)\mu^{\text{(on)}}(a) + T_\text{S}(a) \mu^{\text{(off)}}(a) }{N_t(a) + T_\text{S}(a)} + \frac{T_\text{S}(a) (\mu^{\text{(on)}}(a)  - \mu^{\text{(off)}}(a)) }{N_t(a) + T_\text{S}(a)}\leq \right.\nonumber\\
 &\left. \frac{N_t(a) \cdot \hat{R}_t(a) + T_\text{S}(a)\cdot \hat{X}(a)}{N_t(a) + T_{\text{S}}(a)} + \sqrt{\frac{2\log(2 t / \delta_t)}{N_t(a)+ T_\text{S}(a)}} + \frac{ T_\text{S}(a)}{N_t(a)+ T_\text{S}(a)}\cdot V(a).\right]\nonumber\\
 \geq &\Pr\left[\frac{N_t(a)\mu^{\text{(on)}}(a) + T_\text{S}(a) \mu^{\text{(off)}}(a) }{N_t(a) + T_\text{S}(a)} \leq \frac{N_t(a) \cdot \hat{R}_t(a) + T_\text{S}(a)\cdot \hat{X}(a)}{N_t(a) + T_{\text{S}}(a)}  + \sqrt{\frac{2\log(2 t / \delta_t)}{N_t(a)+ T_\text{S}(a)}}\right] \label{eq:by_Va_1}\\
 \geq & \Pr_{Y_i \sim P^{\text{(on)}}(a)}\left[\frac{n \mu^{\text{(on)}}(a) + T_\text{S}(a) \mu^{\text{(off)}}(a) }{n + T_\text{S}(a)} \leq \frac{\sum_{i=1}^n Y_i + T_\text{S}(a)\cdot \hat{X}(a)}{n + T_{\text{S}}(a)}  + \sqrt{\frac{2\log(2 t / \delta_t)}{n+ T_\text{S}(a)}} \text{\ for $n=1,2, \ldots ,t$}\right] \label{eq:by_coupling_1}\\
 \geq & 1-\delta_t/2\label{eq:by_Hoeffding_1}.
\end{align}
Step (\ref{eq:by_Va_1}) is by the input assumption that $|\mu^\text{(on)}(a) - \mu^\text{(off)}(a)| \leq V(a)$. Step (\ref{eq:by_coupling_1}) is by a union bound over all possible values of $N_t(a)$. Step (\ref{eq:by_Hoeffding_1}) is by the Chernoff inequality.

Next,
\begin{align}
   & \Pr\left( \text{UCB}^\text{S}_t(a)\leq \mu^{(\text{on})}(a)+ \text{rad}^\text{S}_t(a)+ \sqrt{\frac{2\log(2 N^+_t(a) / \delta_t)}{N^+_t(a)+ T_\text{S}(a)}}  + \frac{T_{\text{S}}(a) \cdot (\mu^{\text{(off)}}(a) - \mu^{\text{(on)}}(a))}{N^+_t(a) + T_\text{S}(a)}\right)\nonumber\\
= & \Pr\left( \frac{N_t(a) \cdot \hat{R}_t(a) + T_\text{S}(a)\cdot \hat{X}(a)}{N_t(a) + T_{\text{S}}(a)} \leq \mu^{(\text{on})}(a)+ \sqrt{\frac{2\log(2 N^+_t(a) / \delta_t)}{N^+_t(a)+ T_\text{S}(a)}}  + \frac{T_{\text{S}}(a) \cdot (\mu^{\text{(off)}}(a) - \mu^{\text{(on)}}(a))}{N^+_t(a) + T_\text{S}(a)}\right)\nonumber\\
= & \Pr\left( \frac{N_t(a) \cdot \hat{R}_t(a) + T_\text{S}(a)\cdot \hat{X}(a)}{N_t(a) + T_{\text{S}}(a)} \leq 
\frac{N_t(a)\mu^{\text{(on)}}(a) + T_\text{S}(a) \mu^{\text{(off)}}(a) }{N_t(a) + T_\text{S}(a)} + \sqrt{\frac{2\log(2 N^+_t(a) / \delta_t)}{N^+_t(a)+ T_\text{S}(a)}} \right)\nonumber\\
 \geq & \Pr_{Y_i \sim P^{\text{(on)}}(a)}\left[ \frac{\sum_{i=1}^n Y_i + T_\text{S}(a)\cdot \hat{X}(a)}{n + T_{\text{S}}(a)} \leq  \frac{n \mu^{\text{(on)}}(a) + T_\text{S}(a) \mu^{\text{(off)}}(a) }{n + T_\text{S}(a)}  + \sqrt{\frac{2\log(2 t / \delta_t)}{n+ T_\text{S}(a)}} \text{\ for $n=1,2, \ldots ,t$}\right]  \label{eq:by_coupling_2}\\
 \geq & 1-\delta_t/2\label{eq:by_Hoeffding_2}.
\end{align}
The justifications of (\ref{eq:by_coupling_2}, \ref{eq:by_Hoeffding_2}) are the same as (\ref{eq:by_coupling_1}, \ref{eq:by_Hoeffding_1}) respectively. Altogether, the Lemma is proved.
    
\endproof

\subsection{Proof of instance-dependent regret upper bound for Multi-armed Bandit}\label{sec:app_upper_ins_dep}

\subsubsection{Proof of Theorem \ref{thm:upper_ins_dep}}

Applying Lemma \ref{lemma:crucial_bound}, we get
\begin{align}
    \mathbb{E}[\text{Reg}_T(\pi, T)] & = \sum^T_{t=1} \mathbb{E}\left[(\mu_*^{(\text{on})} - \mu^{(\text{on})}(A_t)) \cdot\left[\mathbf{1}(\xi_t)+\mathbf{1}(\xi^c_t)\right]\right]\nonumber\\
    & \leq \sum^T_{t=1} \mathbb{E}\left[(\mu_*^{(\text{on})} - \mu^{(\text{on})}(A_t)) \cdot\mathbf{1}(\xi_t)\right] +  \Delta_{\text{max}}\sum^T_{t=1} \mathbb{E}\left[\mathbf{1}(\xi^c_t)\right]\nonumber, 
\end{align}
and 
\begin{align}
    \sum^T_{t=1} \mathbb{E}\left[(\mu_*^{(\text{on})} - \mu^{(\text{on})}(A_t)) \cdot\mathbf{1}(\xi_t)\right] & = \sum_{a\in {\cal A}} \Delta(a) \mathbb{E}\left[\sum^T_{t=1} \mathbf{1}(A_t = a)\mathbf{1}(\xi_t)\right]\nonumber\\
    & \leq   \sum_{a\in {\cal A}} \max \left \{ 32 \cdot \frac{\log(4KT^4 )}{\Delta(a)} - T_\text{S}(a)  \cdot\Delta(a) \cdot \max\left\{ 1 - \frac{\omega(a)}{\Delta(a)}, 0\right\}^2, \Delta(a)\right\}\nonumber,
\end{align}
and $\sum^T_{t=1} \mathbb{E}\left[\mathbf{1}(\xi^c_t)\right]\leq \pi^2/6$. Altogether, we arrive at
 \begin{equation} \label{eq:reg_dep_upper_explicit}
       \frac{\pi^2}{6} \Delta_\text{max} + \sum_{a\in {\cal A}:\Delta(a)>0} \max \left \{ 32 \cdot \frac{\log(4KT^4 )}{\Delta(a)} - T_\text{S}(a) \cdot\Delta(a) \cdot \max\left\{ 1 - \frac{\omega(a)}{\Delta(a)}, 0\right\}^2, \Delta(a) \right \},
    \end{equation}
    where $\Delta_\text{max} = \max_{a}\Delta(a)$,  
which proves the Theorem.

\subsection{Proof for instance-dependent bound for Combinatorial Bandits with Linear Rewards}

\label{sec:app-pf-comb-semi-upper-dep}


\subsubsection{Key Lemmas}

\begin{lemma}
    Define event $\mathcal{F}_t$ as
    \begin{equation*}
        \mathcal{F}_t = \left \{\Delta(A_t) > 0,\Delta(A_t) \le \sum_{a \in A_t} \min \left \{2 \sqrt{\frac{2\log(4KT^3)}{N_t(a)}}, 2\sqrt{\frac{2\log(4KT^3)}{N_t(a)+ T_\text{S}(a)}} + \frac{ T_\text{S}(a) \omega(a)}{N_t(a)+ T_\text{S}(a) }\right \} \right \}.
    \end{equation*}
    Then
    \begin{equation*}
        \text{Reg}_T \le O\left(K \Delta_{\max} \right) + \sum_{t=1}^T \Delta(A_t) \boldsymbol{1}\{\mathcal{F}_t\} .
    \end{equation*}
    where $\Delta_{\max} = \max_{A \in \mathcal{B}} \Delta(A)$.
    \label{lem:pf-thm-inst-dept-ft}
\end{lemma}

\proof{Proof for Lemma \ref{lem:pf-thm-inst-dept-ft}}
Notice that $\sum_{a \in A_*} \mu^{\text{(on)}}(a)  - \sum_{a \in A_t} \mu^{\text{(on)}}(a) = \Delta(A_t)$. Clearly, we have
\begin{equation*}
    \text{Regret}_T \le K \Delta_{\max} + \sum_{t=1}^T \Delta(A_t) \boldsymbol{1} \{\xi_t\} + \sum_{t=1}^T \Delta(A_t) \boldsymbol{1} \{\neg\xi_t\}
\end{equation*}
For the third part,
\begin{equation*}
    \sum_{t=1}^T \Delta(A_t) \boldsymbol{1} \{\neg\xi_t\} \le \Delta_{\max} \sum_{t=1}^T 2K \delta_t =  \Delta_{\max} \sum_{t=1}^T \frac{1}{t^2} \le \frac{\pi^2}{6} \Delta_{\max}. 
\end{equation*}

For the second part, by the selection rule for $A_t$ (\ref{eq:alg-comb-At-rule}), we have
\begin{equation*}
    \sum_{a \in A_t} \min\left\{
    \text{UCB}_t(a),  \text{UCB}^\text{S}_t(a) 
    \right\} \ge \sum_{a \in A_*} \min\left\{
    \text{UCB}_t(a),  \text{UCB}^\text{S}_t(a) 
    \right\}.
\end{equation*}
Since
\begin{equation*}
    \begin{aligned}
    & \sum_{a \in A_t} \min\left\{
    \text{UCB}_t(a),  \text{UCB}^\text{S}_t(a) 
    \right\} \\
    \le  & \sum_{a \in A_t} \mu^{\text{(on)}}(a) + \sum_{a \in A_t} \min \left \{2 \sqrt{\frac{2\log(2t / \delta_t)}{N_t(a)}}, 2\sqrt{\frac{2\log(2t / \delta_t)}{N_t(a)+ T_\text{S}(a)}} + \frac{ T_\text{S}(a) \omega(a)}{N_t(a)+ T_\text{S}(a) }\right \},
    \end{aligned}
\end{equation*}
and
\begin{equation*}
    \sum_{a \in A_*} \min\left\{
    \text{UCB}_t(a),  \text{UCB}^\text{S}_t(a) 
    \right\} \ge  \sum_{a \in A_*} \mu^{\text{(on)}}(a),
\end{equation*}
we have
\begin{equation*}
    \begin{aligned}
    \Delta(A_t) & \le  \sum_{a \in A_t} \min \left \{2 \sqrt{\frac{2\log(2t / \delta_t)}{N_t(a)}}, 2\sqrt{\frac{2\log(2t / \delta_t)}{N_t(a)+ T_\text{S}(a)}} + \frac{ T_\text{S}(a) \omega(a)}{N_t(a)+ T_\text{S}(a) }\right \} \\
    & \le  \sum_{a \in A_t} \min \left \{2 \sqrt{\frac{2\log(4KT^3)}{N_t(a)}}, 2\sqrt{\frac{2\log(4KT^3)}{N_t(a)+ T_\text{S}(a)}} + \frac{ T_\text{S}(a) \omega(a)}{N_t(a)+ T_\text{S}(a) }\right \}
    \end{aligned}
\end{equation*}

Therefore,
\begin{equation*}
    \sum_{t=1}^T \Delta(A_t) \boldsymbol{1} \{\xi_t\} = \sum_{t=1}^T \Delta(A_t) \boldsymbol{1} \{\xi_t,\Delta(A_t) > 0\} \le \sum_{t=1}^T \Delta(A_t) \boldsymbol{1} \{\mathcal{F}_t\}.
\end{equation*}

Altogether, the lemma is proved.

\endproof

\begin{lemma}
    Suppose sequence $\{\alpha_i\}_{i\ge 1}$ and $\{\beta_i\}_{i\ge 0}$ satisfies $4 \ge \alpha_1 >1 >  \alpha_2 \ge \cdots$, $1 = \beta_0 \ge \beta_1 \ge \beta_2 \ge \cdots$, $\lim_{i \rightarrow \infty} \alpha_i = \lim_{i \rightarrow \infty} \beta_i = 0$, and
    \begin{equation*}
        \sum_{i=1}^{+\infty} \frac{\beta_{i-1} - \beta_i}{\sqrt{\alpha_i}} \le 1.
    \end{equation*}
    For any $t \in [T]$, suppose $\mathcal{F}_t$ happens, then there exist $i \ge 1$ such that the event $G_{i,t}$ (defined in (\ref{eq:pf-thm-inst-dept-git-def}) with $W_{i,t,a}$ defined in (\ref{eq:pf-thm-inst-dept-wit-def})) happens.
    \label{lem:pf-thm-inst-dept-git-cover}
\end{lemma}

\proof{Proof for Lemma \ref{lem:pf-thm-inst-dept-git-cover}}
    The proof for this lemma follows the same line of \cite{kveton2015tight}, with novel analysis on our special $W_{i,t,a}$ value. Fix $t$ such that $\Delta(A_t) > 0$. Denote
    \begin{equation*}
        S_{i,t} = \{a \in \tilde{A}_t:N_t(a) \le W_{i,t,a}\}.
    \end{equation*}
    Then
    \begin{equation*}
        G_{i,t} = \left(\cap_{j=1}^{i-1} \{|S_j| < \beta_j m\} \right) \cap \{|S_i| \ge \beta_i m\}.
    \end{equation*}
    Denote $G_t^C = \left(\cup_{i=1}^{\infty} G_{i,t} \right)^C$, then
    \begin{equation*}
        G_t^C = \left(\bigcup_{i=1}^{\infty} G_{i,t} \right)^C = \bigcap_{i=1}^{\infty} G_{i,t}^C = \bigcap_{i=1}^{\infty}\{|S_i| < \beta_i m\}.
    \end{equation*}
    Denote $\bar{S}_{i,t} = \tilde{A}_t \backslash S_{i,t}$. Then clearly $\bar{S}_{i-1,t} \subseteq \bar{S}_{i,t}$ for all i and $\tilde{A}_t = \cup_{i=1}^{\infty} (\bar{S}_{i,t} \backslash \bar{S}_{i-1,t})$. Suppose $G_t^C$ happens, then
    \begin{subequations}
        \begin{align}
            \Delta(A_t) & \le \sum_{a \in A_t} \min \left \{2 \sqrt{\frac{2\log(4KT^3)}{N_t(a)}}, 2\sqrt{\frac{2\log(4KT^3)}{N_t(a)+ T_\text{S}(a)}} + \frac{ T_\text{S}(a) \omega(a)}{N_t(a)+ T_\text{S}(a) }\right \} \label{eq:pf-lem-delatat-gcevent-a} \\
            & = \sum_{i=1}^{+\infty} \sum_{a \in \bar{S}_{i,t} \backslash \bar{S}_{i-1,t}} \min \left \{2 \sqrt{\frac{2\log(4KT^3)}{N_t(a)}}, 2\sqrt{\frac{2\log(4KT^3)}{N_t(a)+ T_\text{S}(a)}} + \frac{ T_\text{S}(a) \omega(a)}{N_t(a)+ T_\text{S}(a) }\right \} \label{eq:pf-lem-delatat-gcevent-b}\\
            & < \sum_{i=1}^{+\infty} \sum_{a \in \bar{S}_{i,t} \backslash \bar{S}_{i-1,t}} \frac{\Delta(A_t)}{m \sqrt{\alpha_i}} \label{eq:pf-lem-delatat-gcevent-c} \\
            & = \frac{\Delta(A_t)}{m} \cdot \sum_{i=1}^{+\infty} \frac{|\bar{S}_{i,t} \backslash \bar{S}_{i-1,t}|}{\sqrt{\alpha_i}} \nonumber \\
            & < \frac{\Delta(A_t)}{m} \cdot \sum_{i=1}^{+\infty} \frac{(\beta_{i-1} - \beta_i)m}{\sqrt{\alpha_i}} \label{eq:pf-lem-delatat-gcevent-d} \\
            & \le \Delta(A_t). \label{eq:pf-lem-delatat-gcevent-e}
        \end{align}
        \label{eq:pf-lem-delatat-gcevent}
    \end{subequations}
    This is a contradiction. Therefore, Lemma \ref{lem:pf-thm-inst-dept-git-cover} holds after (\ref{eq:pf-lem-delatat-gcevent-a}) - (\ref{eq:pf-lem-delatat-gcevent-e}) are proved. Specifically, 
    \begin{itemize}
        \item (\ref{eq:pf-lem-delatat-gcevent-a}) comes from the fact that $\mathcal{F}_t$ holds.
        \item (\ref{eq:pf-lem-delatat-gcevent-b}) comes from the fact that $\bar{S}_{i-1,t} \subseteq \bar{S}_{i,t}$ for all $i$. More importantly, $\lim_{i \rightarrow +\infty} W_{i,t,a} \le 0$. Thus, there exists a $j$ such that $\bar{S}_{i,t} = \tilde{A}_t$ for all $i > j$, hence $\tilde{A}_t = \cup_{i=1}^{\infty} (\bar{S}_{i,t} \backslash \bar{S}_{i-1,t})$.
        \item (\ref{eq:pf-lem-delatat-gcevent-c}) comes from the following analysis: For any $a \in \bar{S}_{i,t} \backslash \bar{S}_{i-1,t}$, $N_t(a) > W_{i,t,a}$. For simplicity, we denote
    \begin{equation*}
        J_{i,t,a} =  T_{\text{S}}(a) \cdot \max \left\{1 - \frac{m\sqrt{\alpha_i} \cdot \omega(a)}{\Delta(A_t)},0 \right \}^2 
    \end{equation*}
    and $W_{i,t,a} =16 \alpha_i \frac{m^2}{\Delta(A_t)^2} \log(4KT^3) - J_{i,t,a}$. If $J_{i,t,a} \le 8 \alpha_i \frac{m^2}{\Delta(A_t)^2} \log(4KT^3)$,
    \begin{equation*}
        N_t(a) > W_{i,t,a} \ge 8 \alpha_i \frac{m^2}{\Delta(A_t)^2} \log(4KT^3),
    \end{equation*}
    hence
    \begin{equation*}
        2 \sqrt{\frac{2\log(4KT^3)}{N_t(a)}} < \frac{\Delta(A_t)}{m \sqrt{\alpha_i}}.
    \end{equation*}
    Otherwise, we have
    \begin{equation*}
         T_{\text{S}}(a) \cdot \max \left\{1 - \frac{m\sqrt{\alpha_i} \cdot \omega(a)}{\Delta(A_t)},0 \right \}^2 =  T_{\text{S}}(a) \cdot  \left (1 - \frac{m \sqrt{\alpha_i} \cdot \omega(a)}{\Delta(A_t)}\right )^2 > 8 \alpha_i \frac{m^2}{\Delta(A_t)^2} \log(4KT^3). 
    \end{equation*}
    Then
    \begin{equation*}
        \begin{aligned}
            2\sqrt{\frac{2\log(4KT^3)}{N_t(a)+ T_\text{S}(a)}} + \frac{ T_\text{S}(a) \omega(a)}{N_t(a)+ T_\text{S}(a) } & \le 2\sqrt{\frac{2\log(4KT^3)}{ T_\text{S}(a)}} +  \omega(a) \\
            & < \frac{\Delta(A_t)}{m \sqrt{\alpha_i}} \left(1 - \frac{m \sqrt{\alpha_i}\cdot \omega(a)} {\Delta(A_t)} \right) + \frac{\Delta(A_t)}{m\sqrt{\alpha_i}} \cdot \frac{m \sqrt{\alpha_i} \cdot \omega(a)}{\Delta(A_t)} \\
            & =  \frac{\Delta(A_t)}{m \sqrt{\alpha_i}}.
        \end{aligned} 
    \end{equation*}
    Overall, (\ref{eq:pf-lem-delatat-gcevent-c}) is proved.
        \item (\ref{eq:pf-lem-delatat-gcevent-d}) comes from the following (this part is the same as the proof for Lemma 4 in \cite{kveton2015tight}, we provide it for completeness):
    \begin{equation*}
        \begin{aligned}
        \sum_{i=1}^{+\infty} \frac{|\bar{S}_{i,t} \backslash \bar{S}_{i-1,t}|}{\sqrt{\alpha_i}} & = \sum_{i=1}^{+\infty} \frac{|S_{i-1,t} \backslash S_{i,t}|}{\sqrt{\alpha_i}} \\
        & = \sum_{i=1}^{+\infty} \frac{|S_{i-1,t}| - |S_{i,t}|}{\sqrt{\alpha_i}} \\
        & = \frac{|S_0|}{\sqrt{\alpha_1}} + \sum_{i=1}^{+\infty} |S_{i,t}| \left(\frac{1}{\sqrt{\alpha_i}} - \frac{1}{\sqrt{\alpha_{i-1}}} \right) \\
        & < \frac{\beta_0 m}{\sqrt{\alpha_1}} + \sum_{i=1}^{+\infty} \beta_i m \left(\frac{1}{\sqrt{\alpha_i}} - \frac{1}{\sqrt{\alpha_{i-1}}} \right) \\
        & = \sum_{i=1}^{+\infty} \frac{(\beta_{i-1} - \beta_{i})m}{\sqrt{\alpha_i}}.
        \end{aligned}
    \end{equation*}
        \item (\ref{eq:pf-lem-delatat-gcevent-e}) comes from the fact that
    \begin{equation*}
        \sum_{i=1}^{+\infty} \frac{\beta_{i-1} - \beta_i}{\sqrt{\alpha_i}} \le 1.
    \end{equation*}
    \end{itemize}
    Altogether, the lemma is proved.
\endproof

\subsubsection{Proof for Theorem \ref{thm:comb-instance-dependent-upper-bound}}

The proof follow the similar line of \cite{kveton2015tight}. We define event sequence $\{G_{i,t}\}_{i \ge 1, t \in [T]}$ as
\begin{equation}
        \begin{aligned}
        G_{i,t} = &\left \{\text{less than $\beta_j m$ items in $\tilde{A}_t$ satisfy $N_t(a) \le W_{j,t,a}$ for $1 \le j < i$},\right. \\ 
        & \left.\quad \text{at least $\beta_i m$ items in $\tilde{A}_t$ satisfy $N_t(a) \le W_{i,t,a}$.} \right\}
        \end{aligned}
        \label{eq:pf-thm-inst-dept-git-def}
\end{equation}
where $\tilde{A}_t = A_t \backslash A_*$, and
\begin{equation}
    W_{i,t,a} = 
            16 \alpha_i \frac{m^2}{\Delta(A_t)^2} \log(4KT^3) -   T_{\text{S}}(a) \cdot \max \left\{1 - \frac{m\sqrt{\alpha_i} \cdot \omega(a)}{\Delta(A_t)},0 \right \}^2 
        \label{eq:pf-thm-inst-dept-wit-def}
\end{equation}
Here, the sequence $\{\alpha_i\}_{i \ge 1}$ and $\{\beta_i\}_{i \ge 0}$ is defined as follows:
\begin{equation}
        \alpha_i = 12.3 \cdot 0.25^i,\qquad \beta_i = 0.3^i.
        \label{eq:pf-thm-inst-dept-alphabeta-def}
\end{equation}
Clearly, $\{\alpha_i\}_{i \ge 1}$ and $\{\beta_i\}_{i \ge 0}$ satisfy $4 \ge \alpha_1 >1 >  \alpha_2 \ge \cdots$, $1 = \beta_0 \ge \beta_1 \ge \beta_2 \ge \cdots$, $\lim_{i \rightarrow \infty} \alpha_i = \lim_{i \rightarrow \infty} \beta_i = 0$, and
\begin{equation*}
    \sum_{i=1}^{+\infty} \frac{\beta_{i-1} - \beta_i}{\sqrt{\alpha_i}} = \sqrt{\frac{1}{12.3}} \cdot \frac{1 - 0.3}{\sqrt{0.25} - 0.3} \approx 0.998 < 1.
\end{equation*}
By Lemma \ref{lem:pf-thm-inst-dept-ft}, \ref{lem:pf-thm-inst-dept-git-cover}, we have
\begin{equation*}
    \text{Reg}_T \le O\left(K \Delta_{\max} \right) + \sum_{t=1}^T \Delta(A_t) \boldsymbol{1} \{\mathcal{F}_t\} = O\left(K \Delta_{\max} \right)+ \sum_{i=1}^{+\infty} \sum_{t=1}^T \Delta(A_t) \boldsymbol{1}\{G_{i,t},\Delta(A_t) > 0\}.
\end{equation*}
Now, we further define the event with respect to each specific $a$. Specifically, we define
\begin{equation*}
    G_{i,t,a} = G_{i,t} \cap \left \{a \in \tilde{A}_t,N_t(a) \le W_{i,t,a} \right \}.
\end{equation*}
Then we have
\begin{equation*}
    \boldsymbol{1}\{G_{i,t},\Delta(A_t) > 0\} \le \frac{1}{\beta_i m} \sum_{a \in \mathcal{A} \backslash A_*} \boldsymbol{1}\{G_{i,t,a},\Delta(A_t) > 0\}.
\end{equation*}
This is because by the definition of $G_{i,t}$, there are at least $\beta_i m$ arms satisfies $N_t(a) \le W_{i,t,a}$. Denote $n(a)$ be the number of sub-optimal solution $A \in \mathcal{B}$ such that $a \in A$. Denote this $n(a)$ solution as $A_{a,1},\cdots,A_{a,n(a)}$, and we assume the gaps satisfy $\Delta(A_{a,1}) \ge \cdots \ge \Delta(A_{a,n(a)}) = \Delta_{\min}(a)$. Thus, we have 
\begin{subequations}
    \begin{align}
        & \sum_{i=1}^{+\infty} \sum_{t=1}^T \Delta(A_t) \boldsymbol{1}\{G_{i,t},\Delta(A_t) > 0\} \nonumber\\
        \le & \sum_{a \in \mathcal{A} \backslash A_*} \sum_{i=1}^{+\infty} \sum_{t=1}^T  \boldsymbol{1}\{G_{i,t,a},\Delta(A_t) > 0\} \cdot \frac{\Delta(A_t)}{\beta_i m} \nonumber \\
        \le & \sum_{a \in \mathcal{A} \backslash A_*} \sum_{i=1}^{+\infty} \sum_{t=1}^T \sum_{n=1}^{n(a)}  \boldsymbol{1}\{G_{i,t,a},\Delta(A_t) = \Delta(A_{a,n})\} \frac{\Delta(A_{a,n})}{\beta_i m} \nonumber \\
        \le & \sum_{a \in \mathcal{A} \backslash A_*} \sum_{i=1}^{+\infty} \sum_{t=1}^T \sum_{n=1}^{n(a)}  \boldsymbol{1} \left\{a \in \tilde{A}_t,N_t(a) \le W_{i,t,a},\Delta(A_t) = \Delta(A_{a,n}) \right\} \frac{\Delta(A_{a,n})}{\beta_i m} \label{eq:pf-thm-inst-dept-upper-main-new-a} \\
        \le & \sum_{a \in \mathcal{A} \backslash A_*} \sum_{i=1}^{+\infty} \frac{1}{\beta_i m} \cdot \left(\Delta(A_{a,1}) W_{i,a,1} + \sum_{n=2}^{n(a)} \Delta(A_{a,n}) (W_{i,a,n} - W_{i,a,n-1}) \right) \label{eq:pf-thm-inst-dept-upper-main-new-b} \\
        \le & 16m \log(4KT^3) \sum_{a \in \mathcal{A} \backslash A_*} \sum_{i=1}^{+\infty} \frac{\alpha_i}{\beta_i} \cdot \frac{2}{\Delta_{\min}(a)} \nonumber \\
        - &  \sum_{a \in \mathcal{A} \backslash A_*} \sum_{i=1}^{+\infty} \frac{T_{\text{S}}(a)}{\beta_i m} \cdot \Delta_{\min}(a) \cdot  \max \left\{1 - \frac{m \sqrt{\alpha_i} \cdot \omega(a)}{\Delta_{\min}(a)},0 \right \}^2 \label{eq:pf-thm-inst-dept-upper-main-new-c}\\
        \le & 16m \log(4KT^3) \sum_{a \in \mathcal{A} \backslash A_*} \sum_{i=1}^{+\infty} \frac{\alpha_i}{\beta_i} \cdot \frac{2}{\Delta_{\min}(a)} \nonumber \\
        - &  \sum_{a \in \mathcal{A} \backslash A_*} \sum_{i\ge 1: m \sqrt{\alpha_i} \le 1} m T_{\text{S}}(a) \cdot \frac{\alpha_i}{\beta_i} \cdot \Delta_{\min}(a) \cdot  \max \left\{1 - \frac{\omega(a)}{\Delta_{\min}(a)},0 \right \}^2 \label{eq:pf-thm-inst-dept-upper-main-new-d}\\
        \le & O \left( \sum_{a \in \mathcal{A} \backslash A_*} \frac{m \log(T)}{\Delta_{\min}(a)} - m T_{\text{S}}(a) \cdot \Delta_{\min}(a) \cdot \max \left\{1 - \frac{ \omega(a)}{\Delta_{\min}(a)},0 \right \}^2 \right). \label{eq:pf-thm-inst-dept-upper-main-new-e}
    \end{align}
    \label{eq:pf-thm-inst-dept-upper-main-new}
\end{subequations}
Here,
\begin{itemize}
    \item (\ref{eq:pf-thm-inst-dept-upper-main-new-a}) comes from the definition of $G_{i,t,a}$.
    \item In (\ref{eq:pf-thm-inst-dept-upper-main-new-b}),
    \begin{equation*}
        W_{i,a,n} = 16 \alpha_i \frac{m^2}{\Delta(A_{a,n})^2} \log(4KT^3) -   T_{\text{S}}(a) \cdot \max \left\{1 - \frac{m\sqrt{\alpha_i} \cdot \omega(a)}{\Delta(A_{a,n})},0 \right \}^2.
    \end{equation*}
    (\ref{eq:pf-thm-inst-dept-upper-main-new-b}) comes from the step nature of $\{\Delta(A_{a,n})\}_{n=1}^{n(a)}$. Specifically, as long as $N_t(a) \le W_{i,a,n}$ for some $n$, such $\Delta(A_{a,n})$ may occur. Since $\Delta(A_{a,n})$ decreases as $n$ increases, $W_{i,t,a}$ increases as $n$ increases. Hence, 
    \begin{equation*}
        \begin{aligned}
            \begin{cases}
                \{A_{a,\ell}\}_{\ell=1}^{n(a)} \text{ cannot be recognized as sub-optimal} & N_t(a) \le W_{i,a,1}, \\
                \{A_{a,\ell}\}_{\ell=n+1}^{n(a)} \text{ cannot be recognized as sub-optimal} & W_{i,a,n} < N_t(a) \le W_{i,a,n+1}, n\ge 1.
            \end{cases}
        \end{aligned}
    \end{equation*}
    Thus,
    \begin{equation*}
        \begin{aligned}
         &\sum_{t=1}^T \sum_{n=1}^{n(a)}  \boldsymbol{1} \left\{a \in \tilde{A}_t,N_t(a) \le W_{i,t,a},\Delta(A_t) = \Delta(A_{a,n}) \right\} \frac{\Delta(A_{a,n})}{\beta_i m}  \\
         \le & \Delta(A_{a,1}) W_{i,a,1} + \Delta(A_{a,2}) (W_{i,a,2} - W_{i,a,1}) + \cdots 
         \\
         = & \Delta(A_{a,1}) W_{i,a,1} + \sum_{n=2}^{n(a)} \Delta(A_{a,n}) (W_{i,a,n} - W_{i,a,n-1})
         \end{aligned}
    \end{equation*}
    \item (\ref{eq:pf-thm-inst-dept-upper-main-new-c}) comes from following: by the definition of $W_{i,a,n}$, we have
    \begin{equation*}
        \begin{aligned}
            & \frac{1}{\beta_{i}m} \cdot \left(\Delta(A_{a,1}) W_{i,a,1} + \sum_{n=2}^{n(a)} \Delta(A_{a,n}) (W_{i,a,n} - W_{i,a,n-1}) \right) \\
            = & 16m \log(4KT^3)\cdot \frac{\alpha_i}{\beta_i} \cdot \left[\Delta(A_{a,1}) \cdot \frac{1}{\Delta(A_{a,1})^2} + \sum_{n=2}^{n(a)}\Delta(A_{a,n}) \cdot \left(\frac{1}{\Delta(A_{a,n})^2}-\frac{1}{\Delta(A_{a,n-1})^2} \right)\right] \\
             - &  \frac{T_{\text{S}}(a)}{\beta_i m} \cdot \left[\Delta(A_{a,1}) \cdot  \max \left\{1 - \frac{m \sqrt{\alpha_i} \cdot \omega(a)}{\Delta(A_{a,1})},0 \right \}^2  \right . \\
             + & \left. \sum_{n=2}^{n(a)}\Delta(A_{a,n}) \cdot \left(\max \left\{1 - \frac{m \sqrt{\alpha_i} \cdot \omega(a)}{\Delta(A_{a,n})},0 \right \}^2-\max \left\{1 - \frac{m \sqrt{\alpha_i} \cdot \omega(a)}{\Delta(A_{a,n-1})},0 \right \}^2 \right)\right] \\
             \le & 16m \log(4KT^3)\cdot \frac{\alpha_i}{\beta_i} \cdot \frac{2}{\Delta_{\min}(a)} - \frac{T_{\text{S}}(a)}{\beta_i m} \cdot \Delta_{\min}(a) \cdot  \max \left\{1 - \frac{m \sqrt{\alpha_i} \cdot \omega(a)}{\Delta_{\min}(a)},0 \right \}^2.
        \end{aligned}
    \end{equation*}
    The last inequality comes from two facts: First by \cite{kveton2014matroid} Lemma 3, we have
    \begin{equation*}
        \Delta(A_{a,1}) \cdot \frac{1}{\Delta(A_{a,1})^2} + \sum_{n=2}^{n(a)}\Delta(A_{a,n}) \cdot \left(\frac{1}{\Delta(A_{a,n})^2}-\frac{1}{\Delta(A_{a,n-1})^2} \right) \le \frac{2}{\Delta(A_{a,n(a)})} = \frac{2}{\Delta_{\min}(a)}.
    \end{equation*}
    Second by the fact that $\Delta(A_{a,n}) \le \Delta(A_{a,n-1})$, we have
\begin{equation*}
    \begin{aligned}
    & \Delta(A_{a,1}) \cdot  \max \left\{1 - \frac{m \sqrt{\alpha_i} \cdot \omega(a)}{\Delta(A_{a,1})},0 \right \}^2 \\
    + & \sum_{n=2}^{n(a)}\Delta(A_{a,n}) \cdot \left(\max \left\{1 - \frac{m \sqrt{\alpha_i} \cdot \omega(a)}{\Delta(A_{a,n})},0 \right \}^2-\max \left\{1 - \frac{m \sqrt{\alpha_i} \cdot \omega(a)}{\Delta(A_{a,n-1})},0 \right \}^2 \right) \\
    \ge & \Delta(A_{a,1}) \cdot  \max \left\{1 - \frac{m \sqrt{\alpha_i} \cdot \omega(a)}{\Delta(A_{a,1})},0 \right \}^2 \\
    + & \sum_{n=2}^{n(a)} \left (\Delta(A_{a,n}) \cdot \max \left\{1 - \frac{m \sqrt{\alpha_i} \cdot \omega(a)}{\Delta(A_{a,n})},0 \right \}^2-\Delta(A_{a,n-1}) \cdot\max \left\{1 - \frac{m \sqrt{\alpha_i} \cdot \omega(a)}{\Delta(A_{a,n-1})},0 \right \}^2 \right) \\
    = & \Delta(A_{a,n(a)}) \cdot  \max \left\{1 - \frac{m \sqrt{\alpha_i} \cdot \omega(a)}{\Delta(A_{a,n(a)})},0 \right \}^2 = \Delta_{\min}(a) \cdot  \max \left\{1 - \frac{m \sqrt{\alpha_i} \cdot \omega(a)}{\Delta_{\min}(a)},0 \right \}^2.
    \end{aligned}
\end{equation*}
    \item (\ref{eq:pf-thm-inst-dept-upper-main-new-d}) comes from the fact that only when $m \sqrt{\alpha_i} \le 1$,
\begin{equation*}
    \max \left\{1 - \frac{m \sqrt{\alpha_i} \cdot \omega(a)}{\Delta_{\min}(a)},0 \right \}^2 \ge m^2 \alpha_i \cdot \max \left\{1 - \frac{ \omega(a)}{\Delta_{\min}(a)},0 \right \}^2.
\end{equation*}
    \item (\ref{eq:pf-thm-inst-dept-upper-main-new-e}) comes from the fact that
    \begin{equation*}
        \sum_{i=1}^{+\infty} \frac{\alpha_i}{\beta_i} < 74 < + \infty.
    \end{equation*}
\end{itemize}


Altogether, the theorem is proved.

\section{Proofs for Instance-independent Regret Upper Bounds}

\subsection{Key Lemmas Related to the LP program (\ref{eq:indpt-LP})}

\label{sec:app-inst-indep-key-lemmas}

We first provide two lemmas: Lemma \ref{lem:app-inst-indep-ip-trans}, Lemma \ref{lem:app-inst-indep-lp-solu}, and their proofs. These two lemmas bridge the connection between $\tau_*$, the optimum of the LP program (\ref{eq:indpt-LP}), and the heterogeneity of $\{T_\text{S}(a)\}_{a\in {\cal A}}$, playing a key role in deriving our instance-independent regret upper bounds.



\begin{lemma}
    For any $\rho \in (0,1]$, suppose that $(N^*(a))_{a\in {\cal A}}\in \mathbb{N}_{\geq 0}^{\cal A}$ is an optimal solution to the following integer optimization problem (IP):
    \begin{equation}
        \begin{aligned}
            \max_{(N(a))_{a\in {\cal A}}} \quad & \sum_{a\in {\cal A}} \sum^{N(a)}_{n(a)=1}\left( n(a)+ T_\text{S}(a)\right)^{-\rho/2} \\
            \text{s.t.} \quad & \sum_{a\in {\cal A}}N(a) =T,\\
            & N(a)\in \mathbb{N}_{\geq 0},\quad \forall a \in \mathcal{A}.
        \end{aligned}
        \label{eq:app-inst-indep-ip}
    \end{equation}
Then it must be the case that $N^*(a) \leq \max\{\lceil \tau_*\rceil - T_\text{S}(a), 0\}$ for all $a\in {\cal A}$, where we recall that $\tau_*$ is the optimum of the LP program (\ref{eq:indpt-LP}).
\label{lem:app-inst-indep-ip-trans}
\end{lemma}

\begin{lemma}
    For an optimal solution $(\tau_*, \{n_*(a)\}_{a\in {\cal A}})$ to  the LP program (\ref{eq:indpt-LP}), it holds that
\begin{equation}
\label{eq:app-inst-indep-lp-n-solu}
n_*(a) = \max\{ \tau_* - T_\text{S}(a), 0\}\text{ for every arm $a$.}
\end{equation}
\label{lem:app-inst-indep-lp-solu}
\end{lemma}


\subsubsection{Proof for Lemma \ref{lem:app-inst-indep-ip-trans}}

We establish the lemma by a contradiction argument. Suppose there exists an arm $a$ such that $N^*(a) \geq \max\{\lceil \tau_*\rceil - T_\text{S}(a), 0\}+1$. Firstly, we assert that there must exist another arm $a' \in {\cal A}\setminus \{a\}$ such that $N^*(a') \leq \max\{\lceil \tau_*\rceil - T_\text{S}(a'), 0\}-1$. If not, we then have $N^*(a) \geq \max\{\lceil \tau_*\rceil - T_\text{S}(a), 0\}+1$ and $N^*(a'') \geq \max\{\lceil \tau_*\rceil - T_\text{S}(a''), 0\}$ for all $a''\in {\cal A}$, but then
\begin{align*}
    \sum_{k\in {\cal A}}N^*(k) > \sum_{k\in {\cal A}}\max\{\lceil \tau_*\rceil - T_\text{S}(k), 0\} \geq \sum_{k\in {\cal A}}\max\{ \tau_* - T_\text{S}(k), 0\} \overset{(\dagger)}{=} \sum_{k\in {\cal A}}n_*(k) = T,
\end{align*}
violating the constraint $\sum_{k\in {\cal A}}N^*(k) = T$. Note that the equality $(\dagger)$ is by Lemma \ref{lem:app-inst-indep-lp-solu}. Thus, the claimed arm $a'$  exists. To this end, the condition $N^*(a') \leq \max\{\lceil \tau_*\rceil - T_\text{S}(a'), 0\}-1$ implies that $\lceil \tau_*\rceil > T_\text{S}(a')$, since otherwise $N^*(a') \leq -1$, which violates the non-negativity constraint on $N^*(a)$. Altogether we have established the existence of two distinct arms $a, a'\in {\cal A}$ such that
\begin{align}
    N^*(a) + T_\text{S}(a) &\geq \max\{\lceil \tau_*\rceil, T_\text{S}(a)\} + 1\text{, in particular } N^*(a)\geq 1,   \nonumber\\
    N^*(a') + T_\text{S}(a') &\leq\max\{\lceil \tau_*\rceil, T_\text{S}(a')\}-1 = \lceil \tau_*\rceil - 1. \nonumber
\end{align}
To establish the contradiction argument, consider another solution $(\tilde{N}(k))_{k\in {\cal A}}$ to the displayed optimization problem in the claim, where 
$$ \tilde{N}(k) =
  \begin{cases}
     N^*(k) - 1      & \quad \text{if } k = a\\
      N^*(k) + 1      & \quad \text{if } k = a'\\
    N^*(k)  & \quad \text{if } k\in {\cal A}\setminus \{a, a'\}.
  \end{cases}
$$
By the property that $ N^*(a)\geq 1$, $\tilde{N}(a) \geq 0$, and $(\tilde{N}(k))_{k\in {\cal A}}$ is a feasible solution. But then we have 
\begin{equation*}
    \begin{aligned}
        &\sum_{a\in {\cal A}} \sum^{\tilde{N}(a)}_{n(a)=1}(n(a)+ T_\text{S}(a))^{-\rho/2}  - \sum_{a\in {\cal A}} \sum^{N^*(a)}_{n(a)=1}(n(a)+ T_\text{S}(a))^{-\rho/2}  \\
= & (N^*(a') + T_\text{S}(a') + 1)^{-\rho/2} - (N^*(a) + T_\text{S}(a))^{-\rho/2}\\
\geq &\lceil \tau_*\rceil^{-\rho/2} - (\max\{\lceil \tau_*\rceil, T_\text{S}(a)\} + 1)^{-\rho/2} > 0, 
    \end{aligned}
\end{equation*}
which contradicts the assumed optimality of $(N^*(a))_{a\in {\cal A}}$. Thus Lemma \ref{lem:app-inst-indep-ip-trans} is shown.

\subsubsection{Proof for Lemma \ref{lem:app-inst-indep-lp-solu}}

Now, since $(\tau_*, \{n_*(a)\}_{a\in {\cal A}})$ is feasible to the (LP), the constraints in the (LP) give us that $n_*(a) \geq \max\{ \tau_* - T_\text{S}(a), 0\}$. We establish our asserted equality by a contradition argument. Suppose there is an arm $a'$ such that $n_*(a') - \max\{ \tau_* - T_\text{S}(a'), 0\} = \epsilon > 0$. Then it can be verified that the solution $(\tau', \{n'(a)\}_{a\in {\cal A}})$ defined as $\tau' = \tau_* + \epsilon / K$, $n'(a) = n_*(a) + \epsilon / K$ for all $a\in {\cal A}\setminus \{a'\}$ and $n'(a') = n_*(a) - \frac{K - 1}{K}\epsilon$ is also a feasible solution to (LP). But then $\tau' > \tau_*$, which contradicts with the optimality of $(\tau_*, \{n_*(a)\}_{a\in {\cal A}})$. Thus Lemma \ref{lem:app-inst-indep-lp-solu} is shown.

\subsection{Proof for Instance-independent bound for Multi-armed Bandit}

\subsubsection{Proof for Theorem \ref{thm:upper_indpt}} 

Conditioned on $\cap^T_{t=1} \xi_t$, we have
\begin{subequations}
\begin{align}
    \text{Reg}_T & = \sum^T_{t=1}\left[\mu^{(\text{on})}_* - \mu^{(\text{on})}(A_t)\right]\nonumber\\
    &\leq K\Delta_\text{max} + \sum^T_{t=K+1}\left[\min\{\text{UCB}_t(a_*),  \text{UCB}^\text{S}_t(a_*)\}  - \mu^{(\text{on})}(A_t)\right]\nonumber\\
    &\leq K\Delta_\text{max}+ \sum^T_{t=K+1}\left[\min\{\text{UCB}_t(A_t),  \text{UCB}^\text{S}_t(A_t)\}  - \mu^{(\text{on})}(A_t)\right]\nonumber\\
    &\leq K\Delta_\text{max} + \sum^T_{t=K+1}\left[\min\{\mu^{(\text{on})}(A_t) + 2\text{rad}_t(A_t),\mu^{(\text{on})}(A_t) + 2 \text{rad}^\text{S}_t(A(t))\}  - \mu^{(\text{on})}(A_t)\right]\label{eq:ind_by_Lemma2}\\
    &\leq K\Delta_\text{max} + 2\min\left\{ \sum^T_{t=K+1}\text{rad}_t(A_t), \sum^T_{t=K+1}\text{rad}^\text{S}_t(A_t)\right\}.
\end{align}
\end{subequations}
(\ref{eq:ind_by_Lemma2}) is by applying Lemma \ref{lem:conf-event}, and $\mu^{(\text{off})}(a) - \mu^{\text{(on)}}(a) \leq V(a)$ in the upper bound of $\text{UCB}^\text{S}_t(a)$. The conventional analysis shows that $\sum^T_{t=K+1}\text{rad}_t(A_t) = O(\sqrt{KT\log(T/\delta)})$. We focus on:
\begin{align}
\sum^T_{t=K+1}\text{rad}^\text{S}_t(A_t) &= \sum^T_{t=K+1} \left[\sqrt{\frac{2\log(2 t / \delta_t)}{N_t(a)+ T_\text{S}(a)}} + \frac{ T_\text{S}(a)}{N_t(a)+ T_\text{S}(a)}\cdot V(a) \right] \nonumber \\
& \leq \max_{a\in {\cal A}}V(a) \cdot T + \sum_{a\in {\cal A}} \sum^{N_T(a)}_{n=1}\sqrt{\frac{8\log(T / \delta)}{n+ T_\text{S}(a)}}.\nonumber
\end{align}
We claim that
\begin{equation}\label{eq:ind_crucial}
\sum_{a\in {\cal A}} \sum^{N_T(a)}_{n(a)=1}\sqrt{\frac{1}{n(a)+ T_\text{S}(a)}} \leq \sum_{a\in {\cal A}} \sum^{\max\{\lceil \tau_*\rceil - T_\text{S}(a), 0\}}_{n(a)=1}\sqrt{\frac{1}{n(a)+ T_\text{S}(a)}}.
\end{equation}
Consider the following integer programming problem, and denote the optimal solution as $(N^*_T(a))_{a\in {\cal A}}$ where $(N^*_T(a))_{a\in {\cal A}}\in \mathbb{N}_{\geq 0}^{\cal A}$:
\begin{equation*}
    \begin{aligned}
        \text{(IP)}:\max_{(N_T(a))_{a\in {\cal A}}} \quad & \sum_{a\in {\cal A}} \sum^{N_T(a)}_{n(a)=1}\sqrt{\frac{1}{n(a)+ T_\text{S}(a)}} \\
        \text{s.t.} \quad & \sum_{a\in {\cal A}}N_T(a) =T\\
        & N_T(a)\in \mathbb{N}_{\geq 0}, \ \forall a \in \mathcal{A}.
    \end{aligned}
\end{equation*}
Then Lemma \ref{lem:app-inst-indep-ip-trans} (with $\rho=1$) is immediately useful for proving (\ref{eq:ind_crucial}), since
$$
\sum_{a\in {\cal A}} \sum^{N_T(a)}_{n(a)=1}\sqrt{\frac{1}{n(a)+ T_\text{S}(a)}} \leq \sum_{a\in {\cal A}} \sum^{N^*_T(a)}_{n(a)=1}\sqrt{\frac{1}{n(a)+ T_\text{S}(a)}}  \le \sum_{a\in {\cal A}} \sum^{\max\{\lceil \tau_*\rceil - T_\text{S}(a), 0\}}_{n(a)=1}\sqrt{\frac{1}{n(a)+ T_\text{S}(a)}},
$$
where the first inequality is by the optimality of $(N^*_T(a))_{a\in {\cal A}}$ to (IP), and the fact that any realized $(N_T(a))_{a\in {\cal A}}$ must be feasible to (IP). The second inequality is a direct consequence of Lemma \ref{lem:app-inst-indep-ip-trans}. 
For each $a\in {\cal A}$ we then have 
\begin{equation}\label{eq:ind_crucial_1}
 \sum^{\max\{\lceil \tau_*\rceil - T_\text{S}(a), 0\}}_{n=1}\sqrt{\frac{1}{n+ T_\text{S}(a)}} \leq \frac{\max\{\lceil \tau_*\rceil - T_\text{S}(a), 0\}}{\lceil \tau_*\rceil} \sum^{\lceil \tau_*\rceil}_{t=1}\sqrt{\frac{1}{t}}\leq \max\{\lceil \tau_*\rceil - T_\text{S}(a), 0\}\cdot \frac{4}{\sqrt{\tau_*}}.
\end{equation}
Applying the feasibility of $(\tau_*, n_*)$ to (LP), we have
$$
\max\{\lceil \tau_*\rceil - T_\text{S}(a), 0\} \leq \max\{ \tau_* - T_\text{S}(a), 0\} + 1 \leq n_*(a) + 1.
$$
Combining the above with (\ref{eq:ind_crucial_1}) gives
$$
\sum_{a\in {\cal A}} \sum^{\max\{\lceil \tau_*\rceil - T_\text{S}(a), 0\}}_{n=1}\sqrt{\frac{1}{n+ T_\text{S}(a)}} \leq \sum_{a\in {\cal A}}(n_*(a) + 1)\cdot \frac{4}{\sqrt{\tau_*}}\leq T\cdot \frac{8}{\sqrt{\tau_*}},
$$
hence $\sum_{a\in {\cal A}} \sum^{N_T(a)}_{n=1}\sqrt{\frac{8\log(T / \delta)}{n+ T_\text{S}(a)}} = O(T\sqrt{\log(T / \delta) / \tau_*})$. Altogether, the Theorem is proved.

\subsection{Proof for Instance-independent bound for Combinatorial Bandit} 

\label{sec:app-pf-comb-upper-indep}

\subsubsection{Proof for Theorem \ref{thm:comb-gen-instance-independent-upper-bound}}

Denote $\text{minUCB}_t = \left(\min\left\{
         \text{UCB}_t(a),  \text{UCB}^\text{S}_t(a) 
         \right\} \right)_{a \in \mathcal{A}}$. Conditioned on $\cap^T_{t=1} \xi_t$, we have
\begin{subequations}
    \begin{align}
         \text{Reg}_T\ & = \sum_{t=1}^T \left ( \alpha \cdot \beta \cdot \text{r}^*_{\boldsymbol{\mu}^{\text{(on)}}} - \mathbb{E}\left[\text{r}_{\boldsymbol{\mu}^{\text{(on)}}}(A_t) \right ]\right) \nonumber \\
         & = \sum_{t=1}^T \left ( \alpha \cdot \beta \cdot \text{r}_{\boldsymbol{\mu}^{\text{(on)}}}(A^*_{\boldsymbol{\mu}^{\text{(on)}}}) - \mathbb{E}\left[\text{r}_{\boldsymbol{\mu}^{\text{(on)}}}(A_t) \right ]\right)  \label{eq:pf-comp-gen-regret-upper-decom-a} \\
         & \le  \sum_{t=1}^T \left ( \alpha \cdot \beta \cdot \text{r}_{\text{minUCB}_t}(A^*_{\boldsymbol{\mu}^{\text{(on)}}}) - \mathbb{E}\left[\text{r}_{\boldsymbol{\mu}^{\text{(on)}}}(A_t) \right ]\right) \label{eq:pf-comp-gen-regret-upper-decom-b} \\
         & \le \sum_{t=1}^T \left ( \alpha \cdot \beta \cdot \text{r}_{\text{minUCB}_t}^* - \mathbb{E}\left[\text{r}_{\boldsymbol{\mu}^{\text{(on)}}}(A_t) \right ]\right) \nonumber \\
         & \le K \Delta_{\max} + \sum_{t=1}^T \left ( \mathbb{E}\left[\text{r}_{\text{minUCB}_t}(A_t) \right ] - \mathbb{E}\left[\text{r}_{\boldsymbol{\mu}^{\text{(on)}}}(A_t) \right ]\right) \label{eq:pf-comp-gen-regret-upper-decom-c} \\
         & \le K \Delta_{\max} + \gamma \cdot \sum_{t=1}^T \left ( \max_{a \in A_t} \left \{\min\left\{
         \text{UCB}_t(a),  \text{UCB}^\text{S}_t(a) 
         \right\} - \boldsymbol{\mu}^{\text{(on)}}(a)  \right \}\right)^{\rho} \label{eq:pf-comp-gen-regret-upper-decom-d} \\
         & = K \Delta_{\max} + \gamma \cdot \sum_{t=1}^T  \min\left\{
         \text{UCB}_t(a_t) - \boldsymbol{\mu}^{\text{(on)}}(a_t),  \text{UCB}^\text{S}_t(a) - \boldsymbol{\mu}^{\text{(on)}}(a_t)
         \right\}^{\rho} \label{eq:pf-comp-gen-regret-upper-decom-e} \\
         & \le K \Delta_{\max} + 2^{\rho} \gamma \cdot \sum_{t=1}^T  \min\left\{
         \text{rad}_t(a_t),  \text{rad}^\text{S}_t(a_t) 
         \right\}^{\rho} \label{eq:pf-comp-gen-regret-upper-decom-f} \\
         & = K \Delta_{\max} + 2^{\rho} \gamma \cdot \sum_{t=1}^T  \min\left\{
         \text{rad}_t(a_t)^{\rho},  \text{rad}^\text{S}_t(a_t)^{\rho} 
         \right\} \nonumber \\
         & \le K \Delta_{\max} + 2^{\rho} \gamma \cdot  \min \left \{ \sum_{t=1}^T  
         \text{rad}_t(a_t)^{\rho} , \sum_{t=1}^T \text{rad}^\text{S}_t(a_t)^{\rho} \right \}. \nonumber
    \end{align}
    \label{eq:pf-comp-gen-regret-upper-decom}
\end{subequations}
Here, (\ref{eq:pf-comp-gen-regret-upper-decom-a}) comes from the definition of $\text{r}_{\boldsymbol{u}}^*$ and $A^*_{\boldsymbol{u}}$. (\ref{eq:pf-comp-gen-regret-upper-decom-b}) comes from the monotonicity of $\text{r}_{\boldsymbol{u}}(\cdot)$ and Lemma \ref{lem:conf-event}. (\ref{eq:pf-comp-gen-regret-upper-decom-c}) comes from the selection rule (\ref{eq:alg-comb-At-rule}) for $A_t$. By the definition of $(\alpha,\beta)$ optimization oracle, we have $\mathbb{E}[\text{r}_{\text{minUCB}_t}(A_t)] \ge \alpha \cdot \beta \cdot \text{r}_{\text{minUCB}_t}^*$. (\ref{eq:pf-comp-gen-regret-upper-decom-d}) comes from the bounded smoothness of $\text{r}_{\boldsymbol{u}}(\cdot)$ and the expression of $f(\cdot)$. In (\ref{eq:pf-comp-gen-regret-upper-decom-e}) $a_t = \mathop{\arg\max}_{a \in A_t} \left \{\min\left\{\text{UCB}_t(a),  \text{UCB}^\text{S}_t(a) \right\} - \boldsymbol{\mu}^{\text{(on)}}(a)  \right \}$. (\ref{eq:pf-comp-gen-regret-upper-decom-f}) is by Lemma \ref{lem:conf-event}. Denote \begin{equation*}
    N_t^{\max}(a) = \max \left \{1, \sum_{s=1}^{t-1} \boldsymbol{1} \left\{ a \in \mathop{\arg\max}_{a' \in A_t} \left \{\min\left\{\text{UCB}_t(a'),  \text{UCB}^\text{S}_t(a') \right\} - \boldsymbol{\mu}^{\text{(on)}}(a')  \right \} \right\}\right \}
\end{equation*}
be the number of times that $a$ becomes the maximum gap arm. Now we analyze the two terms in $\min$, respectively. Firstly, by the expression of $\text{rad}_t(a)$,
\begin{subequations}
    \begin{align}
        \sum_{t=1}^T  
         \text{rad}_t(a_t)^{\rho} & \le \sqrt{2\log(4 K  T^3/ \delta)} \sum_{t=1}^T N_t(a_t)^{-\rho / 2} \nonumber \\
         & \le  (2\log(4 K  T^3/ \delta))^{\rho/2} \sum_{t=1}^T N_t^{\max}(a_t)^{-\rho / 2} \label{eq:pf-comp-gen-regret-upper-1st-a}\\
         & \le (2\log(4 K  T^3/ \delta))^{\rho/2} \sum_{a \in \mathcal{A}} \sum_{n=1}^{N_T^{\max}(a)} n^{-\rho/2} \label{eq:pf-comp-gen-regret-upper-1st-b}\\
         & \le \frac{1}{1 - \rho /2} (2\log(4 K  T^3/ \delta))^{\rho/2} \sum_{a \in \mathcal{A}} N_T^{\max}(a)^{1-\rho/2} \nonumber \\
         & \le \frac{1}{1 - \rho/2} (2\log(4 K  T^3/ \delta))^{\rho/2} \cdot K^{\rho / 2} \cdot T^{1 - \rho/2},\label{eq:pf-comp-gen-regret-upper-1st-c}
    \end{align}
    \label{eq:pf-comp-gen-regret-upper-1st}
\end{subequations}
where (\ref{eq:pf-comp-gen-regret-upper-1st-a}) comes from the fact that $a_t \in A_t$, hence $N_t(a) \ge N_t^{\max}(a)$. (\ref{eq:pf-comp-gen-regret-upper-1st-b}) comes from rearanging the sum. (\ref{eq:pf-comp-gen-regret-upper-1st-c}) comes from Jensen Inequality and the fact that $\sum_{a \in \mathcal{A}} N_T^{\max}(a) = T$. Secondly, by the expression of $\text{rad}^{\text{S}}_t(a)$,
\begin{subequations}
    \begin{align}
        \sum^T_{t=1}\text{rad}^\text{S}_t(a_t)^{\rho} &= \sum^T_{t=1} \left[\sqrt{\frac{2\log(2 t / \delta_t)}{N_t(a_t)+ T_\text{S}(a_t)}} + \frac{ T_\text{S}(a_t)}{N_t(a_t)+ T_\text{S}(a_t)}\cdot V(a_t) \right]^{\rho} \nonumber\\
        &\leq V_{\max}^{\rho} \cdot T + (2\log(4 K  T^3/ \delta))^{\rho/2}\sum_{a\in {\cal A}} \sum^{N_T(a)}_{n=1} (n+ T_\text{S}(a))^{-\rho/2}.\nonumber
    \end{align}
\end{subequations}
Similarly, suppose $(N_T^*(a))_{a \in \mathcal{A}}$ be an optimal solution for IP (\ref{eq:app-inst-indep-ip}), then we claim that
\begin{equation*}
\sum_{a\in {\cal A}} \sum^{N_T(a)}_{n(a)=1}(n(a)+ T_\text{S}(a))^{-\rho/2} \leq \sum_{a\in {\cal A}} \sum^{N^*_T(a)}_{n(a)=1} (n(a)+ T_\text{S}(a))^{-\rho/2} \le \sum_{a\in {\cal A}} \sum^{\max\{\lceil \tau_*\rceil - T_\text{S}(a), 0\}}_{n(a)=1} (n(a)+ T_\text{S}(a))^{-\rho/2}.
\end{equation*}
For each $a\in {\cal A}$ we then have 
\begin{equation*}
 \sum^{\max\{\lceil \tau_*\rceil - T_\text{S}(a), 0\}}_{n=1} (n+ T_\text{S}(a))^{-\rho/2} \leq \frac{\max\{\lceil \tau_*\rceil - T_\text{S}(a), 0\}}{\lceil \tau_*\rceil} \sum^{\lceil \tau_*\rceil}_{t=1} t^{-\rho/2}\leq \max\{\lceil \tau_*\rceil - T_\text{S}(a), 0\}\cdot \frac{1}{1 - \rho/2} \cdot \tau_*^{-\rho/2}.
\end{equation*}
Combining above, we have
\begin{equation*}
\sum_{a\in {\cal A}} \sum^{\max\{\lceil \tau_*\rceil - T_\text{S}(a), 0\}}_{n=1}(n+ T_\text{S}(a))^{-\rho/2} \leq \sum_{a\in {\cal A}}(n_*(a) + 1)\cdot \frac{1}{1 - \rho/2} \cdot \tau_*^{-\rho/2} \leq T \cdot \frac{2}{1 - \rho/2} \cdot \tau_*^{-\rho/2}.
\end{equation*}
Hence
\begin{equation*}
    \sum^T_{t=1}\text{rad}^\text{S}_t(a_t)^{\rho} \le O \left( \left( V_{\max}^{\rho} + (\log(T/ \delta))^{\rho/2} \cdot \tau_*^{-\rho/2}  \right ) \cdot T \right).
\end{equation*}
Altogether, the Theorem is proved.

\subsection{Proof for instance-independent bound for Combinatorial Bandits with Linear Rewards} 

\label{sec:app-pf-comb-semi-upper-indep}

\subsubsection{Proof for Theorem \ref{thm:comb-instance-independent-upper-bound}}

Similarly, conditioned on $\cap^T_{t=1} \xi_t$, we have

\begin{subequations}
    \begin{align}
        \text{Reg}_T & = \sum_{t=1}^T \left(\sum_{a \in A_*} \mu^{\text{(on)}}(a)  - \sum_{a \in A_t} \mu^{\text{(on)}}(a) \right) \label{eq:pf-comp-semi-regret-upper-decom-a} \\
        & \le  K \Delta_{\max} + \sum_{t=1}^T  \sum_{a \in A_t} \left( \min\left\{
         \text{UCB}_t(a),  \text{UCB}^\text{S}_t(a) 
         \right\} - \mu^{\text{(on)}}(a)  \right) \label{eq:pf-comp-semi-regret-upper-decom-b} \\
         & \le  K \Delta_{\max} + \min \left \{\sum_{t=1}^T  \sum_{a \in A_t} \left( 
         \text{UCB}_t(a) - \mu^{\text{(on)}}(a)  \right) ,\sum_{t=1}^T  \sum_{a \in A_t} \left( 
         \text{UCB}^\text{S}_t(a) - \mu^{\text{(on)}}(a)  \right) \right\} \nonumber \\
         & \le K \Delta_{\max} + 2 \min \left \{\sum_{t=1}^T  \sum_{a \in A_t} \text{rad}_t(a) ,\sum_{t=1}^T  \sum_{a \in A_t} \text{rad}_t^{\text{S}}(a_t) \right\}. \nonumber
    \end{align}
    \label{eq:pf-comp-semi-regret-upper-decom}
\end{subequations}
where (\ref{eq:pf-comp-semi-regret-upper-decom-a}) comes from the definition of $\text{r}_{\boldsymbol{u}}(A) = \sum_{a \in A} u(a)$. (\ref{eq:pf-comp-semi-regret-upper-decom-b}) comes from the following that
\begin{equation*}
    \sum_{a \in A_*} \mu^{\text{(on)}}(a) \le \sum_{a \in A_*} \min\left\{
         \text{UCB}_t(a),  \text{UCB}^\text{S}_t(a) 
         \right\} \le \sum_{a \in A_t} \min\left\{
         \text{UCB}_t(a),  \text{UCB}^\text{S}_t(a) 
         \right\} .
\end{equation*}
By conventional analysis, we can obtain
\begin{equation*}
    \sum_{t=1}^T  \sum_{a \in A_t} \text{rad}_t(a)  \le O \left( \sqrt{mKT\log(T/\delta)}\right).
\end{equation*}
By the analysis similar as the proof of Theorem \ref{thm:upper_indpt}, we can obtain
\begin{equation*}
    \sum_{t=1}^T  \sum_{a \in A_t} \text{rad}_t^{\text{S}}(a_t) \le O \left( \left(\sqrt{\frac{\log(T /\delta)}{\tau_*^{\text{C}}}}+ V_\text{max}  \right)\cdot mT\right).
\end{equation*}
Altogether, the Theorem is proved.

\subsubsection{Discussions: (\ref{eq:comb-semi-upper-indpt}) v.s. (\ref{eq:comb-semi-upper-indpt-nottight})}

In the case of combinatorial bandits with linear rewards, if we directly apply Theorem \ref{thm:comb-gen-instance-independent-upper-bound}, we can obtain bound (\ref{eq:comb-semi-upper-indpt-nottight}). However, in this special case, a tighter bound can be achieved. The key observation comes from the structure of the reward function. Specifically, the boundedness smoothness can only guarantee
\begin{equation*}
    \text{r}_{\boldsymbol{u}_1}(A) - \text{r}_{\boldsymbol{u}_2}(A) \le m \max_{a \in A} |u_1(a) - u_2(a)|,
\end{equation*}
the additive nature of the linear rewards setting allow us to establish a much sharper inequality:
\begin{equation*}
    \text{r}_{\boldsymbol{u}_1}(A) - \text{r}_{\boldsymbol{u}_2}(A) \le \sum_{a \in A} |u_1(a) - u_2(a)|.
\end{equation*}
Leveraging this stronger boundedness condition, we can derive the improved regret bound (\ref{eq:comb-semi-upper-indpt}).

\section{Proofs for Regret Lower Bounds}
\subsection{Notational Set Up and Auxiliary Results}\label{sec:app_aux}
\textbf{Notational Set Up.} We recall the notational set-up in Section \ref{sec:impossibility}. 
Consider a non-anticipatory policy $\pi$ and an instance $I$ with reward distribution $P$. We denote $\rho_{P, \pi}$ as the joint probability distribution function on $S, A_1, R_1, \ldots, A_T, R_T$, the concatenation of the offline dataset $S$ and the online trajectory under policy $\pi$ on instance $I$. For a $\sigma(S, A_1, R_1, \ldots, A_T, R_T)$-measurable event $E$, we denote $\Pr_{P, \pi}(E)$ as the probability of $E$ holds under $\rho_{P, \pi}$. For a  $\sigma(S, A_1, R_1, \ldots, A_T, R_T)$-measurable random variable $Y$, we denote   $\mathbb{E}_{P, \pi}[Y]$ as the expectation of $Y$ under the joint probability distribution function $\rho_{P, \pi}$. In our analysis, we make use of $Y$ being $\text{Reg}_T(\pi, P)$ or $N_T(a)$ for an arm $a$. To lighten our notational burden, we abbreviate $\mathbb{E}_{P, \pi}[\text{Reg}_T(\pi, P) ]$ as $\mathbb{E}[\text{Reg}_T(\pi, P) ]$.

\textbf{Auxiliary Results}
To establish the regret lower bounds and the impossibility result, we need three auxiliary results, namely Claim \ref{claim:Gaussian} and Theorems \ref{thm:BH}, \ref{thm:chain}. Claim \ref{claim:Gaussian} is on the KL-divergence between Gaussian random variables:
\begin{claim}\label{claim:Gaussian}
    For $P_i = {\cal N}(\mu_i, \sigma^2)$, where $i\in \{1, 2\}$, we have $\text{KL}(P_1, P_2) = \frac{(\mu_1 - \mu_2)^2}{2\sigma^2}$.
\end{claim}
The following, dubbed Bretagnolle–Huber inequality, is extracted from Theorem 14.2 from \cite{lattimore2020bandit}:
\begin{theorem}\label{thm:BH}
    Let $\mathbb{P}, \mathbb{Q}$ be probability disributions on $(\Omega, {\cal F})$. For an event $E\in \sigma({\cal F})$, it holds that
    $$
    \Pr_\mathbb{P}(E) + \Pr_\mathbb{Q}(E^c) \geq \frac{1}{2}\exp\left(- \text{KL}(\mathbb{P}, \mathbb{Q})\right).
    $$
\end{theorem}
Lastly, the derivation of the chain rule in Theorem \ref{thm:chain} largely follows from the derivation of Lemma 15.1 in \cite{lattimore2020bandit}:
\begin{theorem}\label{thm:chain}
    Consider two instances $I_P, I_Q$ that share the same arm set ${\cal A}$, online phase horizon $T$, offline sample size $\{T_\text{S}(a)\}_{a\in {\cal A}}$, but have two different reward distributions $P = (P^{(\text{on})}, P^{(\text{off})})$, $Q = (Q^{(\text{on})},Q^{(\text{off})})$. For any non-anticipatory policy $\pi$, 
    it  holds that 
    $$
    \text{KL}(\rho_{P, \pi}, \rho_{Q, \pi}) = \sum_{a\in {\cal A}}\mathbb{E}_{P, \pi}[N_T(a)]\cdot \text{KL}(P^{(\text{on})}_a, Q^{(\text{on})}_a) + \sum_{a\in {\cal A}}T_S(a)\cdot \text{KL}(P^{(\text{off})}_a, Q^{(\text{off})}_a).
    $$
\end{theorem}
We provide a proof to theorem \ref{thm:chain} for completeness sake:

\proof{Proof of Theorem \ref{thm:chain}}
The proof largely follow the well-known chain rule in the multi-armed bandit literature, for example see Lemma 15.1 in \cite{lattimore2020bandit}. We start by explicitly expressing the joint probability function $\rho_{P, \pi}$ on $$S = \{\{ X_s(a)\}_{s=1}^{T_\text{S}(a)}\}_{a\in {\cal A}},A_1, R_1, \ldots, A_T, R_T,$$ under reward distribution $P$ and non-anticipatory policy $\pi$. In coherence with our focus on Gaussian instances, we only consider the case when all the random rewards (offline or online) are continuous random variables with support on $\mathbb{R}$. Generalizing the argument to general reward distributions only require notational changes. By an abuse in notation, we denote $P^{(\text{(off)}}_a(x_s(a))$ as the probability density function (with variable $x_s(a)$) of the offline reward distribution with arm $a$, and likewise for the online reward distribution. 

To ease the notation, we denote $x = \{ x_s(a)\}_{s=1}^{T_\text{S}(a)}$. Then $\rho_{P, \pi}$ is expressed as
\begin{align}
    &\rho_{P, \pi}(x,a_1, r_1, \ldots, a_T, r_T) \nonumber\\
     = & \underbrace{\left[\prod_{a\in {\cal A}} \prod^{T_\text{S}(a)}_{s=1} P^{\text{(off)}}_a(x_s(a))\right]}_{\text{on offline rewards}}\cdot \underbrace{\prod^T_{t=1} \left[  \pi_t(a_t |  x ,a_1, r_1, \ldots, a_{t-1}, r_{t-1}) \cdot P^{\text{(on)}}_{a_t}(r_t)  \right]}_{\text{on online arms and rewards}}\label{eq:pdf_P}.
\end{align}
Likewise, by replacing reward distribution $P$ with $Q$ but keeping the fixed policy $\pi$ unchanged, we have
\begin{align}
    &\rho_{Q, \pi}(x,a_1, r_1, \ldots, a_T, r_T) \nonumber\\
     = & \left[\prod_{a\in {\cal A}} \prod^{T_\text{S}(a)}_{s=1} Q^{\text{(off)}}_a(x_s(a))\right] \cdot \prod^T_{t=1} \left[  \pi_t(a_t | x ,a_1, r_1, \ldots, a_{t-1}, r_{t-1}) \cdot Q^{\text{(on)}}_{a_t}(r_t)  \right]\nonumber.
\end{align}
The KL divergence between $\rho_{P, \pi}$ and $\rho_{Q, \pi}$ is 
\begin{align}
    &\text{KL}(\rho_{P, \pi}, \rho_{Q, \pi}) \nonumber\\
= & \int_{x, r_1, \ldots, r_T}  \sum_{a_1, \ldots, a_T\in {\cal A}} \rho_{P, \pi}(x,a_1, r_1, \ldots, a_T, r_T) \log\left[\frac{\rho_{P, \pi}(x,a_1, r_1, \ldots, a_T, r_T)}{\rho_{Q, \pi}(x,a_1, r_1, \ldots, a_T, r_T)}\right]\text{ d}x\text{ d}r_1\ldots \text{ d}r_T\nonumber,
\end{align}
where $\int_x = \int_{ \{ x_s(a)\}_{s=1}^{T_\text{S}(a)} \in \mathbb{R}^{\sum_{a\in {\cal A}} T_\text{S}(a)}  } $, and $\text{d}x = \prod_{a\in {\cal A}} \prod^{T_\text{S}(a)}_{s=1} \text{d}x_s(a)$.

We use the explicit expressions of $\rho_{P, \pi}, \rho_{Q, \pi}$ to decompose the log term:
\begin{align}
    &\log\left[\frac{\rho_{P, \pi}(x,a_1, r_1, \ldots, a_T, r_T)}{\rho_{Q, \pi}(x,a_1, r_1, \ldots, a_T, r_T)}\right]\nonumber\\
=& \sum_{a\in {\cal A}}\sum^{T_\text{S}(a)}_{s=1} \log\left(\frac{P^{\text{(off)}}_a(x_s(a))}{Q^{\text{(off)}}_a(x_s(a))}\right) \nonumber\\
&+ \sum^T_{t=1} \underbrace{\pi_t(a_t | x ,a_1, r_1, \ldots, a_{t-1}, r_{t-1})  \log\left(\frac{\pi_t(a_t |  x ,a_1, r_1, \ldots, a_{t-1}, r_{t-1}) }{\pi_t(a_t |  x ,a_1, r_1, \ldots, a_{t-1}, r_{t-1}) }\right)}_{ = 0} + \sum^T_{t=1} \log\left( \frac{P^{\text{(on)}}_{a_t}(r_t)}{Q^{\text{(on)}}_{a_t}(r_t)} \right)\nonumber\\
=& \sum_{a\in {\cal A}}\sum^{T_\text{S}(a)}_{s=1} \log\left(\frac{P^{\text{(off)}}_a(x_s(a))}{Q^{\text{(off)}}_a(x_s(a))}\right)  + \sum^T_{t=1} \log\left( \frac{P^{\text{(on)}}_{a_t}(r_t)}{Q^{\text{(on)}}_{a_t}(r_t)} \right)\nonumber.
\end{align}
In the second line, the middle sum is equal to zero, which can be interpreted as the fact that we are evaluating the same policy $\pi$ on reward distributions $P, Q$. Finally, marginalizing and making use of (\ref{eq:pdf_P}) gives, for each $a\in {\cal A}$,
\begin{align}
& \int_{x, r_1, \ldots, r_T}  \sum_{a_1, \ldots, a_T\in {\cal A}} \rho_{P, \pi}(x,a_1, r_1, \ldots, a_T, r_T)  \log\left(\frac{P^{\text{(off)}}_a(x_s(a))}{Q^{\text{(off)}}_a(x_s(a))}\right) \text{ d}x\text{ d}r_1\ldots \text{ d}r_T \nonumber\\
= & \int_{x_s(a)} P^{\text{(off)}}_a(x_s(a)) \log\left(\frac{P^{\text{(off)}}_a(x_s(a))}{Q^{\text{(off)}}_a(x_s(a))}\right) 
 \text{ d}x_x(a) = \text{KL}(P^\text{(off)}_a, Q^\text{(off)}_a) \label{eq:chain_off_a},
\end{align}
where the first equality in (\ref{eq:chain_off_a}) follows by integrating with respect to $r_t$ over $\mathbb{R}$ and then summing over $a_t\in {\cal A}$ in the order of $t = T, \ldots, 1$, and then integrating over all variables in $x$ except $x_s(a)$. For each $t\in \{1, \ldots, T\}$ during the online phase we have 
\begin{align}
& \int_{x, r_1, \ldots, r_T}  \sum_{a_1, \ldots, a_T\in {\cal A}} \rho_{P, \pi}(x,a_1, r_1, \ldots, a_T, r_T)  \log\left(\frac{P^{\text{(on)}}_{a_t}(r_t)}{Q^{\text{(on)}}_{a_t}(r_t)}\right) \text{ d}x\text{ d}r_1\ldots \text{ d}r_T \nonumber\\
= & \int_{x, r_1, \ldots, r_t}  \sum_{a_1, \ldots, a_t\in {\cal A}}
 \left[\prod_{a\in {\cal A}} \prod^{T_\text{S}(a)}_{s=1} P^{\text{(off)}}_a(x_s(a))\right] \cdot \prod^t_{\tau=1} \left[  \pi_\tau(a_\tau |  x ,a_1, r_1, \ldots, a_{\tau-1}, r_{\tau-1}) \cdot P^{\text{(on)}}_{a_\tau}(r_\tau)  \right] \nonumber\\
 &\quad \log\left(\frac{P^{\text{(on)}}_{a_t}(r_t)}{Q^{\text{(on)}}_{a_t}(r_t)}\right) \text{ d}x\text{ d}r_1\ldots \text{ d}r_t\nonumber\\
=& \sum_{a_t\in {\cal A}} \left\{\int_{x, r_1, \ldots, r_{t-1}}  \sum_{a_1, \ldots, a_{t-1}\in {\cal A}}\left[\prod_{a\in {\cal A}} \prod^{T_\text{S}(a)}_{s=1} P^{\text{(off)}}_a(x_s(a))\right] \cdot \prod^{t-1}_{\tau=1} \left[  \pi_\tau(a_\tau |  x ,a_1, r_1, \ldots, a_{\tau-1}, r_{\tau-1}) \cdot P^{\text{(on)}}_{a_\tau}(r_\tau)  \right] \right. \nonumber\\
&\quad \times \pi_t(a_t |  x ,a_1, r_1, \ldots, a_{t-1}, r_{t-1})\text{ d}x\text{ d}r_1\ldots \text{ d}r_{t-1} \Big \} \cdot \int_{r_t} P^{\text{(on)}}_{a_t}(r_t) \log\left(\frac{P^{\text{(on)}}_{a_t}(r_t)}{Q^{\text{(on)}}_{a_t}(r_t)}\right)  \text{ d}r_t \nonumber\\
= & \sum_{a_t\in {\cal A}} \Pr_{P, \pi}(A_t = a_t) \int_{r_t} P^{\text{(on)}}_{a_t}(r_t) \log\left(\frac{P^{\text{(on)}}_{a_t}(r_t)}{Q^{\text{(on)}}_{a_t}(r_t)}\right)  \text{ d}r_t \nonumber\\
= & \sum_{a\in {\cal A}}\mathbb{E}_{P, \pi}[\mathbf{1}(A_t = a)] \cdot  \text{KL}(P^\text{(on)}_a, Q^\text{(on)}_a) 
 \label{eq:chain_on_t}.
\end{align}
Summing (\ref{eq:chain_off_a}) over  $s\in \{1, \ldots, T_\text{S}(a)\}$ and $a\in {\cal A}$, as well as summing (\ref{eq:chain_on_t}) over $t\in \{1, \ldots, T\}$ establish the Theorem.    
\endproof

\subsection{Proof for instance dependent regret lower bound for Multi-armed Bandit} \label{sec:app_lower_ins_dep}

\subsubsection{Proof for Theorem \ref{thm:dep_lb}} 
We denote $\mu^\text{(off)}(a), \mu^\text{(on)}(a)$ as the means of the Gaussian distributions $P^{\text{(off)}}_a, P^{\text{(on)}}_a$ respestively, for each arm $a \in {\cal A}$. In addition, recall the notation that $\mu_*^\text{(on)} = \max_{a\in {\cal A}} \mu^\text{(on)}(a)$, $\Delta(a) = \mu_*^\text{(on)} - \mu^\text{(on)}(a)$ and $\omega(a) = V(a) + (\mu^\text{(off)}(a) - \mu^\text{(on)}(a))$.

Consider an arbitrary but fixed arm $a$ with $\Delta(a) > 0$. We claim that
\begin{align}
\mathbb{E}_{P, \pi}[N_T(a)]&\geq \frac{2(1-p)}{(1+\epsilon)^2} \cdot \frac{\log T}{\Delta(a)^2} + \frac{2}{(1+\epsilon)^2 \Delta(a)^2} \cdot\log\left(\frac{\epsilon \Delta(a)}{8 C}\right) \nonumber\\
    & - T_\text{S}(a)\cdot \max\left\{\left(1 - \frac{ \omega(a) }{(1+\epsilon)\Delta(a)}\right), 0\right\}^2\label{eq:dep_lb_step1}
\end{align}
The Lemma then follows by applying (\ref{eq:dep_lb_step1}) on each sub-optimal arm, thus it remains to prove (\ref{eq:dep_lb_step1}). To proceed, we consider another Gaussain instance $Q = (Q^{\text{(off)}}, Q^{\text{(on)}})$, where 
$$
Q^{\text{(off)}}_k = P^{\text{(off)}}_k\text{ for all $k\in {\cal A}\setminus \{a\}$},
$$
but 
$$
Q^{(\text{on})}_a = {\cal N}(\mu^{\text{(on)}}(a) + (1+\epsilon)\Delta(a), 1), $$
\begin{equation}\label{eq:dep_lb_Qoff}
Q^{(\text{off})}_a =   
  \begin{cases} 
   {\cal N}(\mu^{\text{(off)}}(a), 1) & \text{if } \mu^{\text{(off)}}(a) \geq \mu^{\text{(on)}}(a) + (1+\epsilon)\Delta(a) - V(a) \\
    {\cal N}(\mu^{\text{(on)}}(a) + (1+\epsilon)\Delta(a) - V(a), 1)      & \text{if } \mu^{\text{(off)}}(a) < \mu^{\text{(on)}}(a) + (1+\epsilon)\Delta(a) - V(a)
  \end{cases}.
\end{equation}
It is clear that $P\in {\cal I}_V$ implies $Q\in {\cal I}_V$. Indeed, if the second case in (\ref{eq:dep_lb_Qoff}) holds, then the mean of $Q^{(\text{on})}_a - $ the mean of $Q^{(\text{off})}_a$ is equal to $V(a)$, while if the first case hold, then we evidently have
$$
\underbrace{\mu^{\text{(on)}}(a) + (1+\epsilon)\Delta(a)}_{\text{mean of $Q^{(\text{on})}_a$}} + V(a) \geq \mu^{\text{(off)}}(a) \geq \underbrace{\mu^{\text{(on)}}(a) + (1+\epsilon)\Delta(a)}_{\text{mean of $Q^{(\text{on})}_a$}} - V(a),
$$
thus $Q\in {\cal I}_V$.
Consider the event $E = \{N_T(a) \geq T/2\}$. By the assumed consistency of $\pi$, we have
\begin{align}
2 C T^p & \geq  \mathbb{E}[\text{Reg}_T(\pi, P)] + \mathbb{E}[ \text{Reg}_T(\pi, Q) ]\label{eq:by_consistency}\\
\geq & \frac{T}{2}\cdot \epsilon \Delta(a) \cdot\left[\Pr_{P, \pi}(E) + \Pr_{
Q, \pi}(E^c)\right]\label{eq:by_design}\\
\geq & \frac{T}{4}\cdot \epsilon \Delta(a)\cdot\exp\left[ - \mathbb{E}_{\pi, P}[N_T(a)]\cdot \text{KL}(P^{(\text{on})}_a, Q^{(\text{on})}_a) - T_S(a) \text{KL}(P^{(\text{off})}_a, Q^{(\text{off})}_a)\right] \label{eq:by_chain}.
\end{align}
Step (\ref{eq:by_consistency}) is by the definition of consistency. 
Step (\ref{eq:by_design}) is implied by our construction that in instance $P$, arm $a$ is sub-optimal with optiamlity gap $\Delta(a) \geq \epsilon \Delta(a)$, and in instance $Q$, arm $a$ is the unique optimal arm and other arms have optimality gap at least $\epsilon \Delta(a)$. 
Step (\ref{eq:by_chain}) is by the Chain rule Theorem \ref{thm:chain}, as well as the BH inequality. By our set up, we have
\begin{align}
 \text{KL}(P^{(\text{on})}_a, Q^{(\text{on})}_a)  &= \frac{(1+\epsilon)^2 \Delta(a)^2}{2}, \nonumber\\
 \text{KL}(P^{(\text{off})}_a, Q^{(\text{off})}_a)  &= \frac{ \max\{ (1+\epsilon)\Delta(a) - [V(a) + (\mu^\text{(off)}(a) - \mu^\text{(on)}(a))],0 \}^2}{2}. \nonumber
\end{align}
Plugging in we get
\begin{align}
2 C T^p &\geq \frac{T}{4}\cdot \epsilon \Delta(a) \cdot\exp\left[ - \mathbb{E}_{\pi, P}[N_T(a)]\cdot \frac{(1+\epsilon)^2 \Delta(a)^2}{2} \right.\nonumber\\
&\left. - T_S(a) \frac{ \max\{ (1+\epsilon)\Delta(a) - [V(a) + (\mu^\text{(off)}(a) - \mu^\text{(on)}(a))],0 \}^2}{2}\right],
\end{align}
which is equivalent to 
\begin{align*}
&\mathbb{E}_{\pi, P}[N_T(a)]\cdot \frac{(1+\epsilon)^2 \Delta(a)^2}{2} + T_S(a) \frac{ \max\{ (1+\epsilon)\Delta(a) - [V(a) + (\mu^\text{(off)}(a) - \mu^\text{(on)}(a))],0 \}^2}{2}\nonumber\\
\geq & (1-p)\log T+ \log\left(\frac{\epsilon \Delta(a)}{8C}\right).\nonumber
\end{align*}
Rearranging leads to the claimed inequality (\ref{eq:dep_lb_step1}), and the Lemma is proved.

\subsection{Proof for instance independent regret lower bound for Multi-armed Bandits} \label{sec:app_lower_ins_indep}

\subsubsection{Proof for Theorem \ref{thm-lower-independent}}

We proceed by a case analysis:

\textbf{Case 1a: $2\sqrt{KT}> T\cdot ( V_\text{max} + 1/\sqrt{\tau_*})$, and $V_\text{max} \leq 1/\sqrt{\tau_*}$. }We derive a regret lower bound of  $$\Omega \left(\min\left\{\sqrt{KT}, T\cdot \left(\frac{1}{\sqrt{\tau_*}} + V_\text{max}\right)\right\} \right) =\Omega\left(\frac{T}{\sqrt{\tau_*}}\right).$$

\textbf{Case 1b: $2\sqrt{KT}> T\cdot ( V_\text{max} + 1/\sqrt{\tau_*})$, and $V_\text{max} > 1/\sqrt{\tau_*}$. }We derive a regret lower bound of  $$\Omega \left(\min\left\{\sqrt{KT}, T\cdot \left(\frac{1}{\sqrt{\tau_*}} + V_\text{max}\right)\right\} \right) =\Omega(T\cdot V_\text{max}).$$

\textbf{Case 2: $2\sqrt{KT}\leq ( V_\text{max} + 1/\sqrt{\tau_*})\cdot T$.} We derive a regret lower bound of  $$\Omega \left(\min\left\{\sqrt{KT}, T\cdot \left(\frac{1}{\sqrt{\tau_*}} + V_\text{max}\right)\right\} \right) =\Omega(\sqrt{KT}).$$

We establish the three cases in what follows.

\textbf{Case 1a: $2\sqrt{KT}> T\cdot ( V_\text{max} + 1/\sqrt{\tau_*})$, and $V_\text{max} \leq 1/\sqrt{\tau_*}$. }We derive a regret lower bound of  $$\Omega \left(\min\left\{\sqrt{KT}, T\cdot \left(\frac{1}{\sqrt{\tau_*}} + V_\text{max}\right)\right\} \right) =\Omega\left(\frac{T}{\sqrt{\tau_*}}\right).$$
Now, without loss of generality, we assume that $1\in \text{argmax}_{a\in {\cal A}}T_\text{S}(a)$. Consider the following Gaussian reward distributions $P = (P^\text{(on)}, P^\text{(off)})$ with $P^\text{(on)}_a = P^\text{(off)}_a$ (thus $P\in  {\cal I}_V$ for any $V\in \mathbb{R}^{K}_{\geq 0}$), defined as 
$$P^\text{(on)}_a = \begin{cases}
    {\cal N}(\Delta, 1)       & \quad \text{if } a = 1\\
    {\cal N}(0, 1) & \quad \text{if } a\in {\cal A}\setminus \{1\}
  \end{cases}\text{, where }\Delta = \frac{1}{\sqrt{\tau_*}}.$$
Consider the values $(\tilde{\tau}, \tilde{n})$ defined as $\tilde{n}(a) = \mathbb{E}_{P, \pi}[N_T(a)]$ for each $a\in {\cal A}$, and $\tilde{\tau}= \min_{a\in {\cal A}}\{ T_\text{S}(a) + \tilde{n}(a) \}$. Evidently,  $(\tilde{\tau}, \tilde{n})$  is feasible to (LP), and thus we have $\tilde{\tau}\leq \tau_*$, the optimum of (LP). In particular, there exists an arm $a\in {\cal A}$ such that $T_\text{S}(a)+ \mathbb{E}_{P, \pi}[N_T(a)] \leq \tau_*$. To this end, let's consider two situations:

\textbf{Situation (i): $T_\text{S}(a)+ \mathbb{E}_{P, \pi}[N_T(a)] > \tau_*$ for all $a\in {\cal A} \setminus \{1\}$. } The condition immeidately implies $T_\text{S}(1)+ \mathbb{E}_{P, \pi}[N_T(1)] \leq \tau_*$. Then we can deduce that $T_\text{S}(1)+ \mathbb{E}_{P, \pi}[N_T(1)] < T_\text{S}(a)+ \mathbb{E}_{P, \pi}[N_T(a)]$ for all $a\in {\cal A} \setminus \{1\}$, which further implies that $\mathbb{E}_{P, \pi}[N_T(1)] < \mathbb{E}_{P, \pi}[N_T(a)]$ for all $a\in {\cal A} \setminus \{1\}$ since $T_\text{S}(1) = \max_{a\in {\cal A}}T_\text{S}(a)$. The above implies that $\mathbb{E}_{P, \pi}[N_T(1)] < T/K$, which implies that 
$$
\mathbb{E}[\text{Reg}_T(\pi, P)] > \frac{(K-1)T}{K}.
$$

\textbf{Situation (ii): $T_\text{S}(k)+ \mathbb{E}_{P, \pi}[N_T(k)] \leq \tau_*$ for a $k\in {\cal A} \setminus \{1\}$.} In this case, consider the following Gaussian reward distribution $Q =(Q^\text{(on)}, Q^\text{(off)}) $ with $Q^\text{(on)}= Q^\text{(off)}$, defined as 
$$Q^\text{(on)}_a = \begin{cases}
    {\cal N}(\Delta, 1)       & \quad \text{if } a = 1\\
     {\cal N}(2\Delta, 1)       & \quad \text{if } a = k\\
    {\cal N}(0, 1) & \quad \text{if } a\in {\cal A}\setminus \{1, k\}
  \end{cases}\text{, where }\Delta = \frac{1}{\sqrt{\tau_*}}.$$
To this end, note that $P^\text{(off)}_a = Q^{\text{(off)}}_a = P^\text{(on)}_a = Q^{\text{(on)}}_a$ for all $a\in {\cal A}\setminus \{k\}$, but $P^\text{(on)}_{k} = P^\text{(off)}_{k} \neq Q^{\text{(on)}}_{k} = Q^{\text{(off)}}_{k} $. In addition, both $P, Q$ belong to ${\cal I}_V$. Consider the event $E = \{N_T(1) < T/2\}$. We have
\begin{align}
\mathbb{E}[\text{Reg}_T(\pi, P)] +  \mathbb{E}[\text{Reg}_T(\pi, Q)] & \geq  \frac{T}{2}\cdot \Delta \cdot\left[\Pr_{P, \pi}(E) + \Pr_{Q, \pi}(E^c)\right]\label{eq:by_design_case1a}\\
& \geq  \frac{T}{4}\cdot \Delta \cdot\exp\left[ - \mathbb{E}_{P,\pi}[N_T(k)]\cdot \text{KL}(P^{(\text{on})}_{k}, Q^{(\text{on})}_{k}) - T_\text{S}(k)\text{KL}(P^{(\text{off})}_{k}, Q^{(\text{off})}_{k})\right] \label{eq:by_chain_case1a}\\
& =  \frac{T}{4}\cdot \Delta \cdot \exp\left[- \left( T_\text{S}(k)+ \mathbb{E}_{P,\pi}[N_T(k)] \right)\cdot\text{KL}(P^{(\text{on})}_{k}, Q^{(\text{on})}_{k}) 
 \right]\nonumber\\
& =  \frac{T}{4\sqrt{\tau_*}} \cdot \exp\left[- \left( T_\text{S}(k)+ \mathbb{E}_{P,\pi}[N_T(k)]\right) \cdot\frac{1}{2\tau_*} 
 \right]\nonumber\\
& \geq  \frac{T}{4\sqrt{\tau_*}} \cdot \exp\left[- \tau_*\cdot\frac{1}{2\tau_*} 
 \right] = \frac{1}{4\sqrt{e}}\cdot \frac{T}{\sqrt{\tau_*}}\label{eq:by_case1a_k}.
\end{align}
Altogether, in \textbf{Case 1a}, we either have 
$$
\mathbb{E}[\text{Reg}_T(\pi, P)] > \frac{(K-1)T}{K}
= \Omega\left(\frac{T}{\sqrt{\tau_*}}\right),
$$
or 
$$
\max\{\mathbb{E}[\text{Reg}_T(\pi, P)], \mathbb{E}[\text{Reg}_T(\pi, Q)]\} \geq \frac{1}{8\sqrt{e}}\cdot \frac{T}{\sqrt{\tau_*}}.$$

\textbf{Case 1b: $2\sqrt{KT}> T\cdot ( V_\text{max} + 1/\sqrt{\tau_*})$, and $V_\text{max} > 1/\sqrt{\tau_*}$. }We derive a regret lower bound of  $$\Omega \left(\min\left\{\sqrt{KT}, T\cdot \left(\frac{1}{\sqrt{\tau_*}} + V_\text{max}\right)\right\} \right) =\Omega(T\cdot V_\text{max}).$$ 
 Now, set $\Delta = V_\text{max}/2$, and consider the Guassian instance with reward disribution $P$:
$$
P^\text{(off)}_a = N(0, 1) \text{ for all $a\in {\cal A}$, and } P^\text{(on)}_a = \begin{cases}
    {\cal N}(\Delta, 1)       & \quad \text{if } a = 1,\\
    {\cal N}(0, 1) & \quad \text{if } a\in {\cal A}\setminus \{1\}.
  \end{cases}
$$
Now, there exists an arm $a'\in {\cal A}\setminus \{1\}$ such that $\mathbb{E}_{P, \pi}[N_T(a')] \leq T / (K-1)$. Consider the Gaussian instance with 
$$
Q^\text{(off)}_a = {\cal N}(0, 1) \text{ for all $a\in {\cal A}$, and } Q^\text{(on)}_a = \begin{cases}
    {\cal N}(\Delta, 1)       & \quad \text{if } a = 1,\\
    {\cal N}(2\Delta, 1)       & \quad \text{if } a = a',\\
    {\cal N}(0, 1) & \quad \text{if } a\in {\cal A}\setminus \{1, a'\}.
  \end{cases}
$$
To this end, note that $P^\text{(off)}_a = Q^{\text{(off)}}_a$ for all $a\in {\cal A}$, and $P^\text{(on)}_a = Q^{\text{(on)}}_a$ for all $a\in {\cal A}\setminus \{a'\}$, but $P^\text{(on)}_{a'} \neq Q^{\text{(on)}}_{a'}$. In addition, both $P, Q$ belong to ${\cal I}_V$. Consider the event $E = \{N_T(1) < T/2\}$. We have
\begin{align}
\mathbb{E}[\text{Reg}_T(\pi, P)] + \mathbb{E} [\text{Reg}_T(\pi, Q)] \geq & \frac{T}{2}\cdot \Delta \cdot\left[\Pr_{P, \pi}(E) + \Pr_{Q, \pi}(E^c)\right]\label{eq:by_design_case1b} \\
\geq & \frac{T}{4}\cdot \Delta \cdot\exp\left[ - \mathbb{E}_{P, \pi}[N_T(a')]\cdot \text{KL}(P^{(\text{on})}_{a'}, Q^{(\text{on})}_{a'})\right] \label{eq:by_chain_case1b}\\
\geq & \frac{T}{4}\cdot \Delta \cdot\exp\left[ - \frac{T}{K-1}\cdot 2\cdot \Delta^2\right] \nonumber\\
= & \frac{T}{8} \cdot V_\text{max} \cdot \exp\left[ - \frac{T}{K-1}\cdot \frac{V_\text{max}^2}{2}\right]\nonumber\\
> &\frac{T}{8} \cdot V_\text{max} \cdot \exp\left[ - \frac{K}{K-1}\right] \geq  \frac{1}{8 e^2} \cdot T\cdot V_\text{max}\label{eq:by_case1b_aprime}.
\end{align}
Step (\ref{eq:by_design_case1b}) is again by Theorem \ref{thm:BH}. Step (\ref{eq:by_chain_case1b}) is by the Chain rule (Theorem \ref{thm:chain}), and our observations on the constructed $P, Q$. Step (\ref{eq:by_case2_aprime}) is by the choice of arm $a'$ and the KL divergence between $P^{(\text{on})}_{a'}, Q^{(\text{on})}_{a'}$. The strict inequality in (\ref{eq:by_case1b_aprime}) is by the case assumption that  $2\sqrt{KT}> T\cdot ( V_\text{max} + 1/\tau_*) > T\cdot V_\text{max}$, which implies $\frac{T}{K-1}\cdot \frac{V_\text{max}^2}{2} < \frac{K}{K-1}$. The larger than equal in (\ref{eq:by_case1b_aprime}) is by the assumption that $K\geq 2$. Altogether, the desire regret lower bound of 
$$
\max\{\mathbb{E}[\text{Reg}_T(\pi, P)],\mathbb{E}[\text{Reg}_T(\pi, Q)] \}\geq \frac{1}{16 e^2}\cdot T\cdot V_\text{max}
$$
is achieved.

\textbf{Case 2: $2\sqrt{KT}\leq ( V_\text{max} + 1/\sqrt{\tau_*})\cdot T$.} We derive a regret lower bound of  $$\Omega \left(\min\left\{\sqrt{KT}, T\cdot \left(\frac{1}{\sqrt{\tau_*}} + V_\text{max}\right)\right\} \right) =\Omega(\sqrt{KT}), $$ largely by following \textbf{Case 1b} with a different $\Delta$, as well as the proof of Lemma 15.2 in \cite{lattimore2020bandit}. Recall  $\sqrt{KT}\geq T / \sqrt{\tau_*}$ since $\tau^* \geq T / K$, so the case condition implies $\sqrt{KT}\leq T V_\text{max}$, meaning $\sqrt{K/T} \leq  V_\text{max}$. Now, set $\Delta = (1/2)\sqrt{(K-1)/T}$, and consider the Guassian instance with reward disribution $P$:
$$
P^\text{(off)}_a = {\cal N}(0, 1) \text{ for all $a\in {\cal A}$, and } P^\text{(on)}_a = \begin{cases}
    {\cal N}(\Delta, 1)       & \quad \text{if } a = 1,\\
    {\cal N}(0, 1) & \quad \text{if } a\in {\cal A}\setminus \{1\}.
  \end{cases}
$$
Now, there exists an arm $a'\in {\cal A}\setminus \{1\}$ such that $\mathbb{E}_{P, \pi}[N_T(a')] \leq T / (K-1)$. Consider the Gaussian instance with 
$$
Q^\text{(off)}_a = {\cal N}(0, 1) \text{ for all $a\in {\cal A}$, and } Q^\text{(on)}_a = \begin{cases}
    {\cal N}(\Delta, 1)       & \quad \text{if } a = 1,\\
    {\cal N}(2\Delta, 1)       & \quad \text{if } a = a',\\
    {\cal N}(0, 1) & \quad \text{if } a\in {\cal A}\setminus \{1, a'\}.
  \end{cases}
$$
To this end, note that $P^\text{(off)}_a = Q^{\text{(off)}}_a$ for all $a\in {\cal A}$, and $P^\text{(on)}_a = Q^{\text{(on)}}_a$ for all $a\in {\cal A}\setminus \{a'\}$, but $P^\text{(on)}_{a'} \neq Q^{\text{(on)}}_{a'}$. In addition, both $P, Q$ belong to ${\cal I}_V$. Consider the event $E = \{N_T(1) < T/2\}$. We have
\begin{align}
\mathbb{E}[\text{Reg}_T(\pi, P)] +\mathbb{E}[  \text{Reg}_T(\pi, Q)] \geq & \frac{T}{2}\cdot \Delta \cdot\left[\Pr_{P, \pi}(E) + \Pr_{
Q, \pi}(E^c)\right]\label{eq:by_design_case2}\\
\geq & \frac{T}{4}\cdot \Delta \cdot\exp\left[ - \mathbb{E}_{P, \pi}[N_T(a')]\cdot \text{KL}(P^{(\text{on})}_{a'}, Q^{(\text{on})}_{a'})\right] \label{eq:by_chain_case2}\\
\geq & \frac{T}{4}\cdot \Delta \cdot\exp\left[ - \frac{T}{K-1}\cdot 2\cdot \Delta^2\right] \label{eq:by_case2_aprime}.
\end{align}
Step (\ref{eq:by_design_case2}) is again by Theorem \ref{thm:BH}. Step (\ref{eq:by_chain_case2}) is by the Chain rule (Theorem \ref{thm:chain}), and our observations on the constructed $P, Q$. Step (\ref{eq:by_case2_aprime}) is by the choice of arm $a'$ and the KL divergence between $P^{(\text{on})}_{a'}, Q^{(\text{on})}_{a'}$. Finally, putting in $\Delta = (1/2)\sqrt{(K-1)/T}$ gives us
$$
\max\{\text{Reg}_T(\pi, P),  \text{Reg}_T(\pi, Q) \}\geq \frac{1}{8\sqrt{e}}\cdot \sqrt{(K-1)T}.
$$

Altogether, the three cases cover all possibilites and the Theorem is proved.

\subsection{Proof for instance independent bound for Combinatorial Bandits with Linear Rewards}

\label{sec:app_comb_semi_lower_ins_indep}

\subsubsection{Proof for Theorem \ref{thm:comb-instance-independent-lower-bound}}

The proof based on the special $m$-path bandit problem provided by \cite{kveton2015tight}. Specifically, we define the $m$-path bandit problem as follows. The arm set is $\mathcal{A} = [K]$, where $K$ satisfies $K \ge 2m$ and $K / m$ is an integer. The feasible solution set $\mathcal{B}$ is defined as
\begin{equation*}
    \mathcal{B} = \left\{ \left\{\ell + (j-1)m:\ell = 1,\cdots,m \right \}:j=1,\cdots \frac{K}{m}\right\}.
\end{equation*}
In other words, the feasible solution set consists of $K/m$ paths, and each path contains $m$ arms. Both offline and online reward distributions satisfies the property that the distribution for the arms in the same path are identical. Specifically, $P^{\text{(off)}}_{\ell_1 + (j-1)m} = P^{\text{(off)}}_{\ell_2 + (j-1)m}$ and $P^{\text{(on)}}_{\ell_1 + (j-1)m} = P^{\text{(on)}}_{\ell_2 + (j-1)m}$ for any fixed $j$ and for any $\ell_1, \ell_2 \in [m]$. The offline data size $T_{\text{S}}(a)$ also satisfies that the size of dataset for the arms in the same path are the same, i.e. $T_{\text{S}}(\ell_1 + (j-1)m) = T_{\text{S}}(\ell_2 + (j-1)m) = T_{\text{S},j}$ for any fixed $j$ and for any $\ell_1, \ell_2 \in [m]$.

Notice that this problem is equivalent to a mult-armed bandit with offline data problem, with $K/m$ arms, each arm $j$ with $T_{\text{S},j}$ offline data, where $j \in [K/m]$, and the offline and online reward is scaled by $m$. Therefore, apply the result from Theorem \ref{thm-lower-independent}, we obtain a lower bound
\begin{equation*}
    \Omega \left(m \cdot \min\left\{\sqrt{\frac{K}{m}\cdot T}, \left(\frac{1}{\sqrt{\tau_*'}} + V_\text{max}\right)\cdot T\right\} \right) = \Omega \left(\min\left\{\sqrt{mKT}, \left(\frac{1}{\sqrt{\tau_*'}} + V_\text{max}\right)\cdot mT\right\} \right)
\end{equation*}
where
\begin{equation*}
        \begin{aligned}
            \tau_*'= \max_{\tau,n} \quad & \tau \\
            \text{s.t.} \quad & \tau \leq T_{\text{S},j} + n_j \quad \forall j \in [K/m],\\
            & \sum_{j\in [K/m]}n_j =T,\\
            & \tau\geq 0, n_j\geq 0\quad\;\; \forall j \in [K/m].
        \end{aligned}
    \end{equation*}
For each pair of feasible solution of the above optimization problem $(\tau,\{n_j\}_{j=1}^{K/m})$, the solution $(\tau,\{n(a)\}_{a \in \mathcal{A}})$, where
\begin{equation*}
    n(\ell + (j-1)m) = n_j, \ \forall \ell \in [m], \forall j.
\end{equation*}
is a pair of feasible solution for the problem of (\ref{eq:comb-inst-indpt-LP}) . Thus $\tau_*' \le \tau_*^{\text{C}}$, and the lower bound is proved.

\end{document}